\documentclass[10pt]{article}

\usepackage[left=.75in, right=.75in, top=1in,bottom=1in]{geometry}
\usepackage{amssymb,amsmath,amsthm}

\usepackage{gensymb}
\usepackage[font={small,sf,bf}]{caption}
\usepackage{xfrac}  
\usepackage{empheq}  
\usepackage{multirow}
\usepackage[hidelinks]{hyperref}
    \hypersetup{colorlinks=false}
\usepackage{cleveref}  
\usepackage{lipsum}  
\usepackage{xcolor}  
    \definecolor{highlight}{rgb}{1,1,0}
\usepackage{authblk}
\usepackage{graphicx}  
\usepackage{subfig}
\usepackage{grffile}  
\usepackage{tabularx}  
    \newcolumntype{Y}{>{\centering\arraybackslash}X}
\usepackage{colortbl}  
\usepackage{xspace}  
\usepackage{setspace}  
\usepackage{algorithm}  
\usepackage[noend]{algpseudocode}  
    \algrenewcommand\alglinenumber[1]{{\sffamily\color[rgb]{0.15, 0.35, 0.9}\footnotesize#1}}  
    \algrenewcommand{\algorithmiccomment}[1]{\hfill #1}  
    
    \algblock{ParFor}{EndParFor}
    \algnewcommand\algorithmicparfor{\textbf{parfor}}
    \algnewcommand\algorithmicpardo{\textbf{do}}
    \algnewcommand\algorithmicendparfor{\textbf{end\ parfor}}
    \algrenewtext{ParFor}[1]{\algorithmicparfor\ #1\ \algorithmicpardo}
    \algtext*{EndParFor}{} 
    
    \algblock{Streaming}{EndStreaming}
    \algnewcommand\algorithmicstreaming{\textbf{streaming}}
    \algnewcommand\algorithmicstreamingdo{\textbf{}}
    \algnewcommand\algorithmicendstreaming{\textbf{end\ streaming}}
    \algrenewtext{Streaming}[1]{\algorithmicstreaming\ #1\ \algorithmicstreamingdo}
    \algtext*{EndStreaming}{} 
\usepackage{xpatch}  
    \makeatletter
    \xpatchcmd{\algorithmic}{\itemsep\z@}{\itemsep=2 pt}{}{}
    \makeatother
\usepackage{tikz}
    
    \usetikzlibrary{shapes.misc,positioning,fit,calc}
    
\usepackage[version=4]{mhchem}  
\usepackage{cuted}  
\usepackage{booktabs}  
\usepackage{multirow}  
\usepackage{makecell}  
    
\usepackage{enumitem}  
    \setlist{nosep}  
\usepackage{placeins} 
\usepackage{adjustbox}  
\usepackage{soul}  
\usepackage{etoolbox}

\newtheorem{theorem}{Theorem}
    \AfterEndEnvironment{theorem}{\noindent\ignorespaces}
\newtheorem{proposition}{Proposition}
    \AfterEndEnvironment{proposition}{\noindent\ignorespaces}
    
\makeatletter
    \newcommand{\smallotimes}{\mathbin{\mathpalette\make@small\otimes}}
    \newcommand{\make@small}[2]{%
    \vcenter{\hbox{%
        \scalebox{0.6}{$\m@th#1#2$}%
    }}%
    }
    \let\sv@thm\@thm
    \def\@thm{\vspace{0.75em}\let\indent\relax\sv@thm}
\makeatother

\usepackage{pxfonts}
\usepackage[normalem]{ulem}

\theoremstyle{definition}
\newtheorem{defn}{Definition}

\def\ie{\textit{i.e.}\xspace}
\def\eg{\textit{e.g.}\xspace}
\def\R{\mathbb{R}}
\def\C{\mathcal{C}} 
\def\W{\mathcal{W}}

\def\N{\mathcal{N}} 
\def\d{\mathrm{d}}  
\def\ACF{\mathrm{ACF}}

\def\ACF{\mathrm{ACF}}

\newcommand{\SINN}[0]{SINN\xspace}
\newcommand{\REPO}[0]{\texttt{https://github.com/SINN-model/SINN}\xspace}

\title{\bfseries\sffamily Learning Stochastic Dynamics with Statistics-Informed Neural Network}
\author[1]{\normalsize Yuanran Zhu\thanks{\href{mailto:yzhu4@lbl.gov}{yzhu4@lbl.gov}}}
\author[1,3]{\normalsize Yu-Hang Tang\thanks{Corresponding author, \href{mailto:tang.maxin@gmail.com}{tang.maxin@gmail.com}, \href{mailto:tang@lbl.gov}{tang@lbl.gov}}}
\author[2]{\normalsize Changho Kim\thanks{\href{mailto:ckim103@ucmerced.edu}{ckim103@ucmerced.edu}}}
\affil[1]{\small Applied Mathematics and Computational Research Division, Lawrence Berkeley National Laboratory, Berkeley, CA 94720, USA}
\affil[2]{\small Department of Applied Mathematics, University of California, Merced, Merced, CA 95343, USA}
\affil[3]{\small NVIDIA Corporation, Santa Clara, CA 95051, USA}
\date{\today}

\begin{document}
\doublespacing

\maketitle

\paragraph{Abstract}

We introduce a machine-learning framework named statistics-informed neural network (SINN) for learning stochastic dynamics from data. This new architecture was theoretically inspired by a universal approximation theorem for stochastic systems, which we introduce in this paper, and the projection-operator formalism for stochastic modeling. We devise mechanisms for training the neural network model to reproduce the correct \emph{statistical} behavior of a target stochastic process. Numerical simulation results demonstrate that a well-trained SINN can reliably approximate both Markovian and non-Markovian stochastic dynamics. We demonstrate the applicability of SINN to coarse-graining problems and the modeling of transition dynamics. Furthermore, we show that the obtained reduced-order model can be trained on temporally coarse-grained data and hence is well suited for rare-event simulations.

\section{Introduction}
\label{sec:Intro}

The use of machine learning (ML) techniques to model stochastic processes and time series data has seen many contributions in the past years. Two common strategies are to utilize neural networks (NNs) to either \emph{solve} or \emph{learn} the associated differential equations.

In the `solver' strategy, the differential equation governing a dynamical system is assumed to be known a priori and an NN is used to construct a numerical solver for the equation.
A representative approach using the solver strategy is the physics-informed neural network (PINN)~\cite{chen2021learning,zhang2019quantifying}. In the PINN approach, an NN serves as a solver that transforms initial and boundary conditions into approximate solutions. Specifically, an NN is trained using a specialized loss function that is defined in terms of the underlying differential equation. Thus, minimizing this specialized loss function steers the solver towards producing outputs that conform to the target differential equation. More recent members within the PINN family include sparse physics-informed neural network (SPINN)~\cite{RAMABATHIRAN2021110600}, parareal physics-informed neural network (PPINN)~\cite{meng2020ppinn}, and so on~\cite{zhang2020learning}.

The `learning' strategy, on the other hand, aims to learn the hidden dynamics from data using an NN.
This strategy is adopted by the neural ordinary differential equation (NeuralODE) approach~\cite{chenNeuralOrdinaryDifferential2019}. In this approach, an NN is not used to directly solve an equation, but rather to compute the gradient of the state variables. To train such a network, an adjoint dynamic system and a reverse-time ODE solver are adopted to facilitate backpropagation. After updating the parameters of the augmented dynamics, solutions to the differential equations can be found to reconstruct and extrapolate the hidden dynamics. Recent developments on top of NeuralODE include NeuralSDE~\cite{liuNeuralSDEStabilizing2019,li2020scalable}, Neural Jump SDE~\cite{jiaNeuralJumpStochastic2020}, NeuralSPDE~\cite{salvi2021neural}, Neural Operators~\cite{kovachki2021neural}, infinitely deep Bayesian neural networks~\cite{xuInfinitelyDeepBayesian2021}, \textit{etc}.

Generally speaking, learning an unknown dynamics, which is an `inverse' problem, is more challenging than solving or simulating a known dynamics. This is particularly so for stochastic systems. The ever-growing abundance of data necessitates methods that can learn stochastic dynamics in a wide range of scientific disciplines such as, for example,  molecular dynamics~\cite{lei2016data,zhu2021effective,chu2017mori}, computational chemistry~\cite{tuckerman2010statistical}, and ecology~\cite{katz2011inferring}. Aside from the aforementioned work based on ML methodologies, recent progress in this direction includes the work of Lu et al. \cite{lu2021learning,lang2020learning,lu2021learning1} on the learning of interaction kernel of multi-particle systems with non-parametric methods, the kernel-based method \cite{gilani2021kernel,harlim2021machine} for learning discrete non-Markovian time series, and various approximations of the Mori-Zwanzig equation for the learning of non-Markovian stochastic dynamics \cite{lei2016data,zhu2021effective,zhu2020generalized,Li2015,hudson2020coarse,li2015incorporation} in molecular dynamics. Typically, these non-ML methods use sophisticated series expansion and regression techniques to learn and construct the desired stochastic model. 
Although being theoretically complete and numerically successful within their own applicability, such modeling processes are often complex and too difficult to be extended to high-dimensional or highly heterogeneous systems. The success of NN models in dealing with high-dimensional problems for complex systems inspired us to use it to build a simple data-driven, extensible framework for modeling stochastic dynamics, \ie the statistics-informed neural network (SINN). We now detail SINN's construction and the main rationale behind it, as well as its major differences from the existing ML frameworks.

To learn stochastic dynamics with NNs, we first consider how to use deterministic architectures such as the recurrent neural network (RNN) to generate randomness. Being built on top of a deterministic RNN architecture, SINN does not generate randomness \textit{per se}, but rather \emph{transforms} an input stream of discrete independent and identically distributed (i.i.d.) random numbers to produce realizations of stochastic trajectories. The modeling capacity of this simple construction can be examined using the universal approximation theorem (UAT) for RNNs. Specifically, by mimicking the proof of the UAT for a one-layer RNN for arbitrary deterministic, open dynamical systems \cite{schafer2006recurrent}, we obtain a similar result stating that a one-layer RNN with Gaussian white-noise input can universally approximate arbitrary stochastic systems. We then use the long short-term memory (LSTM)~\cite{hochreiter1997long} architecture as the building blocks for SINN in order to capture non-Markovian and memory effects that the underlying stochastic system may contain. Contrary to PINNs and the NeuralSDE, SINN is trained on the statistics, such as probability density function and time autocorrelation functions, of an {\em ensemble} of trajectories. Briefly speaking, SINN is mainly different from other ML frameworks in the following three aspects: (I) SINN is entirely equation-free --- training and modeling do not rely on knowledge about the underlying stochastic differential equation; (II) an RNN, instead of fully connected or convolutional layers, is used as the primary architecture of the model; (III) instead of seeking a {\em pathwise} approximation to the stochastic dynamics, SINN constructs simulated trajectories that converge to the example trajectories in the sense of probability {\em measure} and $n$-th {\em moment}. The computational and modeling merit brought by these three features will be elaborated later with numerical supports. 

\begin{figure}[t]
\centerline{\hspace{1cm}
Equation-based modeling\hspace{5.5cm}
Equation-free modeling
}
\vspace{0.3cm}
\centerline{
\includegraphics[height=5cm]{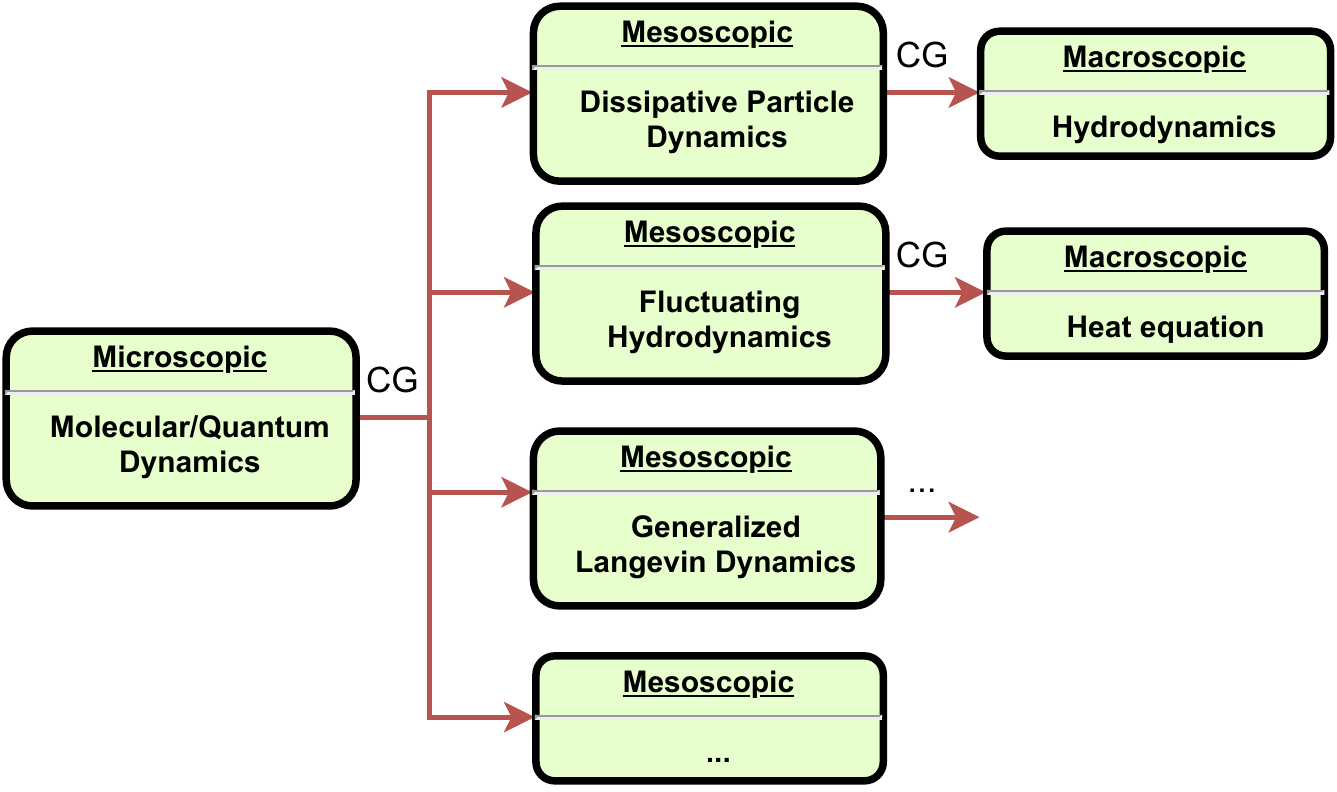}
\hspace{0.5cm}
\includegraphics[height=5cm]{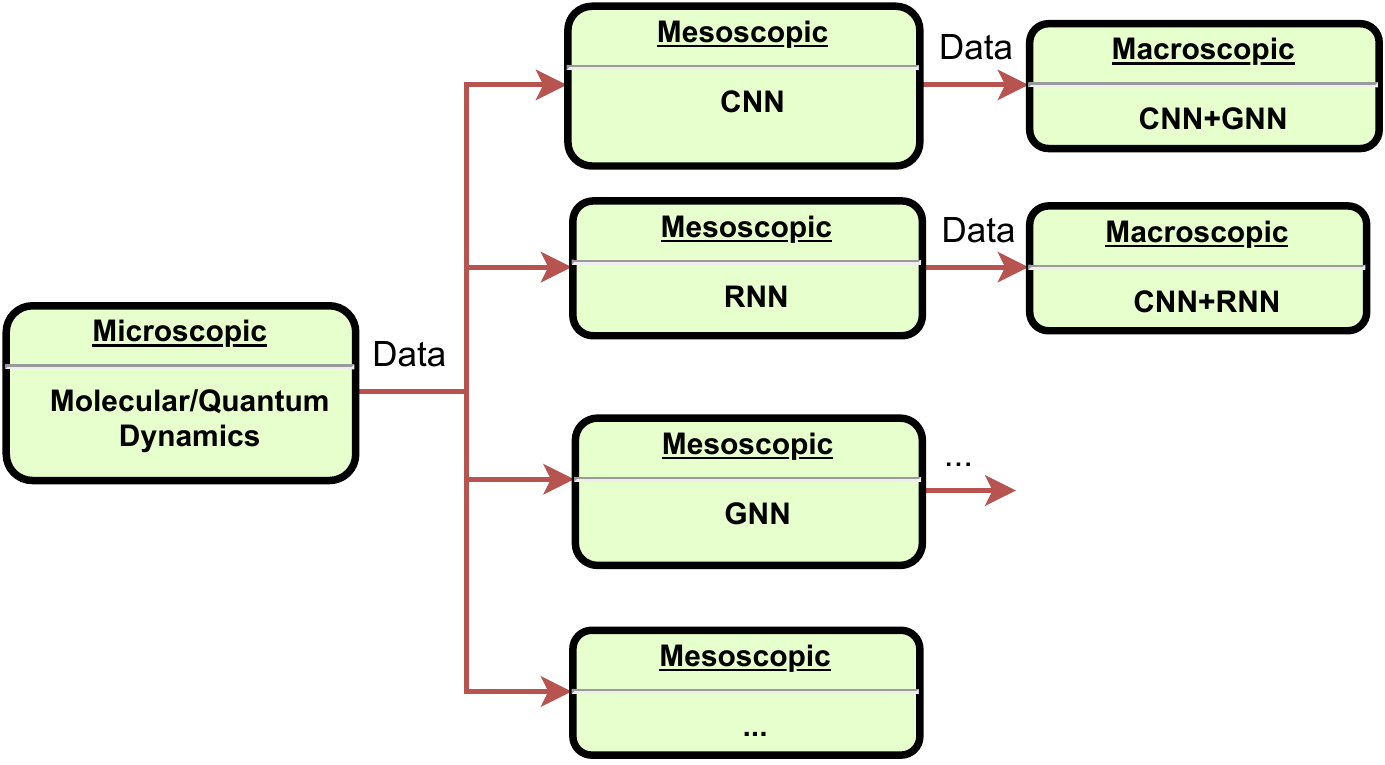}
}
\caption{Equation-based and equation-free modeling diagrams in different spatial-temporal scales. }
\label{fig:model_diagram} 
\end{figure}

In a greater picture, SINN can be categorized as an equation-free ML architecture for discovering the hidden dynamics of a physical system at different spatial-temporal scales. From a modeling perspective, the classical equation-based modeling approach, as shown in \Cref{fig:model_diagram}, normally starts with a microscopic model for the molecular/quantum dynamics. A certain coarse-grained procedure, such as the projection operator method or the mean-field approximation, is then applied to reduce the dimensionality of the system in order to generate mesoscopic and macroscopic models. At different scales, the established physical models have a general form: 
\begin{align}\label{eqn-based}
    \partial_t u=F(u,k,t),
\end{align}
where $u$ is the unknown (vector) variable or a (vector) field $u(x,t)$, $F(u,k,t)$ is the combination of functions, stochastic processes, and operators with modeling parameters $k$. In contrast to the equation-based approach, ML approaches rely on equation-free models based on the convolutional neural network (CNN)~\cite{albawi2017understanding,tian2020cross}, RNN~\cite{hughes2019wave}, or graph neural network (GNN)~\cite{scarselli2008graph} to model a physical process. These NNs need to be trained on data, such as the sample trajectories of an observable in the phase space. The general form of the modeling ansatz for the unknown function $u(x,t)$ relies on the multi-fold function composition:
\begin{align}\label{eqn_free-ansatz}
  u(x,t)=f_1(f_2(f_3(\cdots,k_3),k_2),k_1),
\end{align}
where $f_i$ is the activation function of $i$-th NN layer and $k_i$ represents the corresponding weights and parameters. ML models such as PINNs use the underlying equations to define the loss function while the modeling anstaz is of the form \eqref{eqn_free-ansatz}. Other examples such as the NeuralODE/SDE/SPDE use \eqref{eqn_free-ansatz} to model the derivatives function, \ie the right-hand side of equation \eqref{eqn-based}. In comparison, SINN is completely equation-free during the training and the modeling process.

Specifically for this paper, numerical experiments will be provided to demonstrate the capability of SINN in approximating Gaussian and non-Gaussian stochastic dynamics as well as its ability to capture the memory effect for non-Markovian systems. We also use SINN to discover surrogate models for transition dynamics that often appears in computational chemistry. Our SINN model can be trained using temporally coarse-grained trajectories. This feature makes it an efficient simulator for rare events. In addition to the application in physics, SINN provides a simple and flexible framework to model arbitrary stationary stochastic processes, hence is generally applicable in the areas of uncertainty quantification and time series modeling. Several simple test examples presented in \Cref{sec:test_cases} promisingly show its numerical advantages over established stochastic process modeling tools such as the transformed Karhunen-Lo\'eve or polynomial chaos expansion \cite{phoon2005simulation,sakamoto2002polynomial,zhu2020generalized,zhu2021effective}.  

This paper is organized as follows. In \Cref{sec:UAT_for_stochastic}, we review the established universal approximation theorem for a single-layer RNN model and show that there is a natural extension of this theorem for stochastic systems driven by Gaussian and non-Gaussian white noise. Inspired by this theoretical result, in \Cref{sec:SINN_architecture}, we propose a statistics-informed neural network (SINN) and introduce different types of loss functions. The training method of SINN is provided in \Cref{sec:training_method}. Three simple test examples are presented in \Cref{sec:test_cases} to demonstrate that SINN can well approximate both Gaussian and non-Gaussian stochastic dynamics. In \Cref{sec:applications}, we apply SINN to a coarse-graining problem and also use it as an effective rare-event simulator to evaluate transition rates. Several assessments of SINN as a tool for learning stochastic dynamics are also provided. Lastly, the main findings of this paper are summarized in \Cref{sec:conclusion}. The proofs for the main theorems are given in Appendix~\ref{app1:proof}.

\section{RNN as a Universal Approximator for Stochastic Dynamics}
\label{sec:UAT_for_stochastic}

Recurrent neural network (RNN) is a good candidate architecture for learning the unknown dynamics of a physical system since there is a natural correspondence between the recurrent internal structure of RNN and the time-recursive update rule that quantifies the dynamics. In this section, we discuss the universal approximation properties of RNN for stochastic processes, in particular, discrete stochastic processes corresponding to the numerical solutions of stochastic differential equations (SDEs). We consider a one-layer {\em deterministic} RNN with {\em stochastic} input and show that if the model is wide enough, \ie has a large number of hidden states, it can accurately approximate the finite-difference scheme of a time-homogeneous Markovian SDEs driven by Gaussian white noise.

We first review the universal approximation theorem (UAT) of RNN for deterministic dynamical systems established by Sch\"afer and Zimmermann in \cite{schafer2006recurrent}. To this end, let us consider a one-layer RNN model given by the update rule:
\begin{equation}\label{one_layer_RNN}
\begin{aligned}
    s_{t+1}&=\sigma(As_t+Bx_t-\theta),\\
    y_t&=Cs_t,
\end{aligned}
\end{equation}
where $s_t\in \R^{H}$ is the state vector of the RNN, $x_t\in \R^{I}$ is the input, $\sigma$ is the activation function of the network, and $y_t\in\R^{N}$ is the output. The modeling parameters of this simple, one-layer RNN are the weight matrices $A\in\R^{H\times H}$, $B\in \R^{H\times I}$, $C\in\R^{N\times H}$, and the bias vector $\theta\in \R^{H}$. An immediate observation is that this update rule is very similar to the structure of a discrete, open dynamical system of the form:
\begin{equation}\label{differential_sde}
\begin{aligned}
    s_{t+1}&=g(s_t,x_t),\\
    y_t&=h(s_t),
\end{aligned}
\end{equation}
where $s_t\in\R^{J}$, $x_t\in \R^{I}$, $y_t\in\R^{N}$, and the functions $g(\cdot):\R^{J}\times\R^{I}\rightarrow \R^{J}$, $h(\cdot):\R^{J}\rightarrow \R^{N}$. In fact, Sch\"afer and Zimmermann in \cite{schafer2006recurrent} proved that a one-layer RNN with update rule \eqref{one_layer_RNN} can universally approximate the dynamics of open dynamical system \eqref{differential_sde} with arbitrary accuracy. Their result can be restated as:

\begin{theorem}\label{UAT_deterministic_input}
(Sch\"afer \& Zimmermann \cite{schafer2006recurrent}, UAT for the deterministic RNN). Let $g(\cdot):\R^J\times\R^I\rightarrow \R^J$ be measurable and $h(\cdot):\R^J\rightarrow\R^N$ be continuous, the external input $x_t\in\R^{I}$, the inner state $s_t\in \R^J$, and the outputs $y_t\in \R^N$ $(t=1,\cdots, T)$. Then any discrete, open dynamical system of form \eqref{differential_sde} can be approximated by an RNN model of type \eqref{one_layer_RNN} to an arbitrary accuracy. 
\end{theorem}
The exact definition of an RNN model of type \eqref{one_layer_RNN} and the meaning of the arbitrarily accurate approximation are provided in Appendix \ref{app1:proof}.
The proof of this theorem was established based on the UAT for the feedforward neural networks. Here we mention some key points of Theorem \ref{UAT_deterministic_input}.
\begin{enumerate}
    \item The state vector $s_t\in \R^{J}$ in the dynamical system \eqref{differential_sde} is {\em not} the same as the state vector $s_t\in \R^{H}$ in RNN \eqref{one_layer_RNN}. In fact, to guarantee the accuracy of the approximation, there normally is an enlargement of the state space in RNN, \ie $H>J$. This has also to do with the next point.
    \item The universal approximation is in the sense of matching the input $x_t$ and the output $y_t$. This means that with the exact same input $x_t$, the output $y_t$ of the RNN should match closely with the output $y_t$ of the dynamical system, while the state vectors $s_t$ of these two systems may be different.
    \item The UAT assumes {\em finite}-step input/output, \ie\ $T<+\infty$. 
\end{enumerate}

The above UAT clearly indicates that even with a simple architecture such as the one-layer RNN, one can universally approximate any open dynamical system of the general form \eqref{differential_sde}. However, the theorem itself provides little guidance on \emph{how} to construct such an RNN model for a specific dynamical system. In practice, we rarely use a wide-enough RNN to complete the computing task. Nevertheless, the theorem indubitably shows the potential of the RNN architecture in modeling/learning dynamical systems.

We now show that a similar UAT holds for the resulting stochastic RNN by simply replacing the input vector $x_t$ with i.i.d.\ Gaussian random variables while leaving all other structures unchanged. This result can be stated as:

\begin{proposition}\label{UAT_stochastic_input}
(UAT for RNN with Gaussian inputs). Let $g(\cdot):\R^M\times\R^I\rightarrow \R^M$ be locally Lipschitz and $h(\cdot):\R^M\rightarrow\R^N$ be continuous, the external input $x_t\in\R^{I}$ be i.i.d.\ Gaussian random variables, the state vector $s_t\in \R^M$, and the outputs $y_t\in \R^N$ $(t=1,\cdots, T)$. Then any discrete, stochastic dynamical system of form \eqref{differential_sde} can be pathwisely approximated to an arbitrary accuracy by an RNN model of type \eqref{one_layer_RNN} asymptotically almost surely.
\end{proposition}

A mathematically rigorous statement of Proposition \ref{UAT_stochastic_input} is given as Theorem \ref{UAT_rigorous} in Appendix \ref{app1:proof}, which can be proved using a probabilistic variant of the method proposed by Sch\"afer and Zimmermann \cite{schafer2006recurrent}. The detailed proof is rather technical, hence will also be deferred to Appendix \ref{app1:proof}. An intuitive explanation of why the probability of finding a suitable RNN that approximates \eqref{differential_sde} is only asymptotically 1 is that the key estimate which leads to Sch\"afer and Zimmermann's deterministic UAT (Theorem \ref{UAT_deterministic_input} in Appendix \ref{app1:proof}) is based on the fact that the finite-step input vector $x_t$ can be bounded in a compact domain of $\R^I$. Since the Gaussian random input $x_t\in \R^{I}$ is not compactly supported, but only asymptotically compactly supported, we can only find its universal approximation asymptotically almost surely. This discussion also implies the following corollary:

\begin{proposition}\label{UAT_stochastic_input_compact}
(UAT for RNN with compactly supported stochastic input). Let $g(\cdot):\R^M\times\R^I\rightarrow \R^M$ be continuously differentiable, $h(\cdot):\R^M\rightarrow\R^N$ be continuous, the external input $x_t$ be i.i.d.\ random variables with compact support, the inner state $s_t\in \R^M$, and the outputs $y_t\in \R^N$ $(t=1,\cdots, T)$. Then any discrete, stochastic dynamical system of form \eqref{differential_sde} can be pathwisely approximated to an arbitrary accuracy by an RNN model of type \eqref{one_layer_RNN} almost surely.  
\end{proposition}

\begin{proof}
The proof is easy to obtain following the above arguments and Appendix \ref{app1:proof}.
\end{proof}
As an example, consider i.i.d.\ stochastic input $x_t$ being uniformly distributed in $[a,b]^I$. Then, with probability 1 one can find an RNN model of type \eqref{one_layer_RNN} that accurately approximates open stochastic dynamics \eqref{differential_sde}. For the proposed RNN with stochastic inputs, what the RNN learns is the deterministic update rule that matches the input $x_t$ with the output $y_t$. This is the fundamental reason why the proof of the UAT for RNN with deterministic inputs can be modified to obtain the UAT for RNN with stochastic inputs.

\paragraph{UAT for SDEs.}
The above UATs for RNN with stochastic input can be immediately applied to address the learning and approximation problem of SDEs. Consider It\^o's diffusion on $\R^d$:
\begin{align}\label{Ito_SDE}
    \d X(t) = b\left(X(t)\right)+\sigma\left(X(t)\right)\d\W(t),
\end{align}
where $b(X(t))\in \R^{d}$ is the vector field, $\sigma(X(t))\in \R^{d\times m}$ is the diffusion matrix, and $\W(t)\in \R^{m}$ is the standard Wiener process. Any (explicit) finite difference scheme corresponding to It\^o's diffusion can be written in the form of \eqref{differential_sde} with i.i.d.\ Gaussian input. For instance, the Euler--Maruyama scheme is given by
\begin{align}\label{EM_scheme}
    X(t+\Delta t)&=X(t)+\Delta t\,b\left(X(t)\right)+\sigma\left(X(t)\right)\sqrt{\Delta t}\xi(t)\nonumber\\
    &=g(X(t),\xi(t),\Delta t),
\end{align}
where $\xi(t)$ are i.i.d.\ standard normal random variables. For a fixed $\Delta t$, \eqref{EM_scheme} corresponds to the update rule for the state vector in \eqref{differential_sde} where $x_t=\xi(t)$. Any phase space observables of the stochastic system $y_t=h(X_t)$ can be the output. Under some mild conditions, the Euler--Maruyama scheme \eqref{EM_scheme} is proven to be pathwise convergent to the exact solution of SDE \eqref{Ito_SDE} as $\Delta t\rightarrow 0$ \cite{gyongy1998note,kloeden2007pathwise}. With this result, we can actually obtain the following UAT for It\^o's diffusion: 

\begin{proposition}\label{UAT_Ito_diffusion}
(UAT for the RNN approximation of It\^o's diffusion) Suppose $b(x)$ and $\sigma(x)$ are locally Lipschitz, then the exact solution in a finite time grid to It\^o's diffusion \eqref{Ito_SDE} can be pathwisely approximated to an arbitrary accuracy by an RNN model of the type \eqref{one_layer_RNN} asymptotically almost surely, if we replace $x_t$ by $\xi(t)$.
\end{proposition}

A formal statement of Proposition \ref{UAT_Ito_diffusion}, \ie Theorem \ref{UAT_rigorous_SDE}, and its proof are again provided in Appendix \ref{app1:proof}. Since any time-inhomogeneous SDE admits a time-homogeneous extended dynamics by choosing $t=Y(t)$ as another state variable. Therefore, the universal approximation result naturally extends to time-inhomogeneous SDEs. Similarly, if a non-Markovian SDE admits a suitable embedded Markovian dynamics representation, one can approximate it with RNN model \eqref{one_layer_RNN} by using the latter representation. As an example, consider the generalized Langevin equation (GLE) \cite{lei2016data} that is frequently used in the coarse-grained modeling of large-scale molecular systems:
\begin{align}\label{gle}
\begin{dcases}
    \dot {q}&=p,\\
    \dot p&=F(q)-\int_0^t K(t-s)\,p(s)\,\d s+f(t),
\end{dcases}
\end{align}
where $q(t)$, $p(t)$ are the {\em effective} position and momentum of a coarse-grained particle of unit mass, $F(q)$ is the effective potential energy force, $f(t)$ is the fluctuation force which is often assumed to be a colored Gaussian stochastic process, and the time autocorrelation function of $f(t)$ yields the memory kernel $K(t)=\langle f(t)f(0)\rangle$. GLE \eqref{gle} is a non-Markovian stochastic system because of the time-convolution term $\int_0^t K(t-s)\,p(s)\,\d s$. It is shown in \cite{lei2016data} and many recent works that for many molecular dynamical systems, GLE \eqref{gle} for a coarse-grained particle often admits a Markovian embedded approximation:
\begin{align}\label{gle_markov_embed}
\begin{dcases}
    \d q &= p\,\d t,\\
    \d p &= \left[F(q)+Z^Ts\right]\d t,\\
    \d s &= \left[Bs-QZp\right]\d t+\d\W(t),
\end{dcases}
\end{align}
where the vector $s$ consists of auxiliary variables whose length depends on the order of approximation, and $Z,B,Q$ are the corresponding auxiliary matrices. For such Markovian embedded dynamics, the proposed RNN has the capacity to approximate its output $q(t)$ and $p(t)$ according to Proposition \ref{UAT_Ito_diffusion}.

\section{Statistics-Informed Neural Network}
\label{sec:SINN_architecture}

The UAT shows the potential of RNNs in simulating stochastic dynamics at the large-width limit. In this section, we put this theoretical insight into practice as a framework called the statistics-informed neural network (SINN) to learn stochastic dynamics from data. The main structure of SINN can be briefly summarized as follows. We use the long short-term memory (LSTM) architecture, a specific type of RNN, to capture non-Markovian memory effects a potential stochastic system might have. These LSTM cells take i.i.d.\ random sequences as input and generate ensembles of stochastic time series. A set of training algorithms and statistics-based loss functions are devised to train \SINN to reproduce the statistical characteristics of a target stochastic system.

Before we introduce the specific way to construct SINN, it is worth clarifying our approach. In this paper, we do not seek to construct NNs which are merely an implementation of established theoretical results in \Cref{sec:UAT_for_stochastic}. A wide enough, one-layer neural network has the universal approximation property while is hardly useful in practice. Instead, we use a multi-layer deep neural network to design neural network. Deep neural networks with multiple layers are proven to have many successful applications while the convergence proof is far out of reach. For this reason, we investigate numerical convergence to the target stochastic dynamics in terms of statistics, which corresponds to the weak convergence instead of the pathwise convergence \cite{kloeden2007pathwise}. It is noted that the UAT result should be understood as a justification of the modeling capacity of stochastic RNN, and the way we embed random noise into the system.

\subsection{Model Architecture}
\label{section:model-architecture}

A graphical illustration of the \SINN architecture is shown in \Cref{fig:lstm-mlp-architecture}. The network consists of a multi-layer LSTM component to learn the temporal dynamics of stochastic processes, and a dense layer attached to the output gate of the LSTM as a `read-out' device. Dropout layers can be optionally placed between the layers to control overfitting.

As inspired by the UAT, we use a stream of i.i.d.\ random numbers as the input to the model, which only carries out deterministic operations, in order to generate different  realizations of a stochastic process. The preferred distributions for the input noise are the ones with maximum entropy, \ie, normal distributions for outputs with infinite support, uniform distributions for outputs with compact support, and exponential distributions for outputs with support on $\R^+$. From one perspective, the maximum entropy principle implies that this is the best choice when we assume minimum prior knowledge about the stochastic process. From an alternative perspective, the i.i.d.\ noise sequences can be viewed as the entropy source for \SINN, which in turn can be viewed as a transformer between the input and output stochastic processes. Since information can be lost during the calculation, the maximum entropy distributions help to ensure that the transformation process will not starve of entropy.

\begin{figure}
    \centering
    \includegraphics{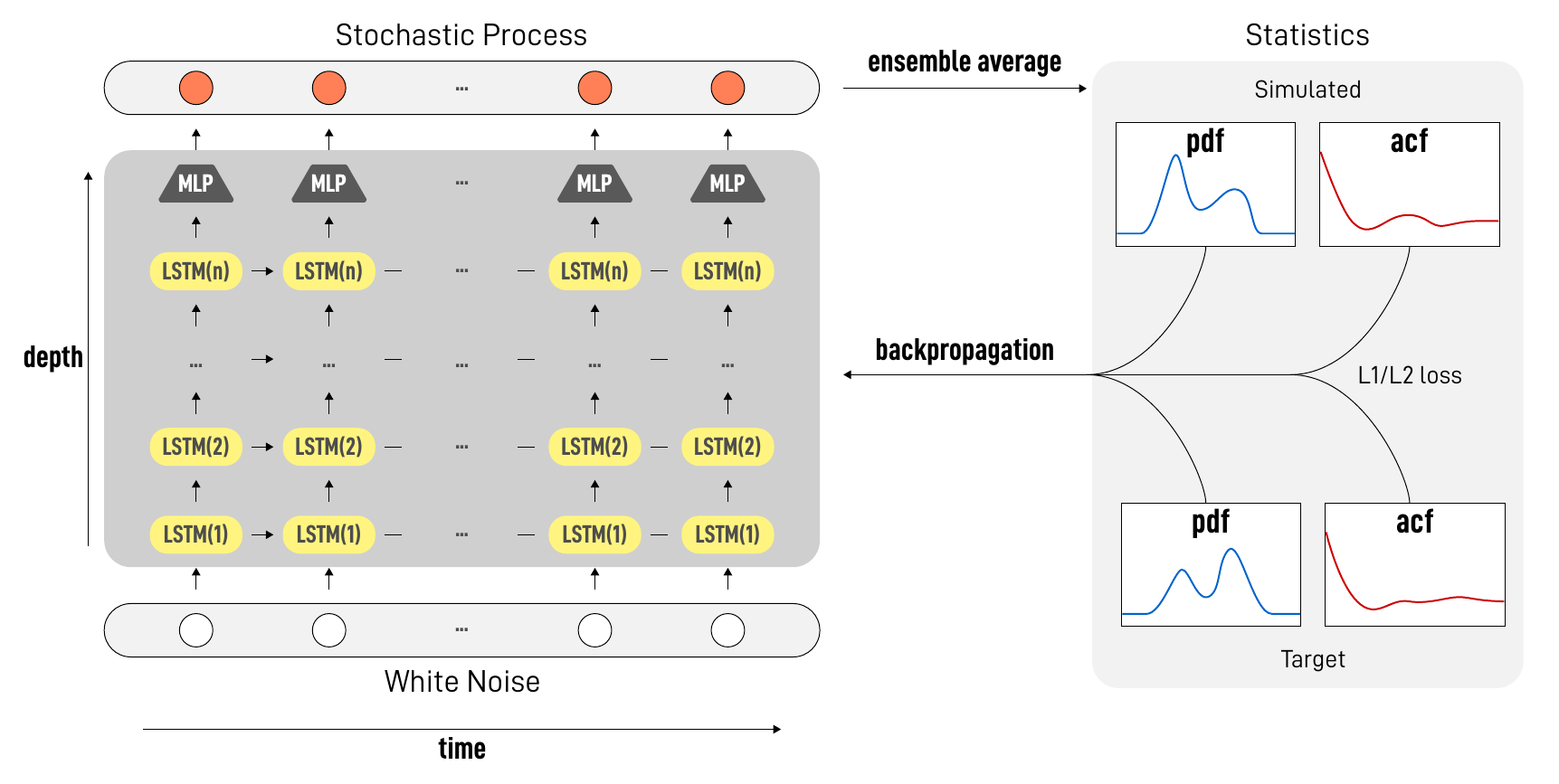}
    \caption{SINN architecture}
    \label{fig:lstm-mlp-architecture}
\end{figure}

Denoting the white noise sequence as $\xi_t$, the forward pass of the first LSTM layer in SINN can be written as:
\begin{align}
f_t^{(1)} &= \sigma_g\left(W_{f} \xi_t + U_{f} h_{t-1}^{(1)} + b_f\right), \\
i_t^{(1)} &= \sigma_g\left(W_{i} \xi_t + U_{i} h_{t-1}^{(1)} + b_i\right), \\
o_t^{(1)} &= \sigma_g\left(W_{o} \xi_t + U_{o} h_{t-1}^{(1)} + b_o\right), \\
\tilde{c}_t^{(1)} &= \sigma_c\left(W_{c} \xi_t + U_{c} h_{t-1}^{(1)} + b_c\right), \\
c_t^{(1)} &= f_t^{(1)} \circ c_{t-1}^{(1)} + i_t^{(1)} \circ \tilde{c}_t^{(1)}, \\
h_t^{(1)} &= o_t^{(1)} \circ \sigma_h(c_t^{(1)}),
\end{align}
where $\tilde c_t^{(1)}$ is the cell input activation, $c_t^{(1)}$ is the cell state, and $f_t^{(1)}$, $i_t^{(1)}$, $o_t^{(1)}$ are the forget gate, the input gate, and the output gate of the first layer, respectively \cite{hochreiter1997long}. For a subsequent layer $i$, the previous-layer output $h_t^{(i-1)}$ replaces $\xi_t$ as the input. The final output $\chi_t$ is then calculated as $\chi_t=W_mh_t^{(n)}$, where $h_t^{(n)}$ is the output of $n$-th LSTM layer, $W_m \in \mathbb{R}^{x}$,  and $x$ is the size of the output vector $\chi_t$. Here we emphasize that only the input sequence is random, while the entire \SINN model itself is deterministic. This makes training of the network very efficient and straightforward.

\subsection{Loss Function}

Rather than seeking a pathwise approximation to the stochastic dynamics, we attempt to match various density functions and the statistics of the trajectories obtained from SINN with those for the target processes. By doing so, we avoid tracking and storing the input Gaussian white noise used in generating the target processes. Moreover, the input noise sequences of SINN can also be temporally coarse-grained if the loss function is measured on the coarse-grained target processes.

\paragraph{Autocorrelation Function.}
The autocorrelation function (ACF) of a sequence $X$ is a deterministic function of lag $\tau$ defined as
\begin{equation}
    \ACF_{X}(\tau) = \frac{\mathrm{E}[X_t  X_{t + \tau}]}{\mathrm{E}[X^2_t]},\label{eq:acf-definition}
\end{equation}  
where the process $X$ is assumed to be zero-mean without loss of generality.

Two common approaches exist for computing the ACF of discrete time series data. The first approach, which we call the \emph{brute force} method, simply uses the definition in \eqref{eq:acf-definition} to compute $\ACF(\tau)$ for every $\tau$. For a sequence of length $n$, computing its ACF up to $\tau = n$ using brute force requires $\mathcal{O}(n^2)$ operations. The second approach, which we call the \emph{fast Fourier transform (FFT)} method, uses the Wiener-Khinchin theorem to efficiently compute the ACF using the Fourier transform of the sequence as
$\mathrm{ACF}_{X} = \mathrm{FFT}^{-1}\left(\mathrm{FFT}(X) \cdot \mathrm{FFT}(X)^*\right)$, where the asterisk ($^*$) denotes the complex conjugate. The FFT approach requires only $\mathcal{O}(n \log n)$ operations for computing the ACF up to $\tau = n$. However, due to the periodicity assumption as implied by the Fourier transformation, the computed ACF can deviate considerably from the true value for large $\tau$. The problem is particularly serious if the ACF does not decay close enough to zero at the length of the sequence data. Hence, the method for computing the ACF must be chosen with discretion while taking the characteristics of the target process into account. In the following numerical examples, both the brute force and FFT approaches are used as appropriate. Both methods permit efficient backpropagation of ACF-based losses to the NN model using popular tensor algebra libraries such as PyTorch~\cite{NEURIPS2019_9015} and JAX~\cite{jax2018github}. For the numerical examples considered, we use a linear combination of $L_1$ and $L_2$ norms to calculate the loss for ACF. Specifically, the loss function we use is
\begin{align}
    \text{Loss}_{acf}=\frac{1}{n}\sum_{\tau \in T}|\ACF_{O}(\tau)-\ACF_{T}(\tau)|+ \frac{1}{n}\sum_{\tau \in T}[\ACF_{O}(\tau)-\ACF_{T}(\tau)]^2,
\end{align}
where $\ACF_{O}(\tau)$ is SINN output ACF function at time $\tau$ and $\ACF_{T}(\tau)$ is the ACF for the target stochastic process. There are flexibilities in terms of the selection of the loss function and hence valid options are not limited to the $L_p$ norms. 

\paragraph{Probability Density Function.}
Binning-based probability density function (PDF) estimators are not differentiable due to the discrete nature of the histogram operation. Therefore, we compute and compare the empirical PDFs of both the target and simulated trajectories using kernel density estimation (KDE):
\begin{equation}
    \widehat{f}_h(x) = \frac{1}{|X|} \sum_i^{|X|} K_h(x - X_i),
\end{equation}
where $K$ is a non-negative \emph{kernel} while $h$ is a smoothing parameter. $K_h(d) \doteq \frac{1}{h} K(\frac{d}{h})$ is the scaled kernel. We use the Gaussian kernel $K^\mathrm{Gauss}(d) = \frac{1}{\sqrt{2\pi}} \exp(-\frac{d^2}{2})$ with a bandwidth parameter $h = |X|^{-\frac{1}{5}}$ \cite{silverman2018density}, where $|X|$ is the length of the sequence $X$. Similarly, the combined $L_1$ and $L_2$ norm are used to calculate the loss for the probability density:
\begin{align}
    \text{Loss}_{pdf}=\frac{1}{n}\sum_{i=1}^n|\widehat{f}^{O}_h(x_i) -\widehat{f}^{T}_h(x_i) |+ \frac{1}{n}\sum_{i=1}^n[\widehat{f}^{O}_h(x_i) -\widehat{f}^{T}_h(x_i)]^2,
\end{align}
where $\widehat{f}^{O}_h(x)$ and $\widehat{f}^{T}_h(x)$ are the estimated PDFs for the output sequence and the target stochastic process, respectively. While the 
Kullback--Leibler divergence appears to be a natural choice for comparing probability distributions, its use of the logarithm operations requires that a large number of output trajectories to be sampled to ensure numeral stability. As such, we are in favor of the $L_p$ norms due to their robustness and the resulting performance benefits.

\section{Numerical Experiments}

\subsection{Training Method}
\label{sec:training_method}

The \SINN model is trained with stochastic gradient descent (SGD) using the Adam optimizer. The learning rate is set to be $10^{-3}$ with $\beta_1=0.9$ and $\beta_2 = 0.999$. Training and validation losses are tracked throughout the training process for every $100$ steps. A {\em new} batch of Gaussian white noise trajectories are generated and used as the model input for every training iteration. This is to ensure that the learned \SINN model is generalizable and not overfitting to a particular realization of the stochastic processes. In our experiments, the training batch and the target data both contain $400$ sequences, while the validation set contain $800$ sequences.

\paragraph{Evaluation of Loss.}
Instead of comparing the ACFs over the entire lag range $1, \ldots, t_\mathrm{max}$, we randomly select a set $T$ of lag values $\tau_1, \ldots, \tau_m$ with $m \ll t_\mathrm{max}$ during each SGD step and compare the ACFs at the selected lags. Typical values of $m$ is around $20$. This procedure is particularly important when the brute force ACF estimator is employed due to its high computational cost.

\paragraph{Input Sequence.}
The white noise sequence, which serves as the input to the \SINN model as described earlier in \Cref{section:model-architecture}, is always created {\em afresh} at each SGD iteration. This refreshing procedure is to ensure that the dynamics SINN learned is independent of any specific realizations of the input random noise. This step is particularly important to guarantee the generality and consistency of SINN's training results.

\paragraph{Validation Sequence.}
The validation data is a fixed number of target sequences that are used to monitor training and detect overfitting. The losses computed on the validation sequence do not participate in backpropagation.

\paragraph{Computational Cost.}
All computations are performed using a workstation with 16 AMD Zen3 cores at 3.0 GHz and one NVIDIA A100 accelerator. A SINN model with 2 LSTM layers each with 25 units can be trained for $1200$ SGD iterations within approximately $1$ minute. Detailed runtime statistics for all numerical examples considered in this paper are summarized in \Cref{Tab:runtime}. Explanations will be presented in following sections.

\begin{table}
\centering
\caption{Wall time measurements for all numerical examples. The `SINN Training' column records the wall time for the SINN model to achieve $\epsilon_T,\epsilon_V<10^{-3}$. Here $\epsilon_T$ and $\epsilon_V$ are the training and validation errors, respectively. The time steps for the Euler--Maruyama scheme are $\Delta t=10^{-2}$ (OU, FPU, Poisson, CG) and $\Delta t=10^{-3}$ (Double-Well). Temporally coarse-grained trajectories with step size $dt=0.2$ are used to train SINN. Hence SINN models all have time scale $dt=0.2$. Other technical details for each example are given in the corresponding section. All wall time values are averaged over 5 simulations. The algorithms are implemented using PyTorch and executed on a workstation with 16 AMD Zen3 cores at 3.0 GHz and one NVIDIA A100 accelerator.}
\begin{tabular}{lclcc}
\toprule
\multirow{2}{*}{\textbf{Stochastic Process}} & \multicolumn{1}{c}{\textbf{Training}} && \multicolumn{2}{c}{\textbf{Generate 5000 trajectories up to $T=1000$}} \\
\cmidrule{2-2} \cmidrule{4-5}
& SINN && Euler--Maruyama & SINN \\
\midrule
OU          & 33 s     && 13 s     & 1.4 s      \\
FPU         & 68 s     && 25 s     & 1.4 s      \\
Poisson     & 150 s    && 13.8 s   & 1.4 s      \\
CG          & 268 s    && 3430 s   & 1.4 s      \\
Double-Well & 779 s    && 252 s    & 1.4 s      \\

\bottomrule
\end{tabular}

\label{Tab:runtime}
\end{table}

Before presenting numerical results, we comment in advance on the modeling advantages of SINN, which echos the three architectural differences we mentioned in Introduction. First, SINN is essentially equation-free since the modeling and training of SINN use no equations. This feature allows the generated dynamics to have {\em tunable} coarse-grained time scales, which makes it particularly suitable for capturing the long-time behavior of stochastic systems. Further discussion in this regard is provided in \Cref{sec:APP_reduced}. Second, SINN learns a {\em deterministic} update rule for SDEs which is similar to the Euler--Maruyama scheme \eqref{EM_scheme}. Thus, it is very natural to do time-domain extrapolation and expect a certain predictability of SINN. Third, the convergence we seek is defined in terms of statistical moments and probability measure. In many cases, such as the transition dynamics simulation in \Cref{sec:APP_reduced}, it can be shown that such convergence is already enough to capture the physics we are interested in.

\subsection{Validation Cases}
\label{sec:test_cases}

We present three test cases here to show that SINN can well approximate Gaussian and non-Gaussian stochastic dynamics. Detailed runtime statistics are listed in \Cref{Tab:runtime}.

\subsubsection{Ornstein--Uhlenbeck Process}
Consider the Ornstein--Uhlenbeck (OU) process given by the following SDE:
\begin{align}\label{OU_process}
\frac{dx}{dt}=-\theta x+\sigma\xi(t),
\end{align}
where $\sigma$ and $\theta$ are positive parameters and $\xi(t)$ is standard Gaussian white noise with correlation function $\langle\xi(t)\xi(s)\rangle=\delta(t-s)$. The OU process is ergodic and admits a stationary, \ie equilibrium, Gaussian distribution $\N(0, \sigma^2/2\theta)$. In addition, the ACF of $x(t)$ at equilibrium is an exponentially decaying function $C(\tau)=\langle x(t+\tau)x(t)\rangle=\frac{\sigma^2}{2\theta}e^{-\theta\tau}$.  
With the parameter values $\sigma=0.5$ and $\theta=1$, we generate approximated dynamics for $x(t)$ using the proposed SINN architecture with two LSTM layers each with one unit. The stationary ACF and the equilibrium PDF, which are analytical, are used as the target by the loss function to train the NN parameters.

\begin{figure}
\centerline{
\includegraphics[height=5cm]{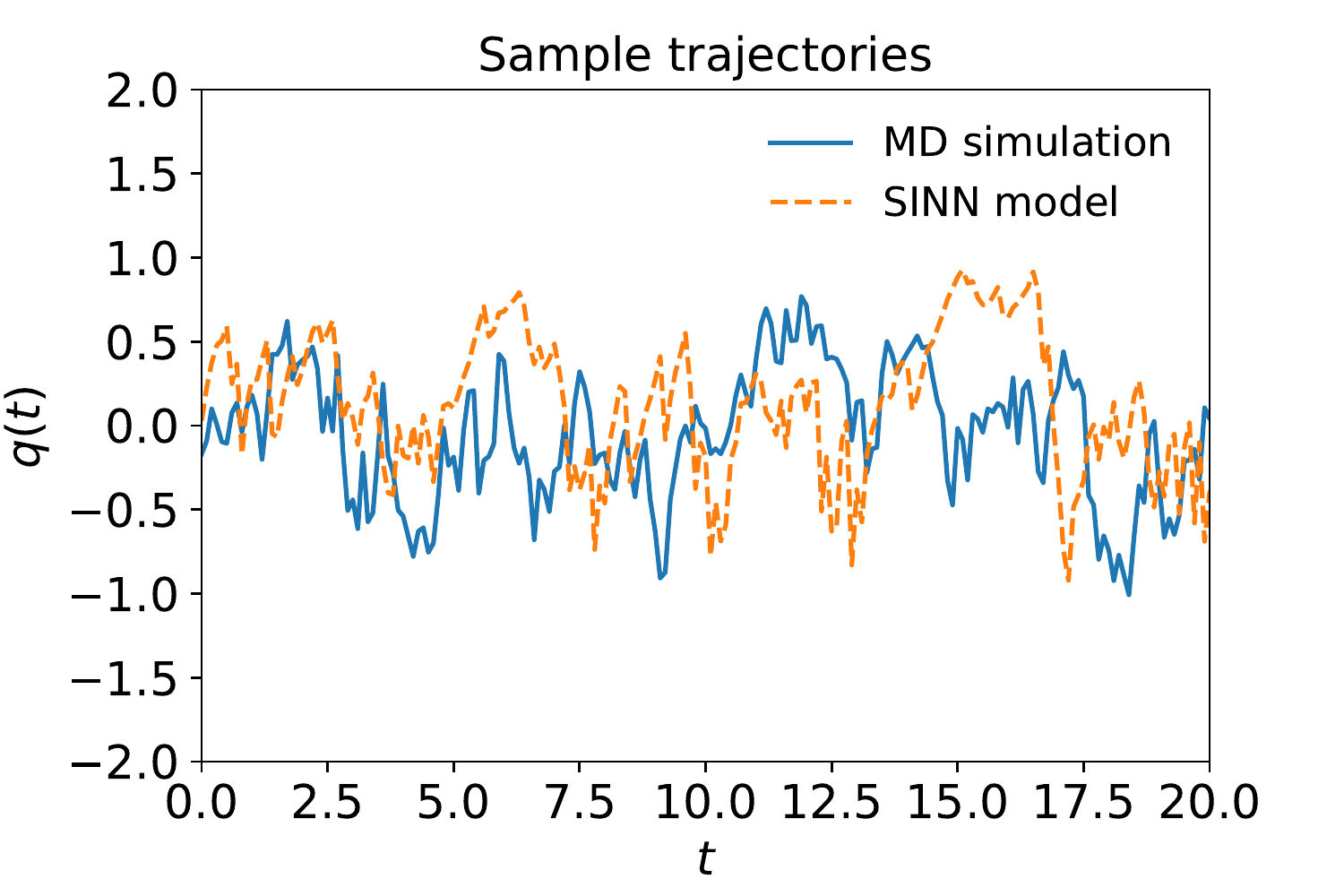}
\includegraphics[height=5cm]{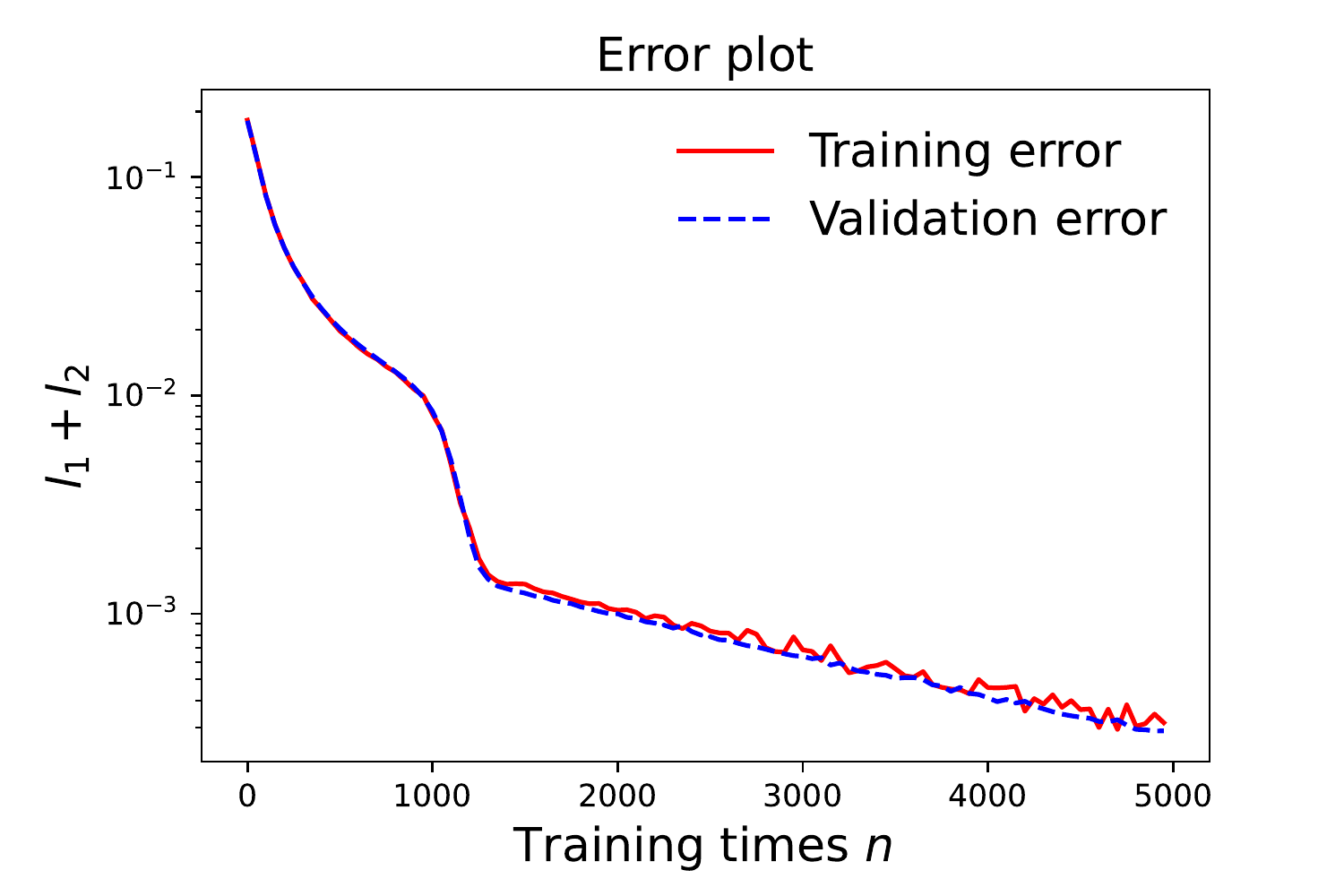}
}
\centerline{
\includegraphics[height=5cm]{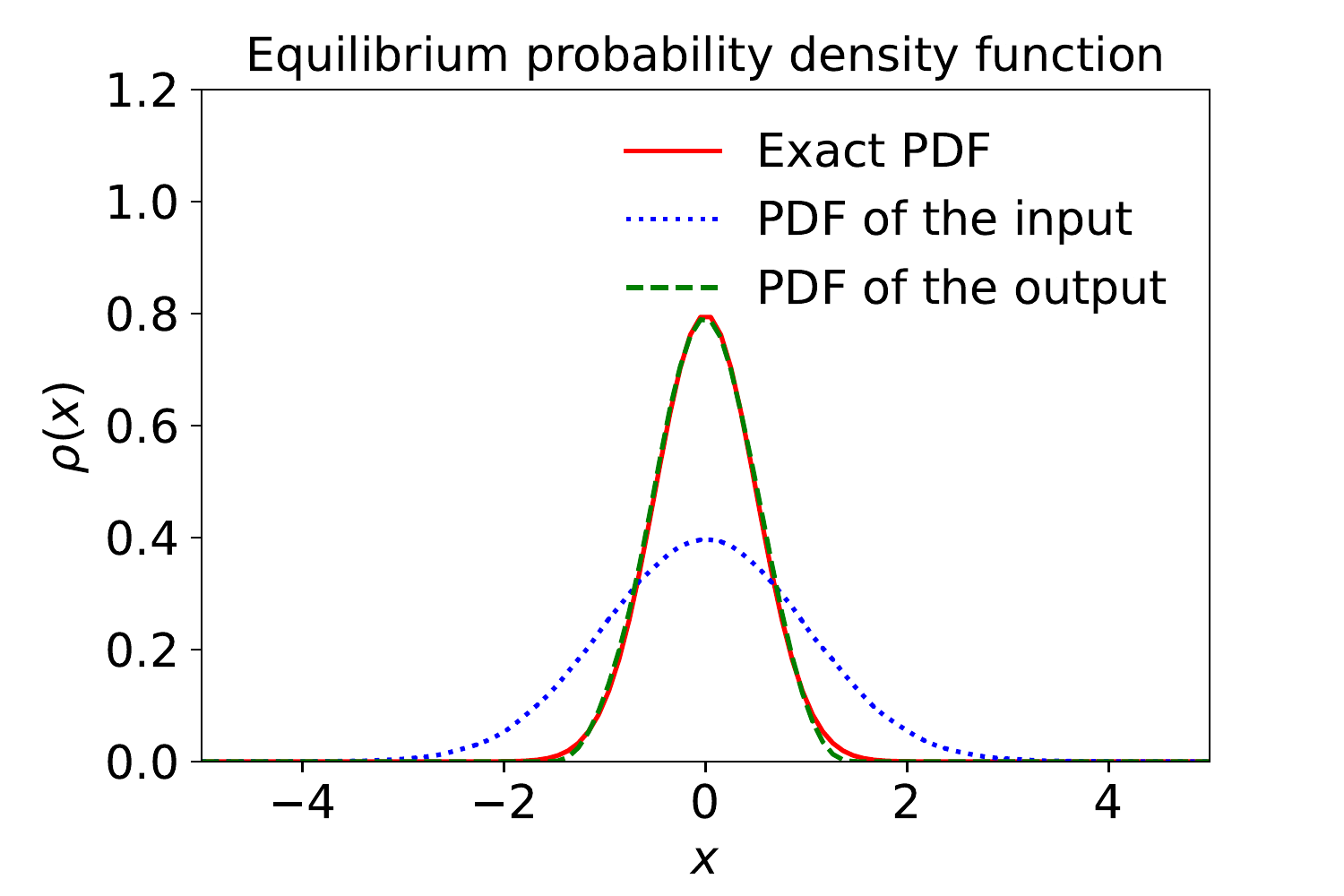}
\includegraphics[height=5cm]{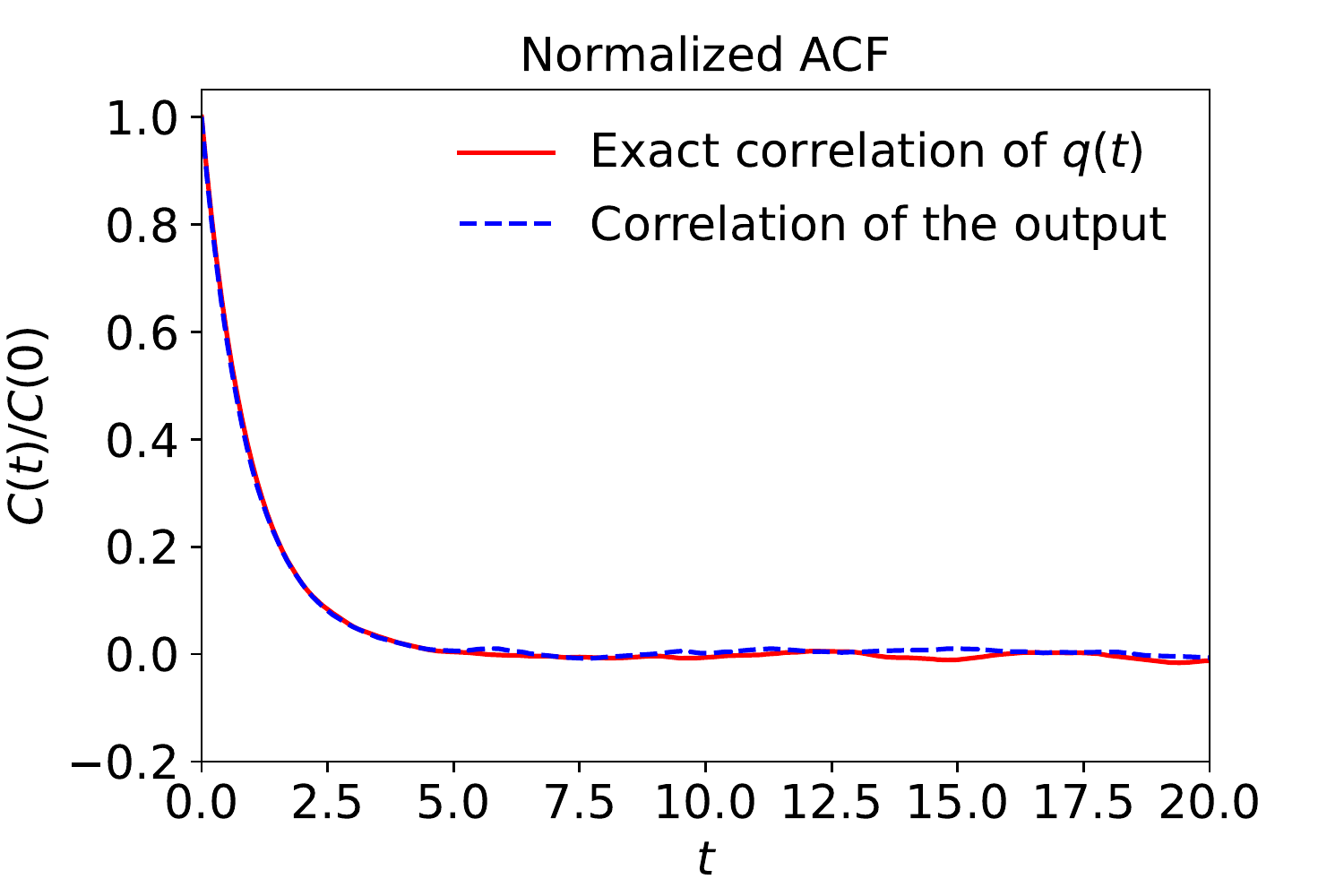}
}
\caption{Comparison of the dynamics of $q(t)$ generated by MD simulation and the \SINN model. The MD simulation results of the sample trajectories (Top Left) are obtained using the Euler--Maruyama scheme for \eqref{OU_process} with step size $\Delta t=10^{-3}$. The target processes are temporally coarse-grained sample trajectories of $q(t)$ with step size $dt=0.1$. Note that sample trajectories simulated by SINN thus have natural coarse-grained time scale $dt\gg\Delta t$. The output statistics (PDF and ACF) are evaluated by taking the ensemble average over 5000 SINN trajectories which are generated using a new set of Gaussian white noise as the SINN input.
}
\label{fig:ou_Sample_path_compare}
\end{figure}

\Cref{fig:ou_Sample_path_compare} clearly shows that the statistics of the OU process is faithfully reproduced by the trajectories simulated by SINN. Here we note that the time step $dt=0.1$ of SINN is much larger than the MD time step $\Delta t=10^{-3}$. This temporally coarse-grained feature of SINN is one of its main characteristics which makes it particularly useful in rare-event simulations as will be detailed in \Cref{sec:APP_reduced} and \ref{sec:asssement}).
The error plot shows that the generalization error gets smaller while remaining at the same magnitude with respect to the training error during the training process. This indicates that over-fitting does not happen. Since the correlation function and the equilibrium probability density {\em uniquely} determine a Gaussian process, we conclude that the stochastic process generated by \SINN faithfully represents the dynamics of the OU process. For timing results as given in \Cref{Tab:runtime}, a slightly different SINN architecture with 2 LSTM layers each with 5 hidden units is adopted to establish better comparability, while other parameters are exactly the same.

\subsubsection{Langevin Dynamics}

Consider the Langevin dynamics for an anharmonic oscillator:
\begin{align}\label{langevin_dynamics}
\begin{dcases}
    \dot q &=p,\\
    \dot p &=-V'(q)-\gamma p+\sigma \xi(t),
\end{dcases}
\end{align}
where $V(q)=\frac{\alpha}{2}q^2+\frac{\theta}{4}q^4$ is the Fermi-Pasta-Ulam (FPU) potential and $\xi(t)$ is Gaussian white noise. Parameters $\gamma$ and $\sigma$ are linked by the fluctuation-dissipation relation $\sigma=(2\gamma/\beta)^{1/2}$, where $\beta$ is proportional to the inverse of the thermodynamic temperature. Langevin dynamics for the FPU oscillator admits the Gibbs-form equilibrium distribution $\rho_{eq}\propto e^{-\beta H}$, where $H=\frac{p^2}{2}+V(q)$. The parameters $\alpha=\beta=\theta=\gamma=1$ and $\sigma=\sqrt{2}$ are chosen for numerical simulations. We use the same SINN model as in the OU process example with two LSTM layers and one hidden state to generate approximated dynamics for $q(t)$. Unlike the case for the OU process, here we do not have an analytical expression for the ACF of $q(t)$. Hence, an empirical estimate of the ACF of $q(t)$ is obtained from numerical solutions to \eqref{langevin_dynamics}. Since $q(t)$ is no longer a Gaussian process, its PDF and ACF {\em cannot} completely characterize its dynamics. To ensure the validity of the model, we add the stationary ACF for $q^2(t)$ as an extra training target for the neural network. The results as presented in \Cref{fig:fpu_Sample_path_compare} show that the \SINN architecture can well approximate the dynamics of the non-Gaussian process \eqref{langevin_dynamics}. Runtime benchmarks use the same architecture as in the OU process example.

\begin{figure}[t]
\centerline{
\includegraphics[height=5cm]{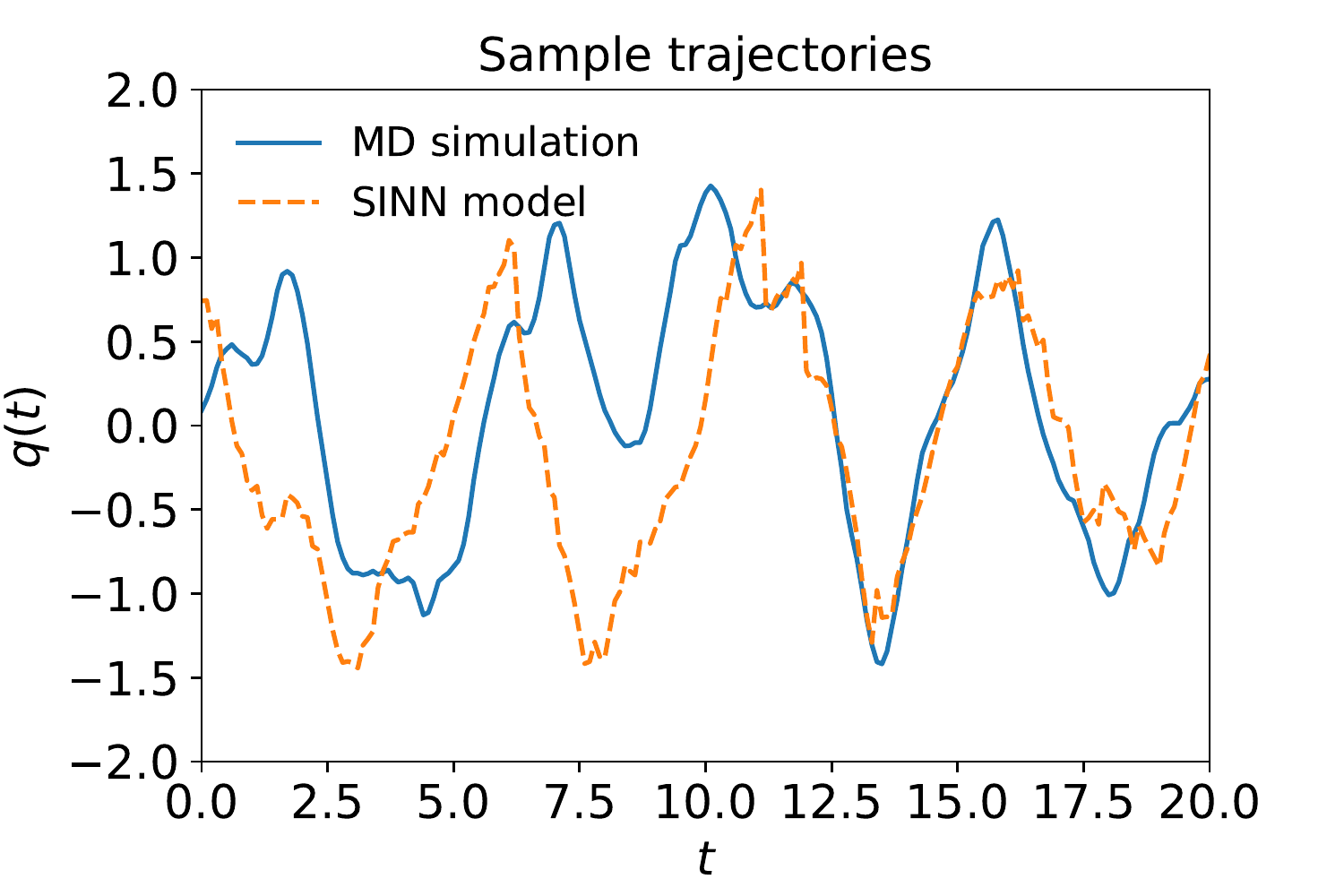}
\includegraphics[height=5cm]{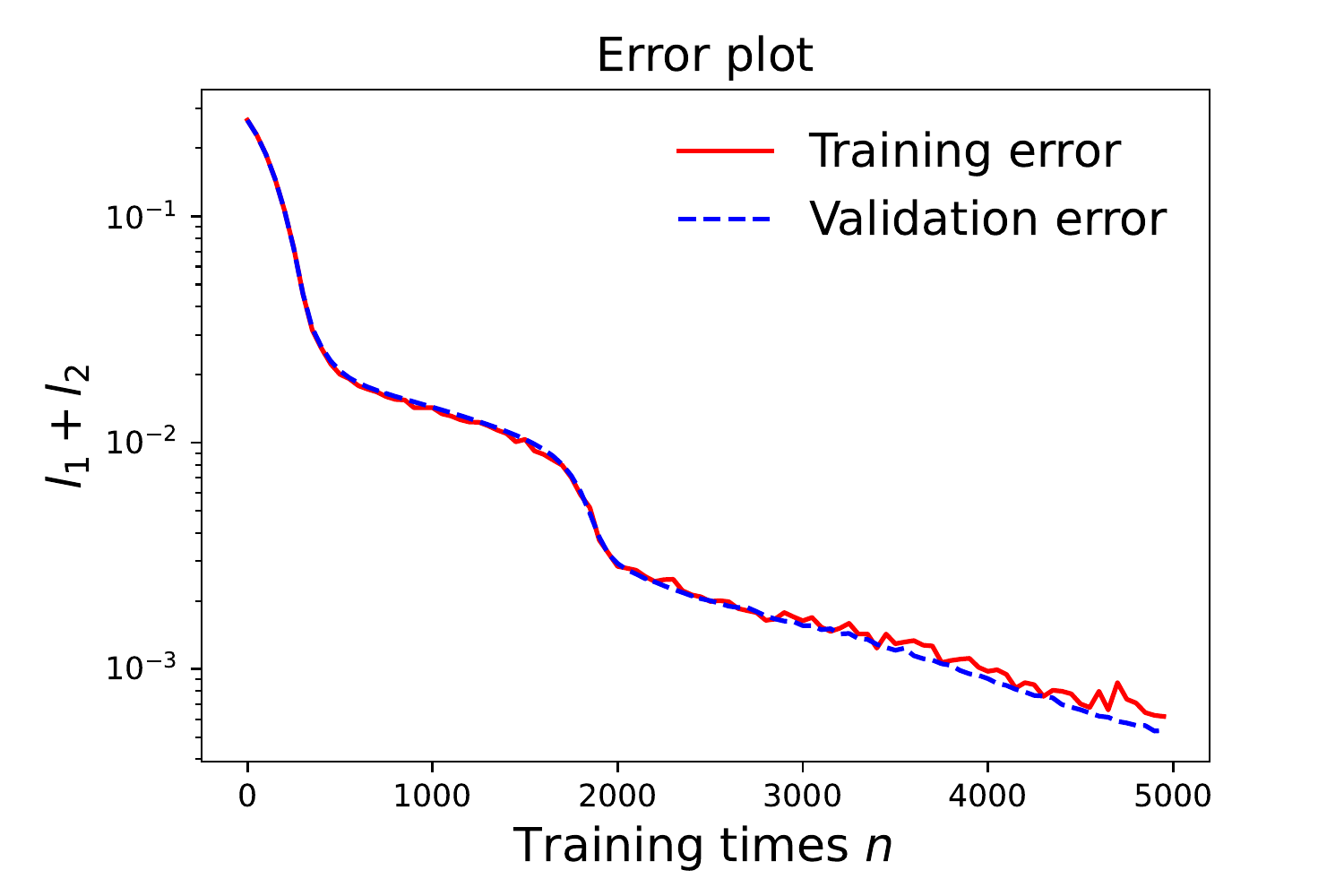}
}
\centerline{
\includegraphics[height=5cm]{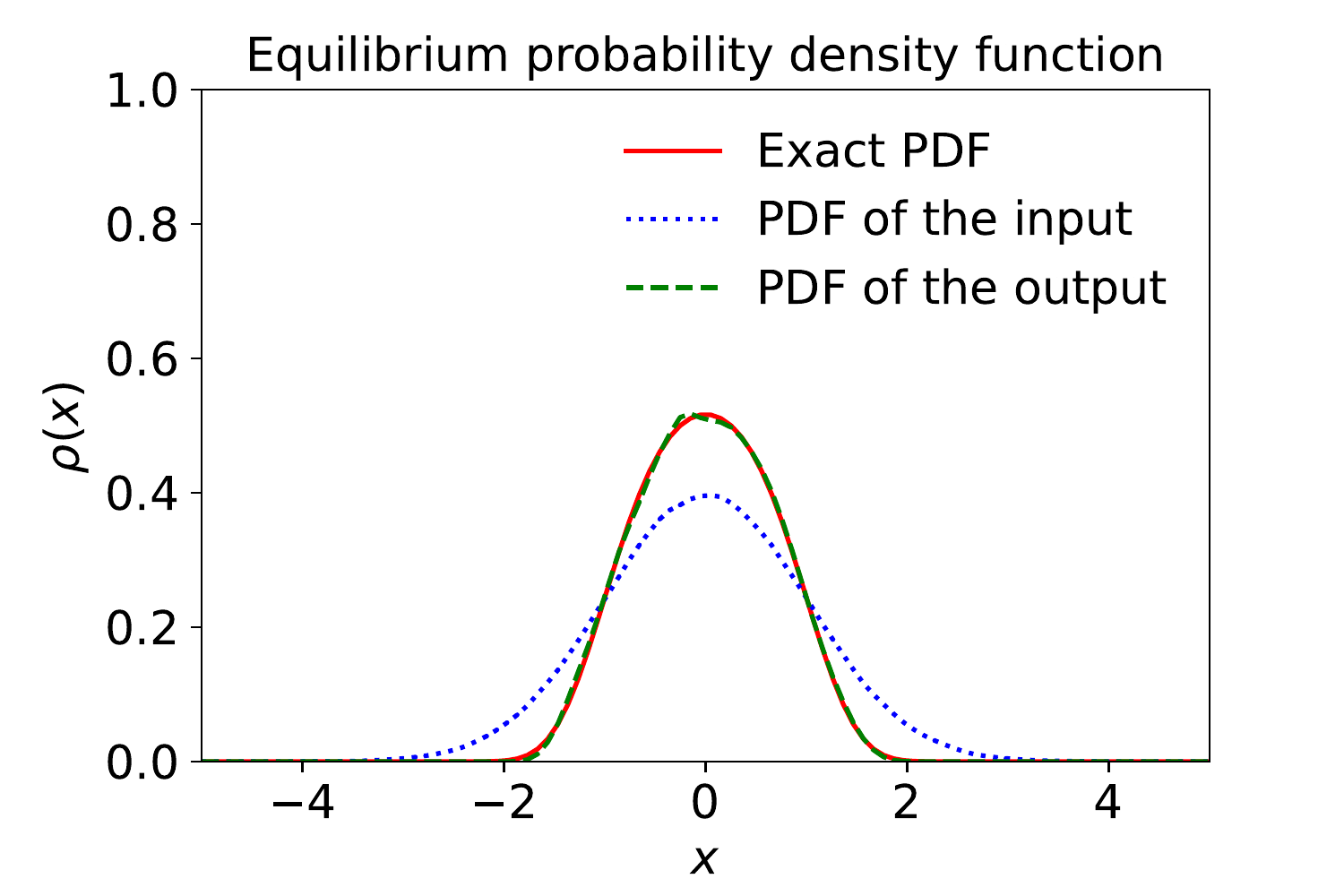}
\includegraphics[height=5cm]{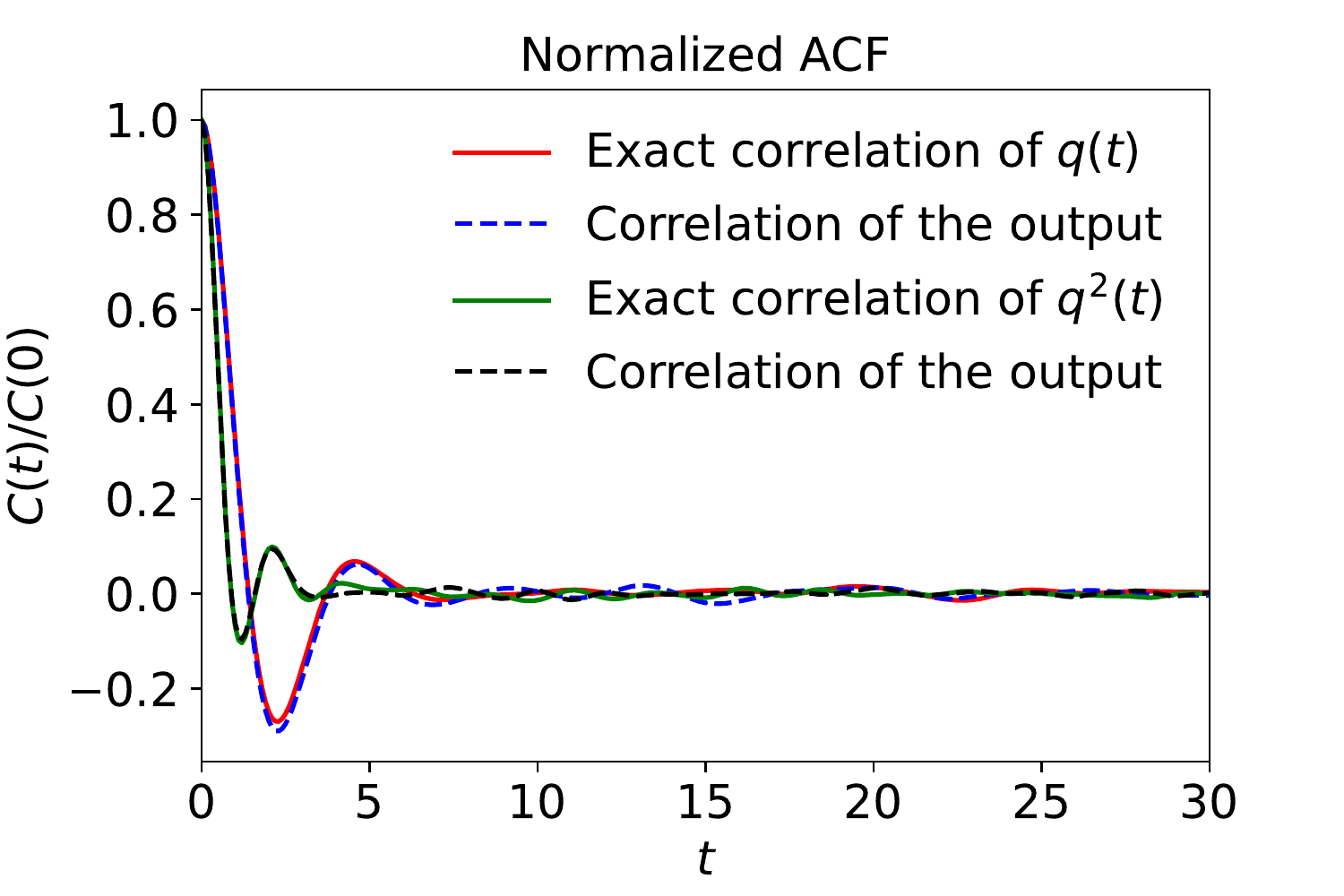}
}
\caption{Comparison of the dynamics of $q(t)$ generated by MD simulation and the \SINN model. The setting is exactly the same as the one used in \Cref{fig:ou_Sample_path_compare} except that we added the ACF $\langle q^2(t)q^2(0)\rangle$ into the total loss function.}
\label{fig:fpu_Sample_path_compare} 
\end{figure}

\paragraph{Remark 1.} Assuming that the Kramers--Moyal expansion \cite{risken1996fokker} for a stochastic process exists, one can continuously improve the approximations to the master equation corresponding to the stochastic process by progressively introducing higher-order moments. This is the reason why we added the ACF for $q^2(t)$ as an additional target to train the model for non-Gaussian dynamics. We note that higher-order moments such as $\langle q^4(t) q^4(0)\rangle$ can also be easily added into the total loss function. Due to this extensibility of SINN, it is fairly simple to include higher-order information so that the generated stochastic process can faithfully approximates that of the original stochastic process. We note that this is not generally guaranteed by established methods in stochastic modeling such as the transformed Karhunen-Lo\'eve or polynomial chaos expansion \cite{phoon2005simulation,sakamoto2002polynomial}.     

\paragraph{Remark 2.} As a data-driven framework, the {\em equation-free} feature of SINN renders it a desirable option for solving reduced-order modeling problems, where the effective dynamics for the low-dimensional resolved observables is generally hidden and has to be extracted from the underlying high-dimensional dynamical systems through coarse-graining procedures. Generally speaking, dimensionality reduction leads to memory effects in the reduced-order dynamics. We emphasize that these effects can
be captured by the LSTM modules of SINN. Langevin dynamics \eqref{langevin_dynamics} provides a good example for this. Here, the system as a whole is Markovian for the state variables $\{q(t),p(t)\}$. However, the reduced-order effective dynamics for the observable $q(t)$ alone is non-Markovian. Using the Mori-Zwanzig framework \cite{zhu2021effective,zhu2021hypoellipticity}, one can derive the following evolution equation for $q(t)$:
\begin{align}\label{sample_gle}
    \frac{\d}{\d t}q(t)=\Omega\,q(t)+\int_0^t\,K(t-s)\,q(s)\,\d s+ f(t),
\end{align}
where $\Omega$ is a modeling constant, $K(t)$ is the memory kernel, and $f(t)$ is the stochastic fluctuation force. In \eqref{sample_gle}, the memory effect is encoded by the convolution integral $\int_0^t\,K(t-s)\,q(s)\,\d s$, where $K(t)$ is generally unknown. \SINN provides a novel mechanism to quantify this complicated memory effect by `storing' it within the LSTM cell state vectors $c_t$, whose update mechanism can be learned through simulation data. The coarse-grained modeling problem considered in \Cref{sub_sec:CG} also provides an example to further illustrate this point.

\subsubsection{SDE Driven by Poisson White Noise}

In order to validate the assertion made in \Cref{section:model-architecture} that SINN can take i.i.d.\ non-Gaussian noise sequences and model stochastic process with support on $\R^+$, we consider the following SDE driven by Poisson white noise:
\begin{align}\label{SDE_poisson}
    \frac{\d x}{\d t}=-bx+\xi(t),
\end{align}
where $\xi(t)=\sum_{i=1}^{n(t)}z_i\delta(t-t_i)$ is a random sequence of $\delta$-pulses. This random pulse is generated as follows. For each time $t$, $n(t)$ satisfies Poisson distribution with probability $P\left(n(t)=n\right)=(\lambda t)^ne^{-\lambda t}/n!$, which counts the number of stimuli that arrive within interval $(0,t]$. $z_i$ are i.i.d.\ exponentially distributed random variables with probability density $\rho(z)=r e^{-rz}$ ($z>0$). For numerical simulations, we choose $b=r=1$ and $\lambda=2$.

SDE \eqref{SDE_poisson} describes the dynamical behavior of a system when randomly perturbed by external stimuli, which is commonly seen in many control systems in electronic engineering and physics. We want to use SINN to generate stochastic processes that recover the statistical features of  $x(t)$. Here we only use observation data $x(t)$ and assume the minimum prior knowledge of its generating mechanism. That is, as shown in \Cref{fig:Poisson_path_ACF}, $x(t)$ is a random jump process and has support on $\R^+$. For this case, we use i.i.d.\ exponentially distributed noise sequence as the input of SINN and train the neural network using the empirical PDF/ACF calculated from sample trajectories of $x(t)$.

\begin{figure}[t]
\centerline{
\includegraphics[height=4cm]{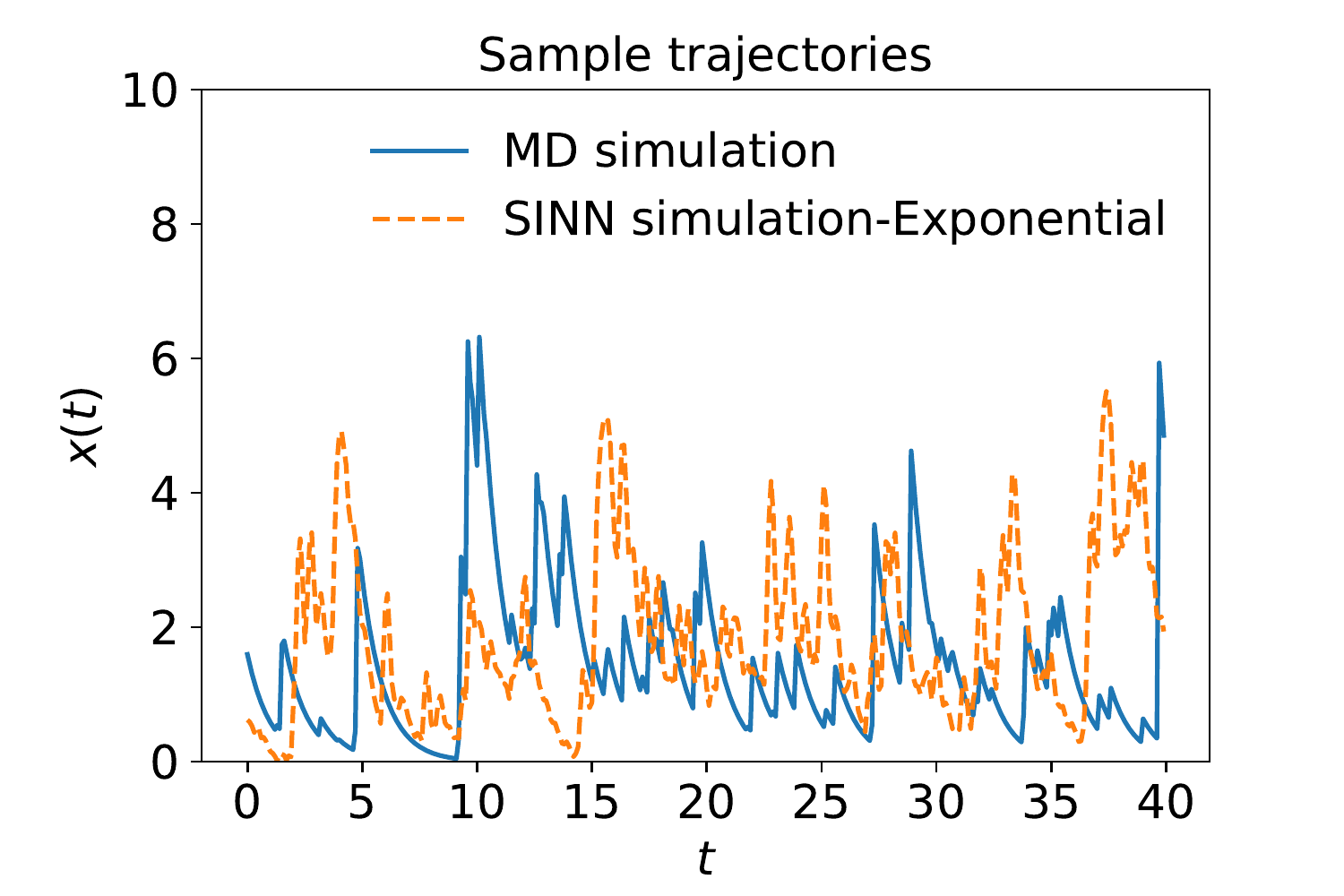}
\includegraphics[height=4cm]{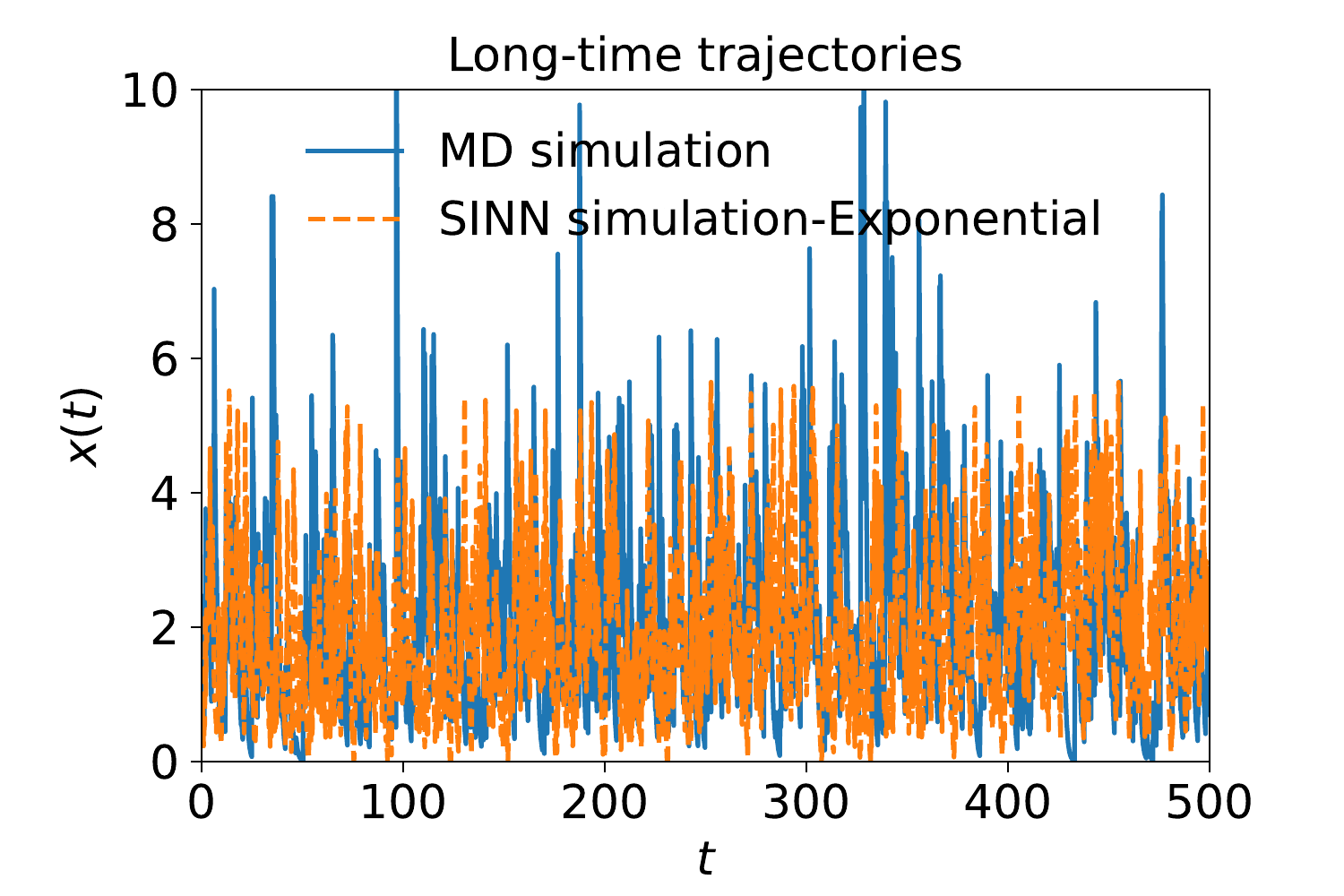}
\includegraphics[height=4cm]{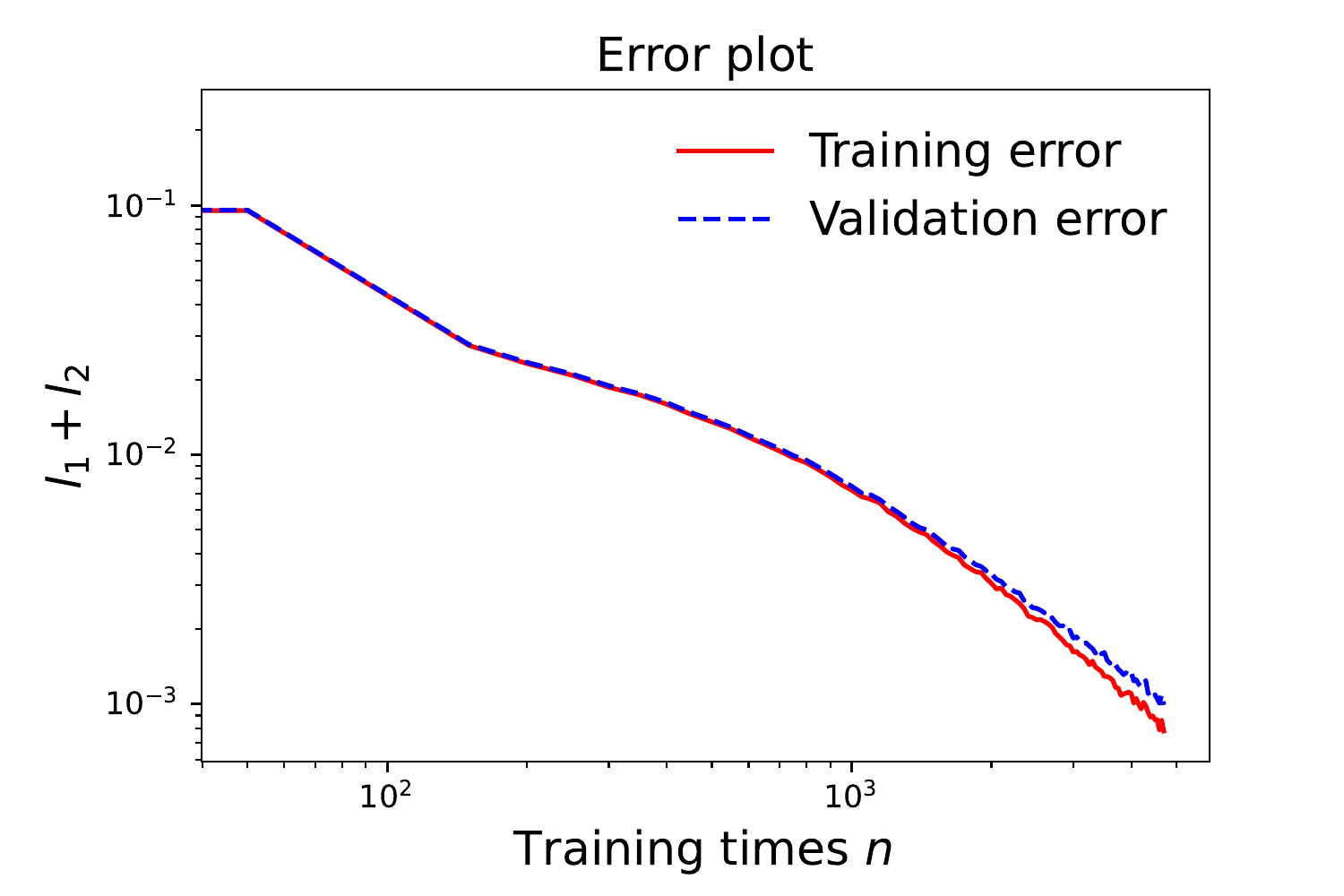}
}
\centerline{
\includegraphics[height=4cm]{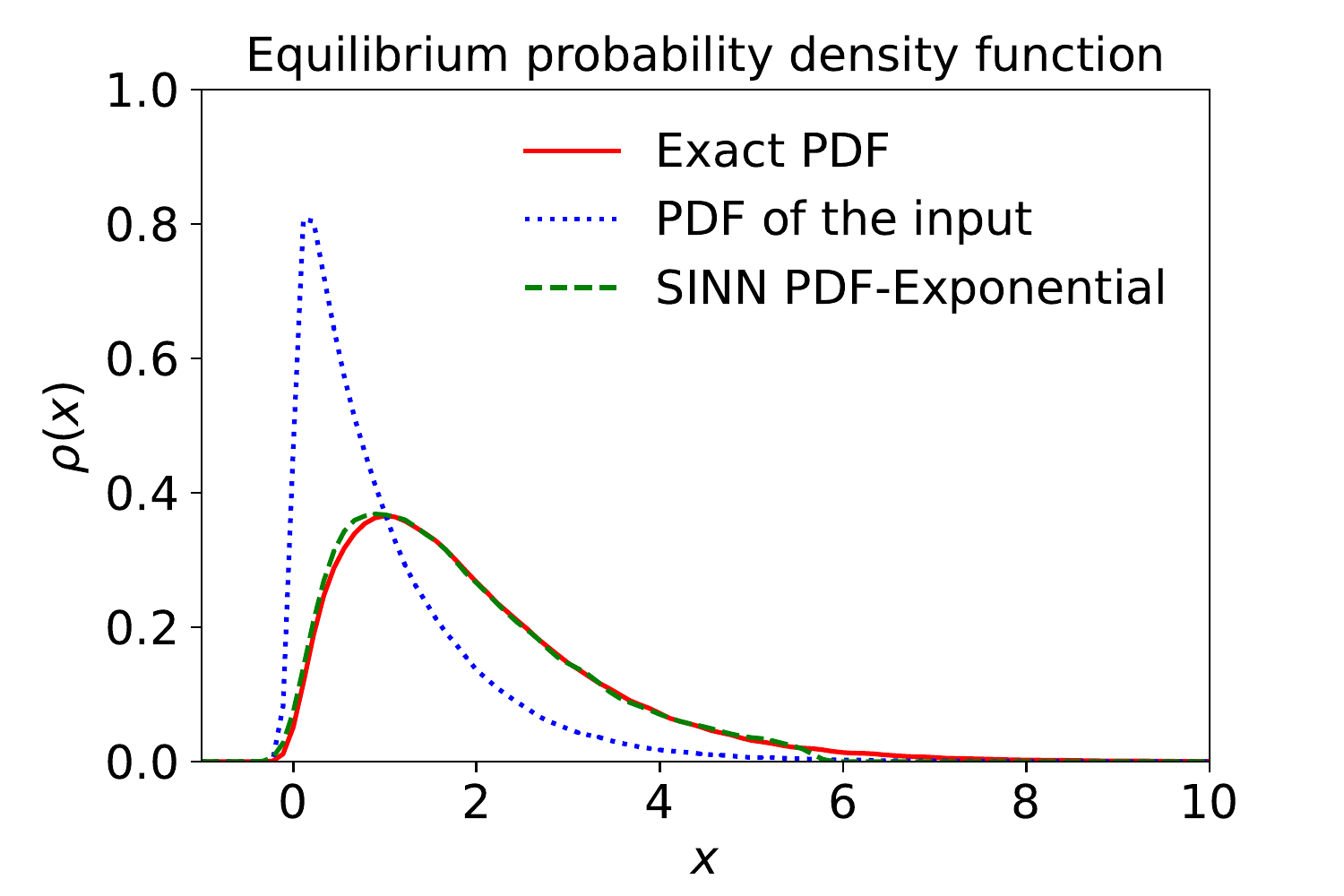}
\includegraphics[height=4cm]{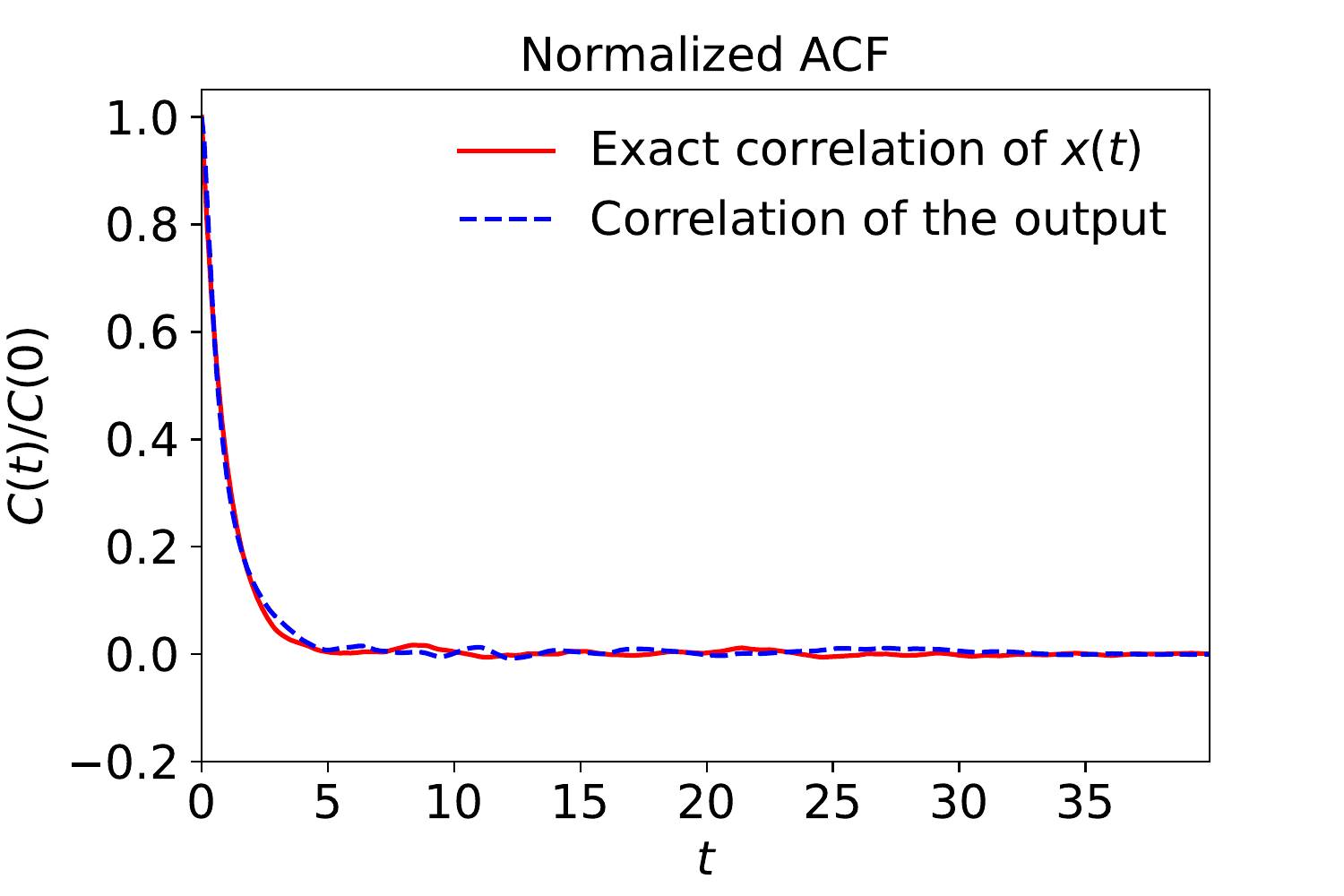}
\includegraphics[height=4cm]{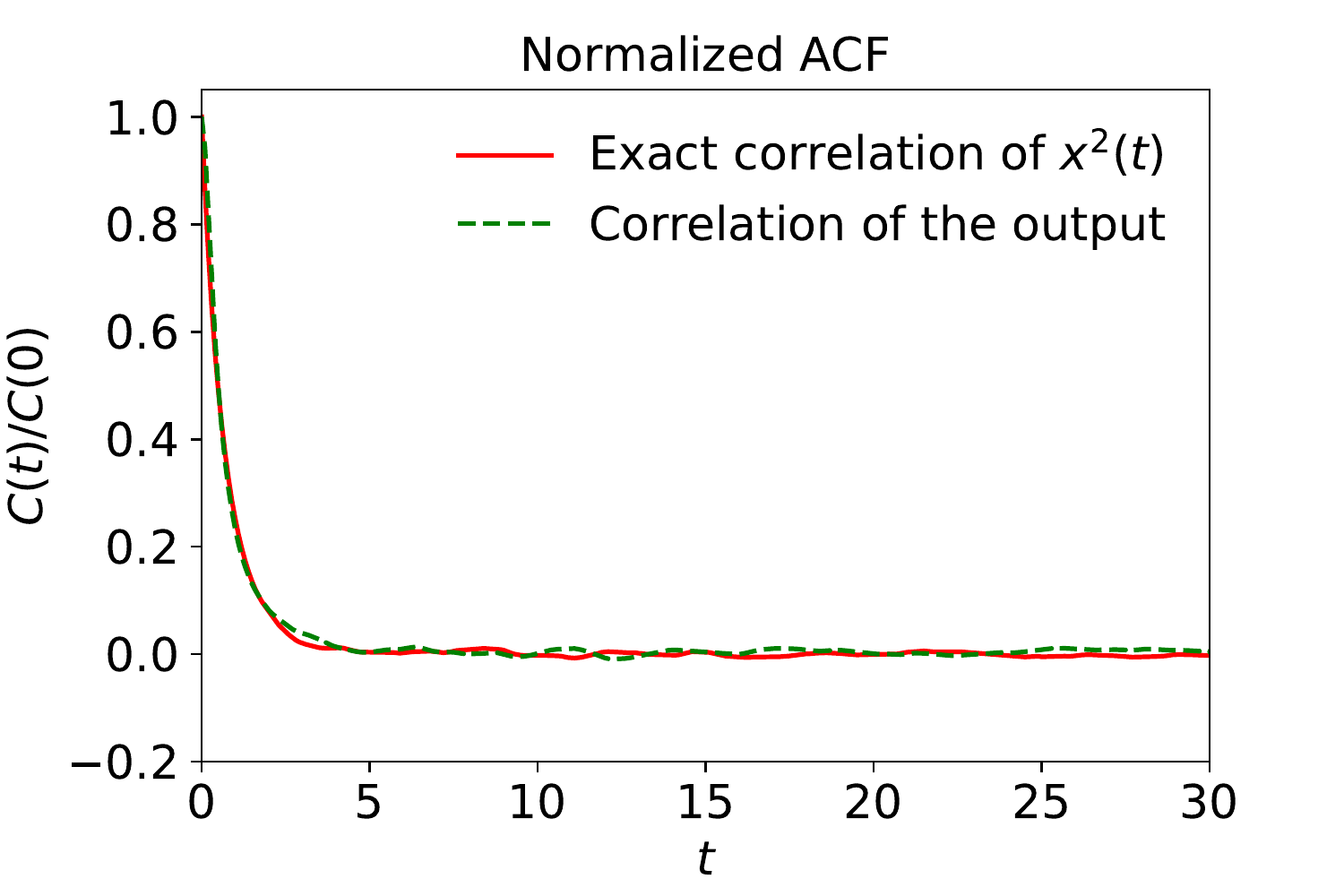}
}
\caption{Comparison of the dynamics of $x(t)$ generated by MD simulation and the \SINN model. The MD simulation results of the sample trajectories (Top Left) are obtained using a first-order finite difference scheme \cite{denisov2009generalized} with step-size $\Delta t=0.01$. The target processes are temporally coarse-grained sample trajectories of $x(t)$ with step size $dt=0.2$. The SINN model has 2 LSTM layers each with 5 hidden units. The input noise sequence satisfies the exponential distribution with $r=1$. Two statistical moments: $\langle x(t),x(0)\rangle$ and $\langle x^2(t)x^2(0)\rangle$, estimated with 400 simulated trajectories of $x(t)$, are used to compose the loss function. The output PDF and ACF are evaluated by taking the ensemble average over 5000 SINN trajectories.}
\label{fig:Poisson_path_ACF} 
\end{figure}

In \Cref{fig:Poisson_path_ACF}, we show that SINN can simulate trajectories of the jump stochastic process using exponentially distributed noise sequence as the input. Other important statistics such as the ACF and the equilibrium PDF are also compared with the ground truth that we obtained through the MD simulation of \eqref{SDE_poisson}. It is remarkable that SINN faithfully recovers the long tail of the equilibrium PDF of $x(t)$. Also, it is noted that we did not use the Poisson distribution to generate the random noise input since this part is assumed to be unknown when training the model. These results clearly demonstrate that SINN can model stochastic dynamics driven by non-Gaussian noise and that the input of the neural network can be accordingly adjusted to be non-Gaussian to accommodate the modeling needs. 

\paragraph{Remark 1.} A similar stochastic jump process was considered in the Neural jump SDE framework \cite{jiaNeuralJumpStochastic2020}. One obvious difference is that SINN seeks convergence in terms of statistics, \eg moments and PDF, whereas the Neural jump SDE targets path-wise convergence. Furthermore, the training process employed in SINN is closer to real-world applications in the sense that {\em only} the trajectory data of $x(t)$, possibly temporally coarse-grained, were used to train the neural network. Moreover, we assumed minimum prior knowledge on the generating mechanism of the observable $x(t)$. In other words, this is a case where the equation of motion for $x(t)$ is completely hidden. This fact will be more obvious with the coarse-grained modeling example considered in \Cref{sub_sec:CG}.

\paragraph{Remark 2.} It is shown in \cite{denisov2009generalized} for SDE \eqref{SDE_poisson} that the governing equation for the transition probability $P(x,t)$ satisfies a generalized Fokker--Planck equation:
\begin{align}\label{GFPE}
    \partial_t P(x,t)=b\partial_xP(x,t)-\lambda P(x,t)+\lambda\int_{-\infty}^{+\infty}r e^{-r(x-y)}\theta(x-y)\,P(y,t)\d y,
\end{align}
where $\theta$ is the Heaviside step function. This equation is not given in the form of the Kramers--Moyal expansion. Nevertheless, since the moments of various orders for $x(t)$ still exist, we expect a high-order Kramers--Moyal expansion to yield a good approximation to \eqref{GFPE}. This is why we added the high-order correlation function $\langle x^2(t)x^2(0)\rangle$ in the loss function when training SINN. It is reasonable to expect that training against more moments will help SINN to generate stochastic trajectories closer to the original process $x(t)$. While we believe SINN can model several common types of stochastic processes, some particular processes such as the L\'evy processes cannot be modeled by the current version of SINN simply because its moments do not exist.\footnote{We thank one of the anonymous reviewers for pointing this out.} A possible remedy for this case might be using instead fractional moments \cite{kuhn2017existence} to define the loss function. This, of course, awaits further investigation.

\section{Applications}
\label{sec:applications}

In this section, we use practical examples to demonstrate the capabilities of SINN to discover hidden stochastic dynamics in given non-Gaussian, non-Markovian stochastic systems. We further verify whether the resulting stochastic model has long-time predictability and numerical stability. The applications we consider here are the coarse-grained modeling of a molecular system and the study of transition dynamics and rare events. Using these examples, we show that SINN has several computational and modeling advantages over the traditional stochastic modeling methods.

\subsection{Coarse-Grained Modeling of a Molecular System}
\label{sub_sec:CG}

\begin{figure}[b]
\vspace{0.5cm}
\centerline{
\includegraphics[height=2cm]{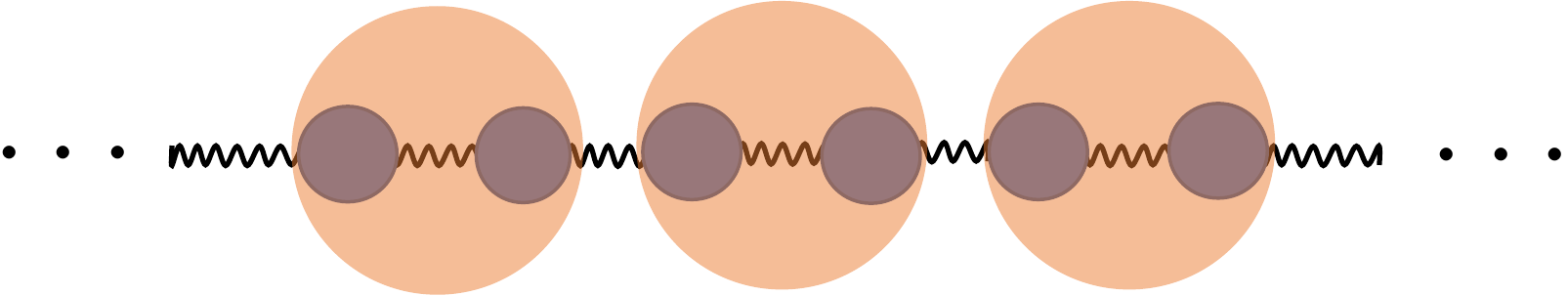}
}
\vspace{0.5cm}
\caption{Schematic illustration of the coarse-graining scheme of a 1D particle chain where two neighboring particles are coarse-grained into a large CG particle with position $Q_i=(q_{2i-1}+q_{2i})/2$ and momentum $P_i=p_{2i-1}+p_{2i}$.}
\label{fig:CG_chain} 
\end{figure}

Coarse-grained modeling of complex molecular systems is an important research area in statistical mechanics and molecular dynamics \cite{li2010coarse,chu2018asymptotic}. The goal is to construct effective dynamics for the coarse-grained (CG) particles from the original high-resolution molecular system. If the effective dynamics capture the core features that are sufficient to understand the important physics of the original system, then one only needs to solve the coarse-grained, low-dimensional effective dynamical system so that the overall computational cost can be greatly reduced. However, analytically solving the coarse-graining problem for realistic molecular systems is virtually intractable. Hence, data-driven approaches are commonly adopted. Classical coarse-graining methods such as those utilizing the Mori--Zwanzig formulation or the GLE extract important statistics for the CG particles from simulation data and build an integro-differential stochastic differential equation, as given in \eqref{sample_gle}, to describe the effective dynamics. This is very close to what we do with SINN, except that the equation is now replaced by a neural network. To test how well SINN models the dynamics of a CG particle, we consider a chain of $N$ particles interacting with each other according to the following Langevin dynamics \cite{zhu2021effective}:
\begin{equation}\label{langevin_d_chain}
\begin{aligned}
    \frac{\d r_j}{\d t}&=\frac{1}{m}(p_j-p_{j-1}),\\
    \frac{\d p_j}{\d t}&=\frac{\partial V(r_{j+1})}{\partial r_{j+1}}-\frac{\partial V(r_j)}{\partial r_j}-\frac{\gamma}{m}p_j+\sigma\xi(t),
\end{aligned}
\end{equation}
where $\{r_i,p_i\}_{i=1}^{N}$ are non-canonical coordinates for the dynamics. Here, $r_j=q_j-q_{j-1}$ represents the displacement between two neighboring particles relative to their equilibrium positions and $p_j$ is the momentum of the $j$-th particle. The two endpoints of the chain are assumed to be fixed, \ie $q_0 = q_{N+1} = 0$, and the chain has $N=100$ particles. The model parameters are chosen to be the same as those used in studying \eqref{langevin_dynamics}. That is, we have an FPU-type interaction potential $V(r)=\frac{\alpha}{2}r^2+\frac{\theta}{4}r^4$ with $\alpha=\beta=\theta=\gamma=1$ and $\sigma=(2\gamma/\beta)^{1/2}=\sqrt{2}$. 

For this chain dynamics, we consider the coarse-grained scheme illustrated in \Cref{fig:CG_chain} and focus on the effective dynamics of the center CG particle with position $Q_{25}=(q_{49}+q_{50})/2$ relative to its equilibrium. We note that $Q_{25}$ can be explicitly expressed in terms of $r$-coordinates as:
\begin{equation}
    Q_{25}=\frac{q_{49}+q_{50}}{2}=\sum_{i=1}^{49}r_i+\frac{r_{50}}{2}.
\end{equation}
Traditionally, a GLE of form \eqref{gle} can be used as an ansatz to approximate the stochastic dynamics for $Q_{25}$ \cite{lei2016data}. Here, we use SINN instead to do the modeling. The simulation results and calculation details are provided in \Cref{fig:CG_path_ACF}. Although the dynamics for $Q_{25}$ is completely hidden, our simulation results indicate that with only a limited amount of 400 temporally coarse-grained trajectories, SINN can act as an accurate surrogate model for $Q_{25}$. Detailed runtime statistics can be found in \Cref{Tab:runtime}. We particularly note the runtime savings when comparing with the Euler--Maruyama scheme. More explanations can be found in \Cref{sec:asssement}.  

\begin{figure}[t]
\centerline{
\includegraphics[height=4cm]{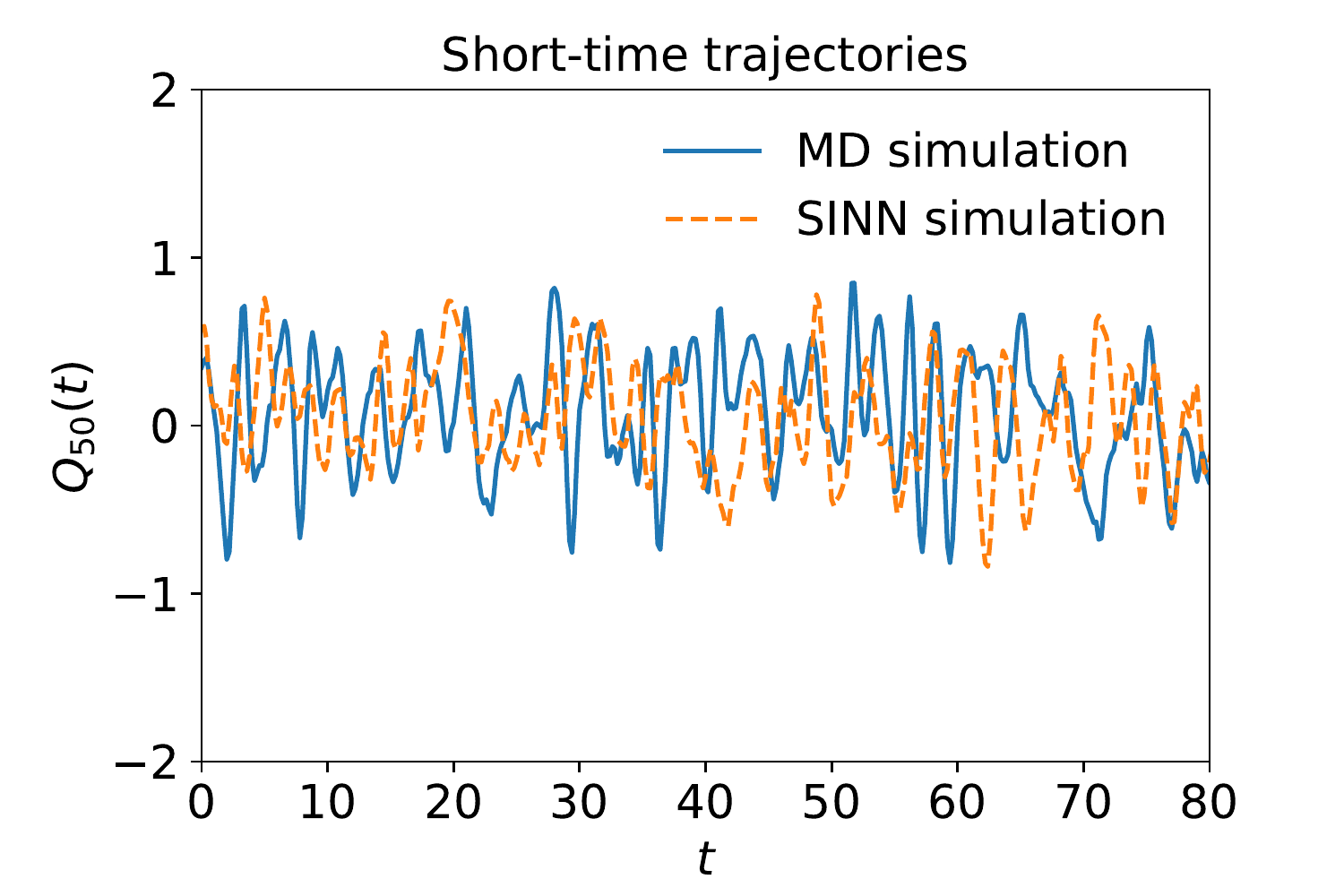}
\includegraphics[height=4cm]{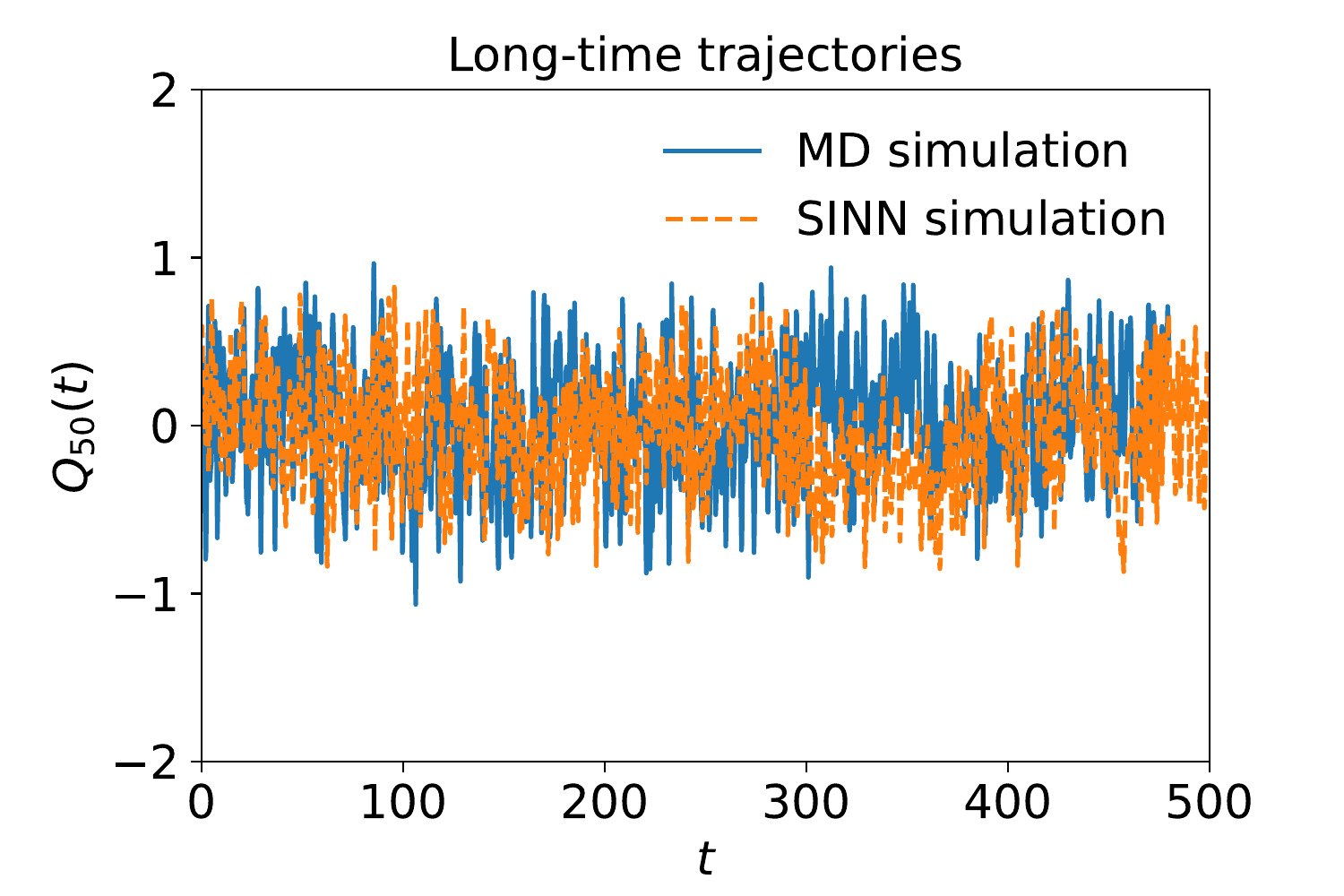}
\includegraphics[height=4cm]{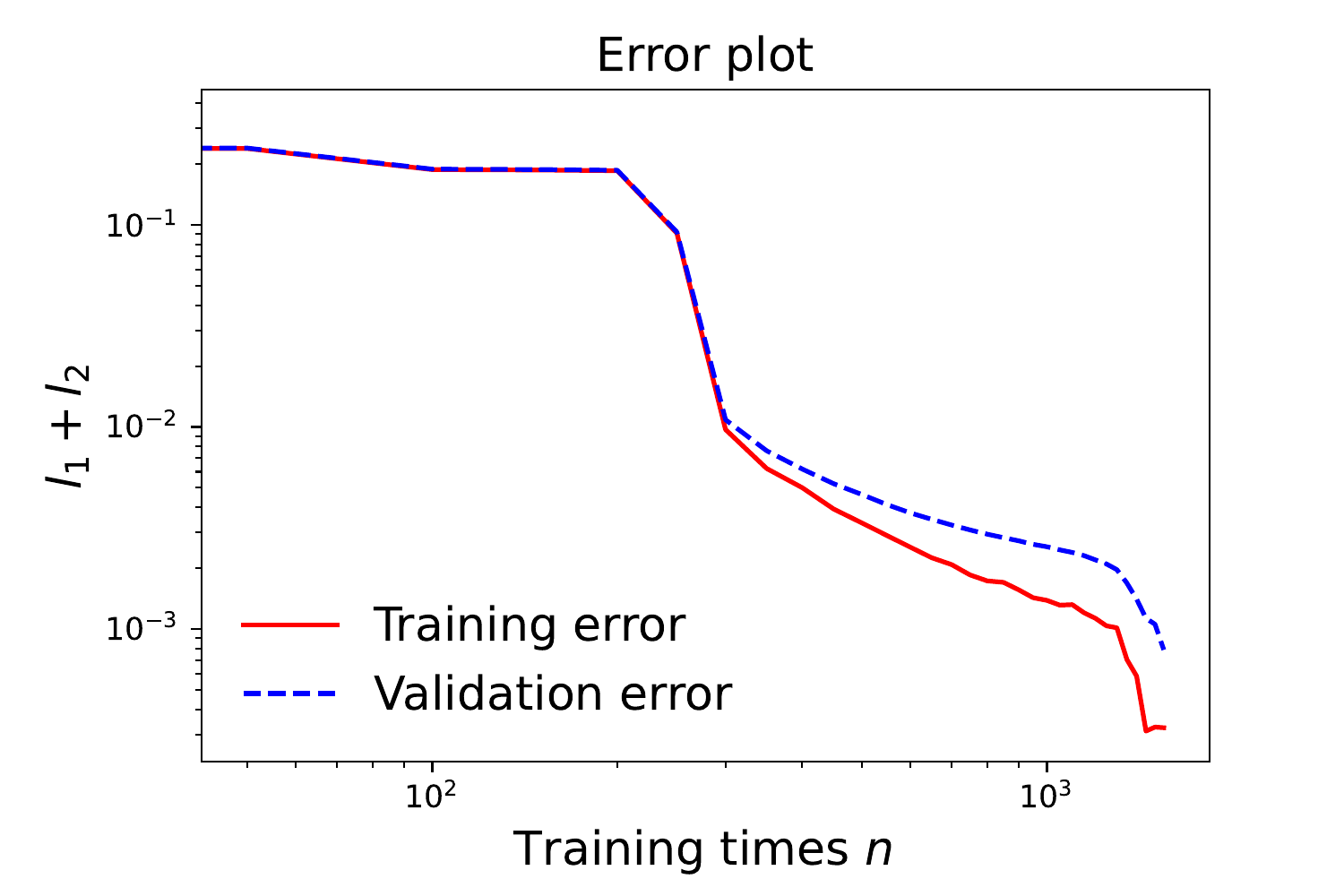}
}
\centerline{
\includegraphics[height=4cm]{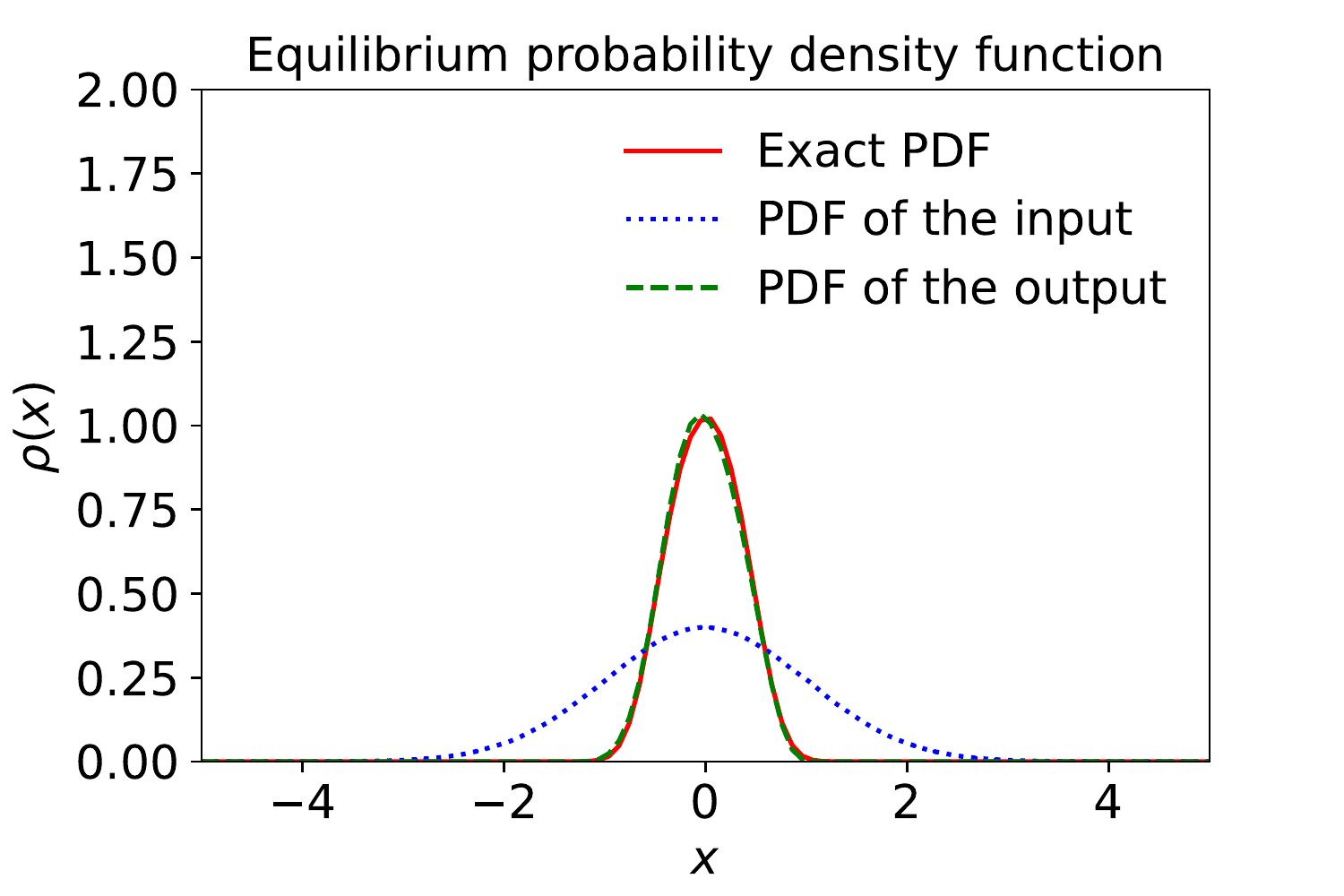}
\includegraphics[height=4cm]{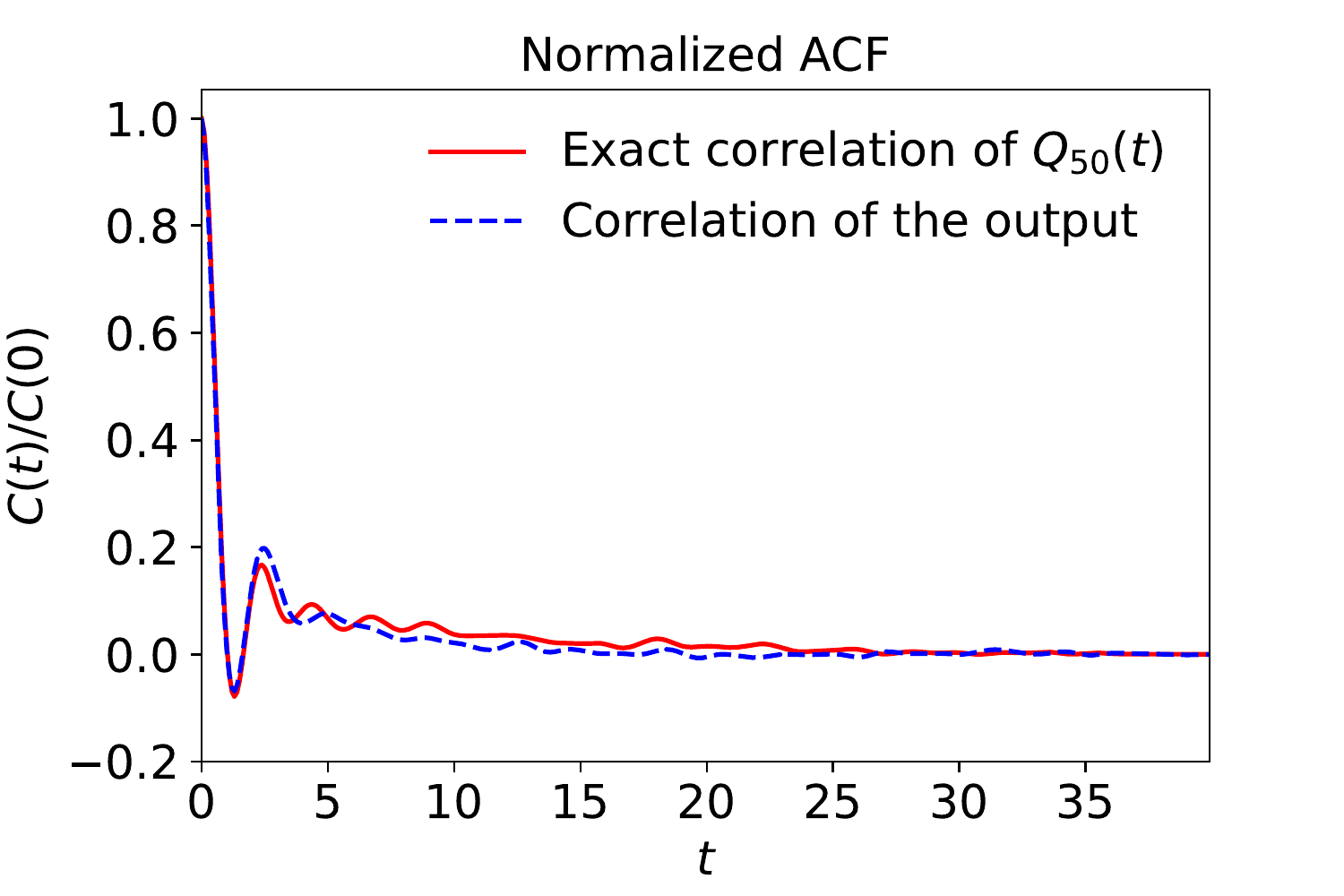}
\includegraphics[height=4cm]{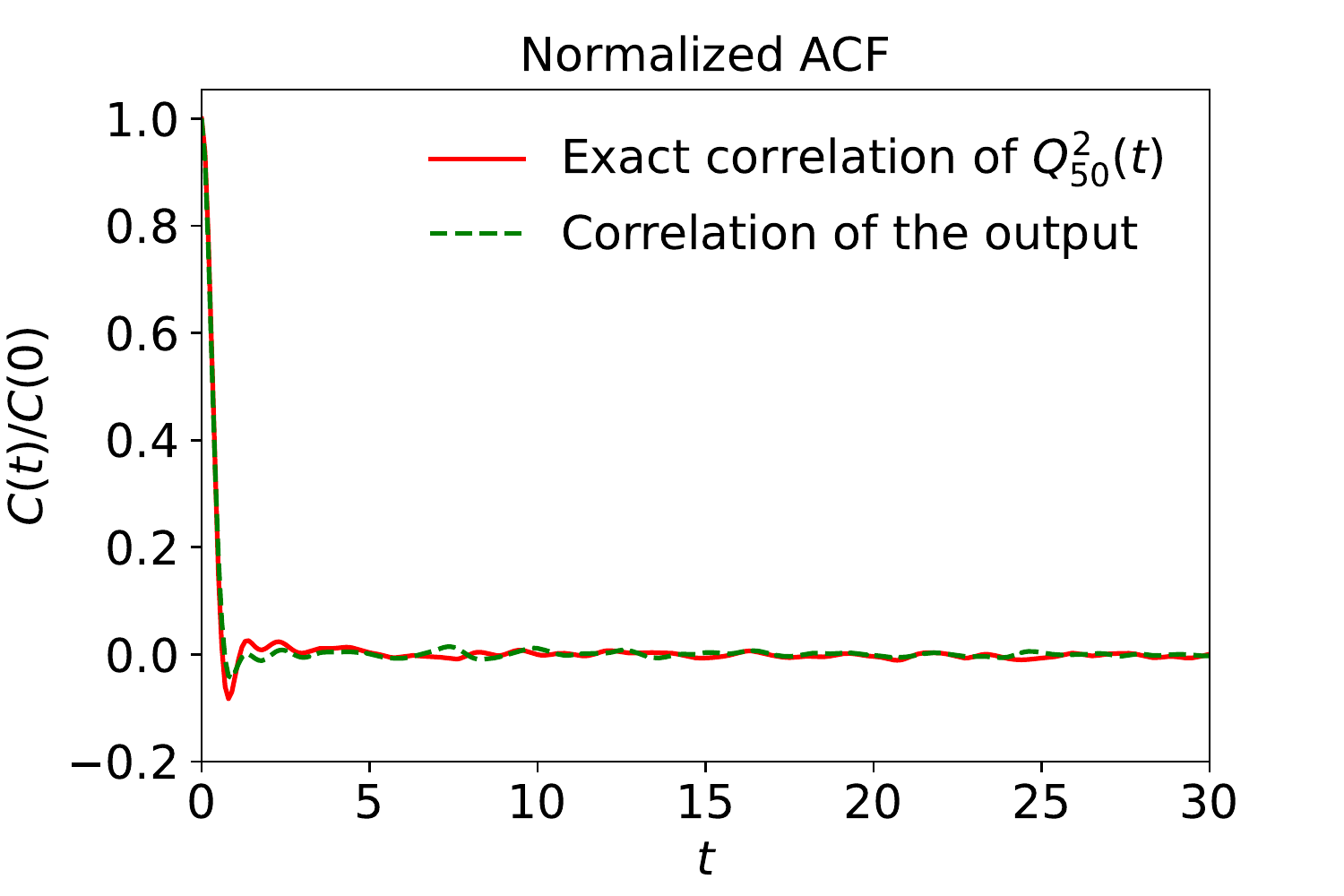}
}
\caption{Comparison of the dynamics of the coarse-grained particle $Q_{50}(t)$ generated by MD simulations and \SINN. Sample trajectories from MD simulation (Top Left) are obtained using the Euler--Maruyama scheme for \eqref{langevin_d_chain} with a step size $\Delta t=10^{-3}$. The training data for SINN are the temporally coarse-grained sample trajectories of $Q_{50}(t)$ with a step size $dt=0.1$. The training setup is exactly the same as the one used in the FPU example \eqref{langevin_dynamics}, except that the SINN model here uses 2 LSTM layers and 5 hidden units per layer.}
\label{fig:CG_path_ACF} 
\end{figure}

\subsection{Transition Dynamics Modeling and Rare-event Simulations}
\label{sec:APP_reduced}

\begin{figure}[t]
\centerline{
\includegraphics[height=6cm]{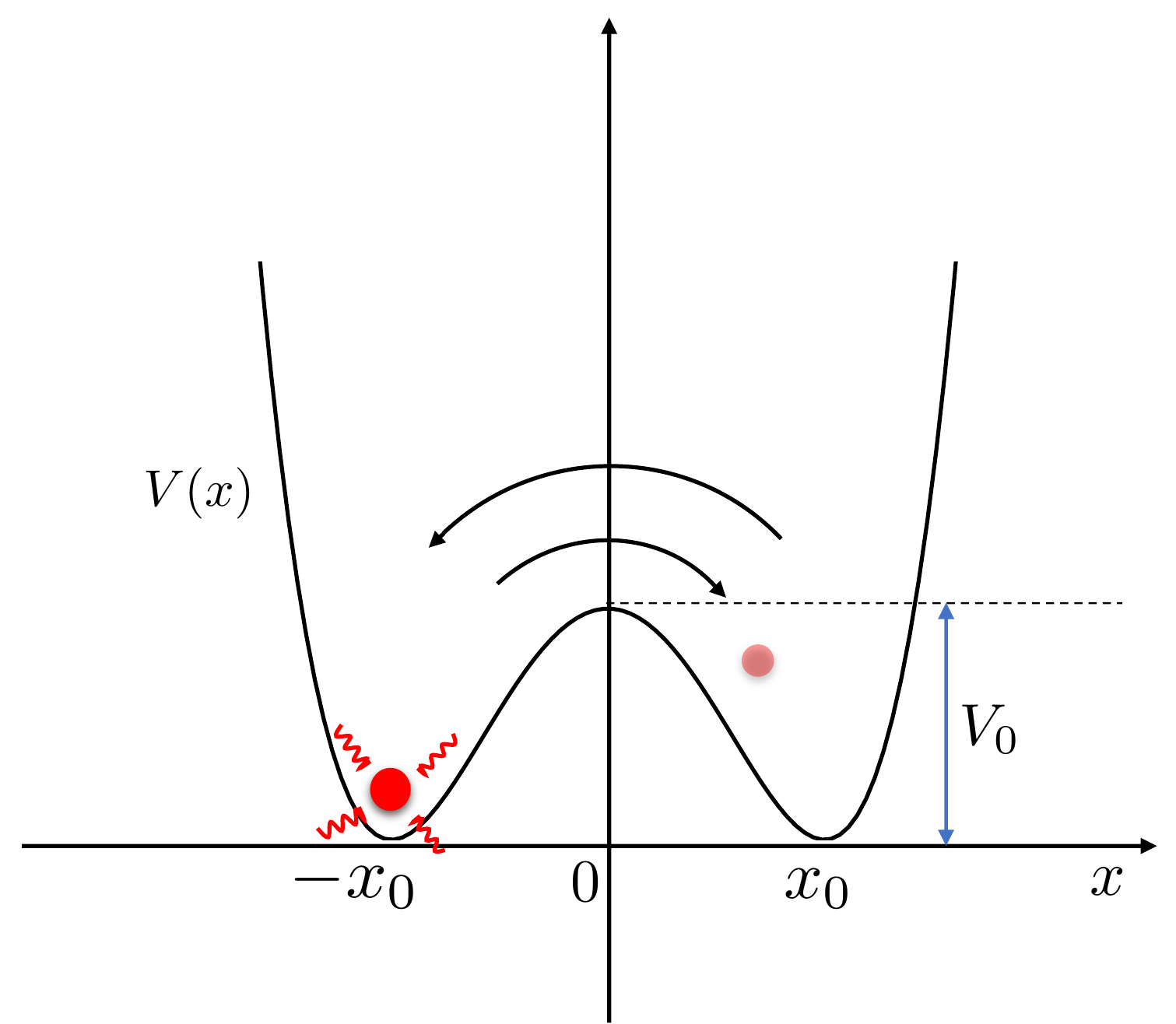}
\hspace{1cm}
\includegraphics[height=5.5cm]{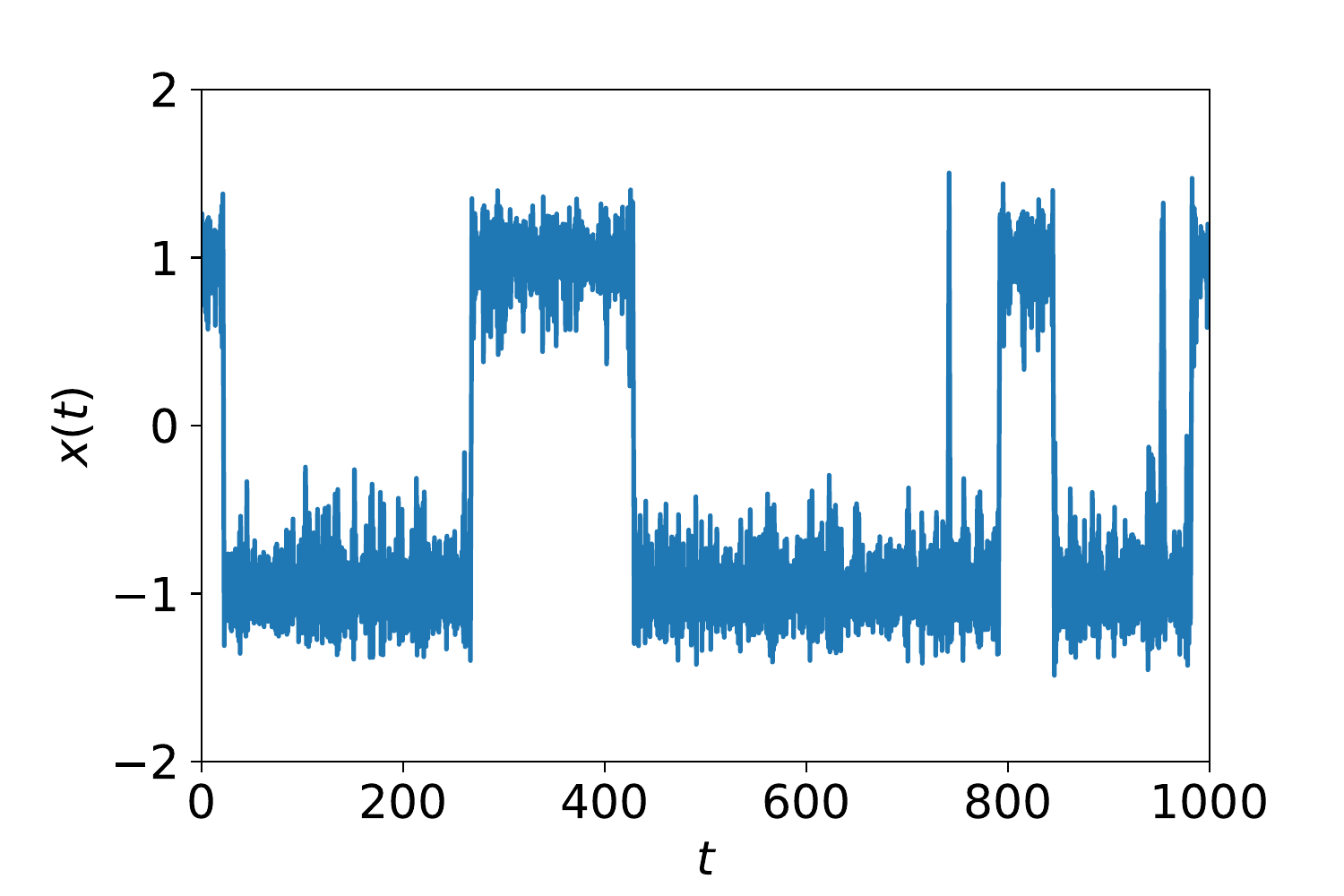}
}
\caption{(Left) Schematic illustration of the hopping events between two states for the reaction coordinate $x(t)$. Through thermodynamic interactions with the environment, an imaginary particle may gain enough energy to cross the energy barrier and make a transition from one well to the other. (Right) Sample trajectory of $x(t)$ simulated using \eqref{langevin_doublewell_dynamics}. The modeling parameters are chosen to be $V_0=5$, $x_0=1$, and $\beta=1$. One can see that hopping between these two states is a rare event for the given height of energy barrier.}
\label{fig:double_well} 
\end{figure}

The motivation for modeling transition dynamics stems from the need to calculate the reaction rate of a chemical reaction. While this is an important problem, determining the reaction rate using simulation trajectories becomes extremely difficult when the reaction is a rare event. This is because it takes a long time, which often exceeds
the typical time scale of MD simulations, to adequately observe the reaction when it happens with a very low probability.
As a rare event example, we consider a toy problem as illustrated in \Cref{fig:double_well} for the transition dynamics given by the Langevin equation for a double-well system \cite{liu2015equation,harlim2021machine}:
\begin{align}\label{langevin_doublewell_dynamics}
\begin{dcases}
    \dot {x}&=p,\\
    \dot p&=-V'(x)-\gamma p+\sigma \xi(t),
\end{dcases}
\end{align}
where $V(x)=V_0\left[1-(x/x_0)^2\right]^2$ is a symmetric double-well potential with depth $V_0$ and two basins around $x_0$ and $-x_0$. The two wells correspond to two states along the reaction coordinate $x(t)$, which, for example, can be the backbone dihedral angle of $n$-butane in the isomerization process \cite{geissler1999kinetic}.
We aim to use SINN to construct a reduced-order model for the reaction dynamics based on the {\em short-time} simulation data of $x(t)$. Once the effective model is built, one may use it as a surrogate model to perform Monte Carlo simulations or to generate {\em long-time} trajectories of $x(t)$ with larger time step sizes. From this, the rate of transition can be calculated in an economical manner. In practical applications, one may obtain sample trajectories for the reaction coordinate $x(t)$ from large-scale MD simulations that model the whole physical system. Here, we use a toy model \eqref{langevin_doublewell_dynamics} to quickly generate sample trajectories of $x(t)$ for the purpose of demonstrating the learning capability of SINN and its validity in simulating rare events.

Since the transition dynamics is more complicated than the examples considered in \Cref{sec:test_cases}, we employ a \SINN model with 2 LSTM layers and 25 hidden units per layer. This structure may not be optimal in terms of complexity or efficiency, but is found to be sufficient for our study. To train the neural network, we solve  \eqref{langevin_doublewell_dynamics} using the Euler--Maruyama scheme with step size $\Delta t=10^{-3}$ and obtain 400 sample trajectories. Temporally coarse-grained trajectories are obtained by subsampling the trajectories with a fixed interval of $dt=0.2$ and used as the training data. After the temporal subsampling, each trajectory contains 400 points uniformly spanning the interval $[0,80]$. The equilibrium PDFs and ACFs for $x(t)$ and $x^2(t)$ are used to construct the loss function. During the training process, occasionally the optimization gets stuck at local optima. In this case, we simply need to retrain the model from random initialization. The numerical results are presented in \Cref{fig:p_Sample_path_compare}. We see that SINN yields an overall excellent approximation for the transition dynamics. Remarkably, it actually reproduces the hopping events between the two states. Moreover, from the comparison of the long-time trajectories and the normalized ACF, we can see that the trained SINN model can generate extrapolated trajectories for a good prediction of the long-time dynamics of $x(t)$, even though the training was carried out using only data from $t=0$ to $t=80$. A more qualitative evaluation of SINN on describing the transition dynamics relies on the calculation of the transition rate from the generated samples trajectories, which can be found in \Cref{sec:asssement}.   

\begin{figure}[t]
\centerline{
\includegraphics[height=4cm]{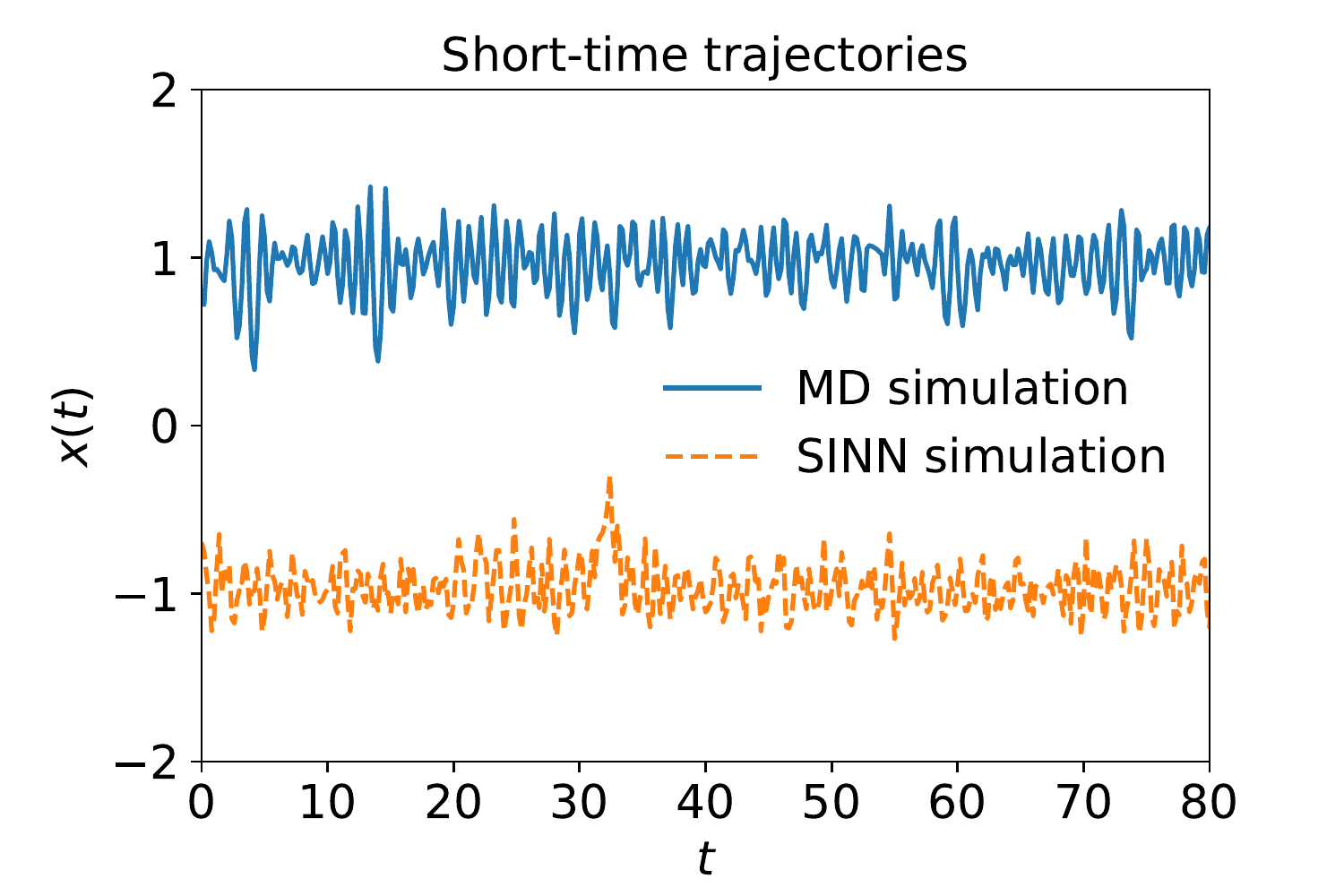}
\includegraphics[height=4cm]{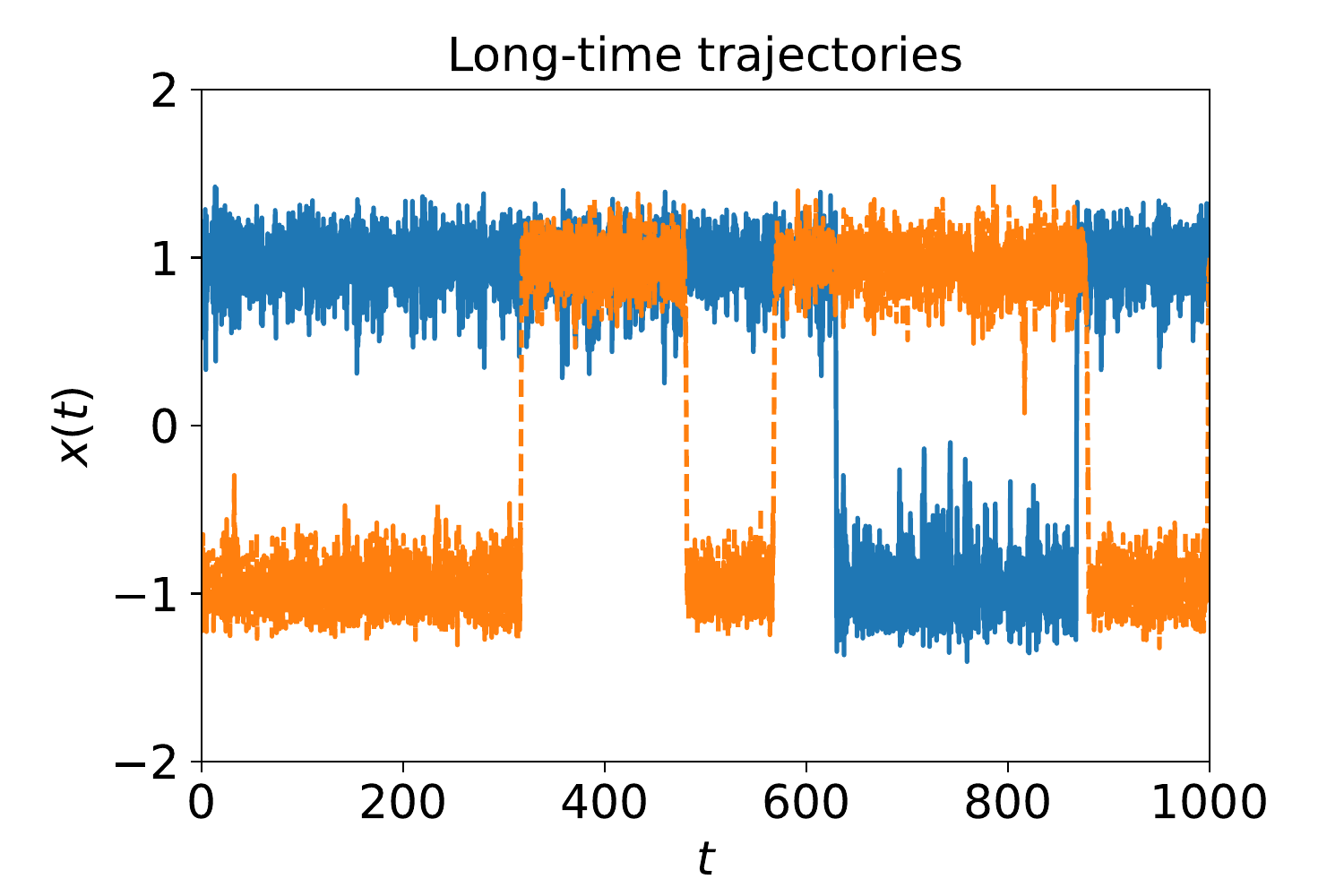}
\includegraphics[height=4cm]{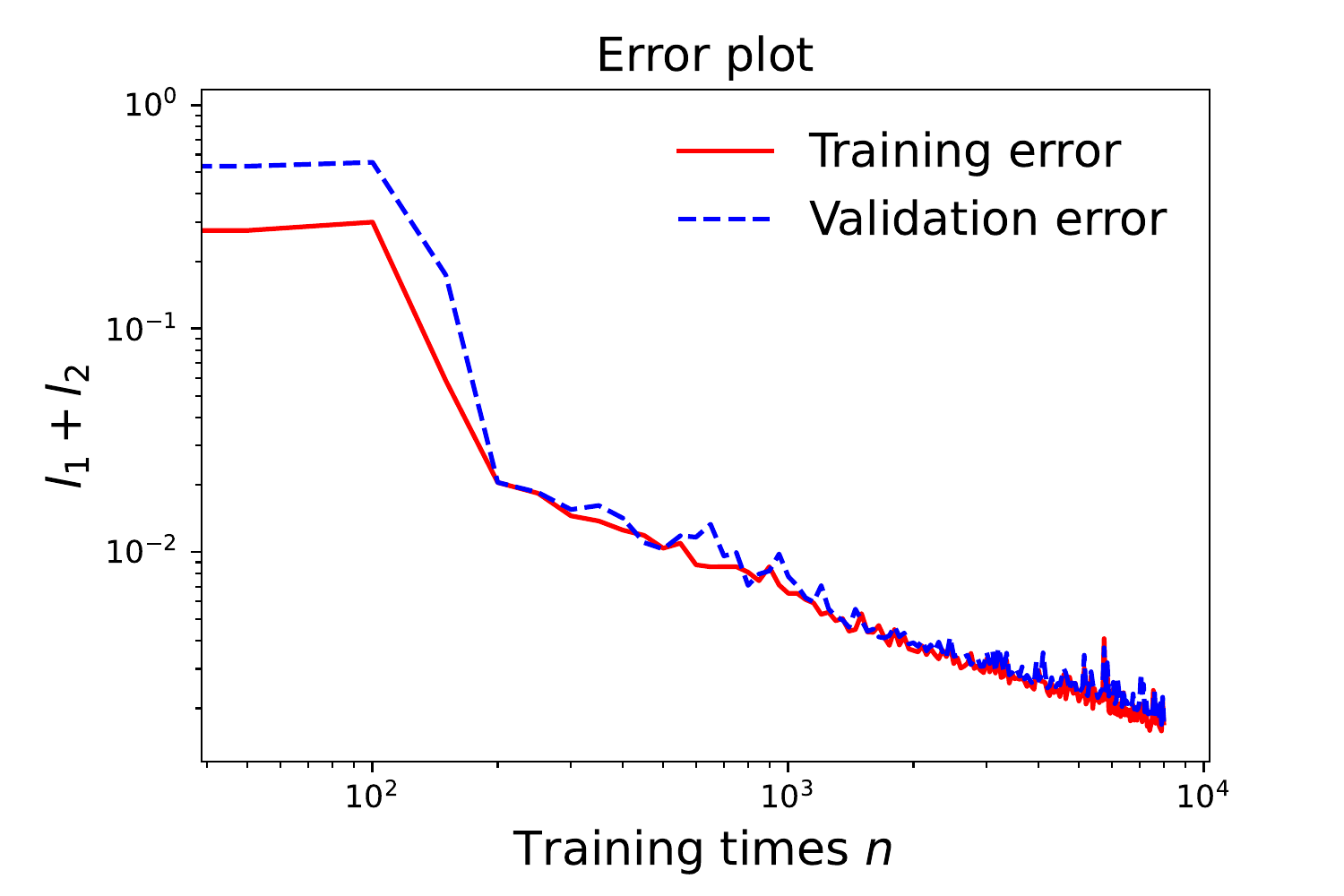}
}
\centerline{
\includegraphics[height=4cm]{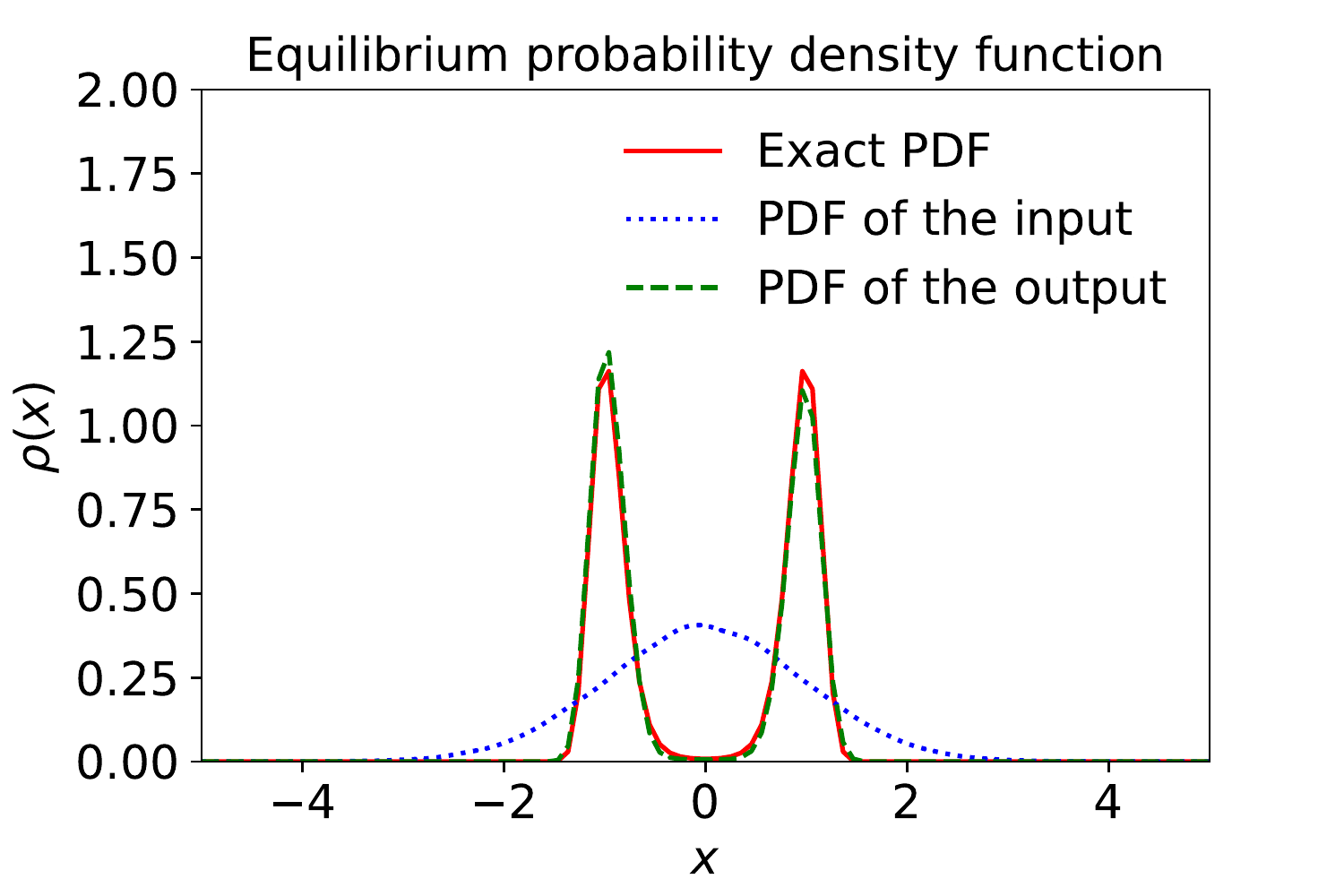}
\includegraphics[height=4cm]{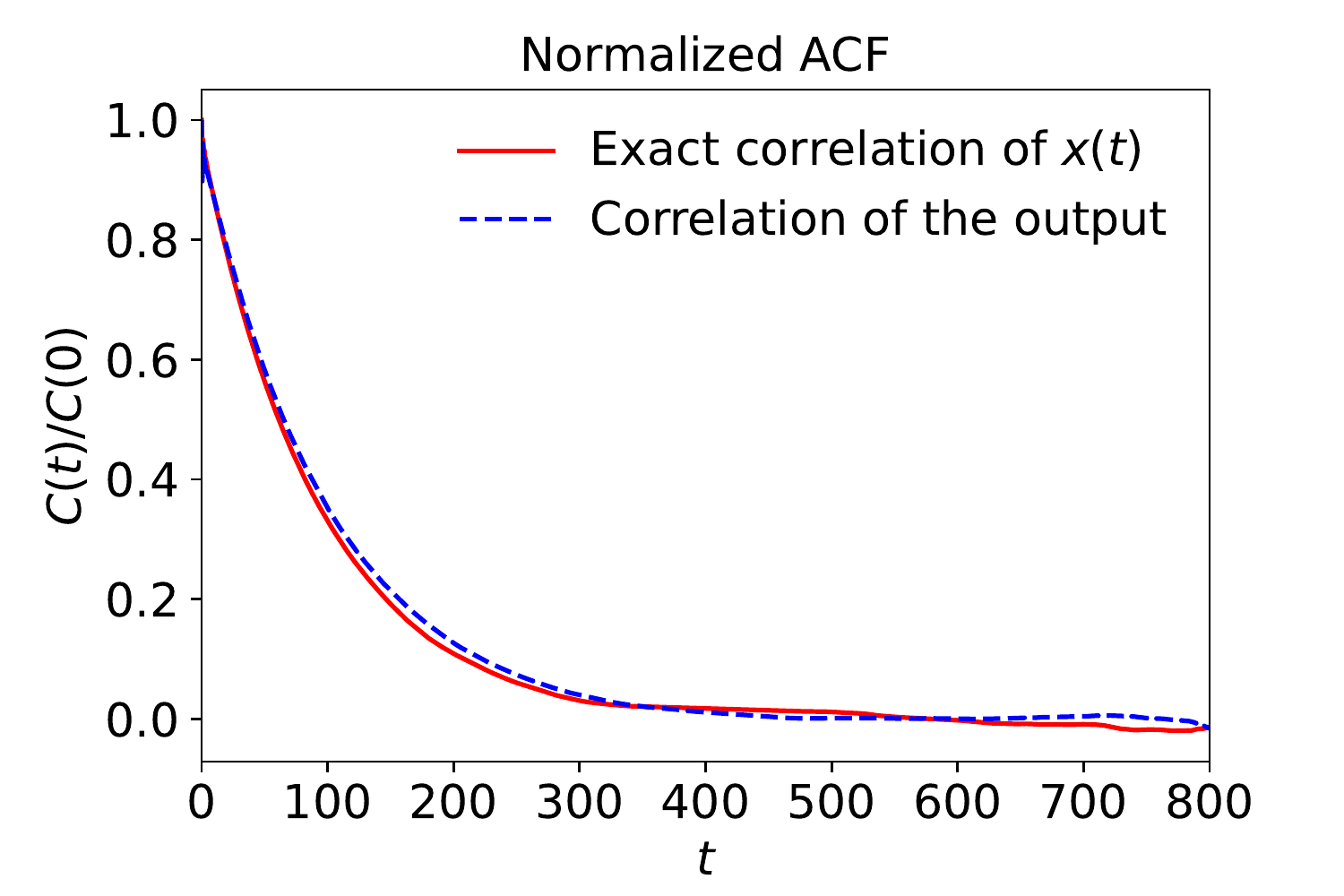}
\includegraphics[height=4cm]{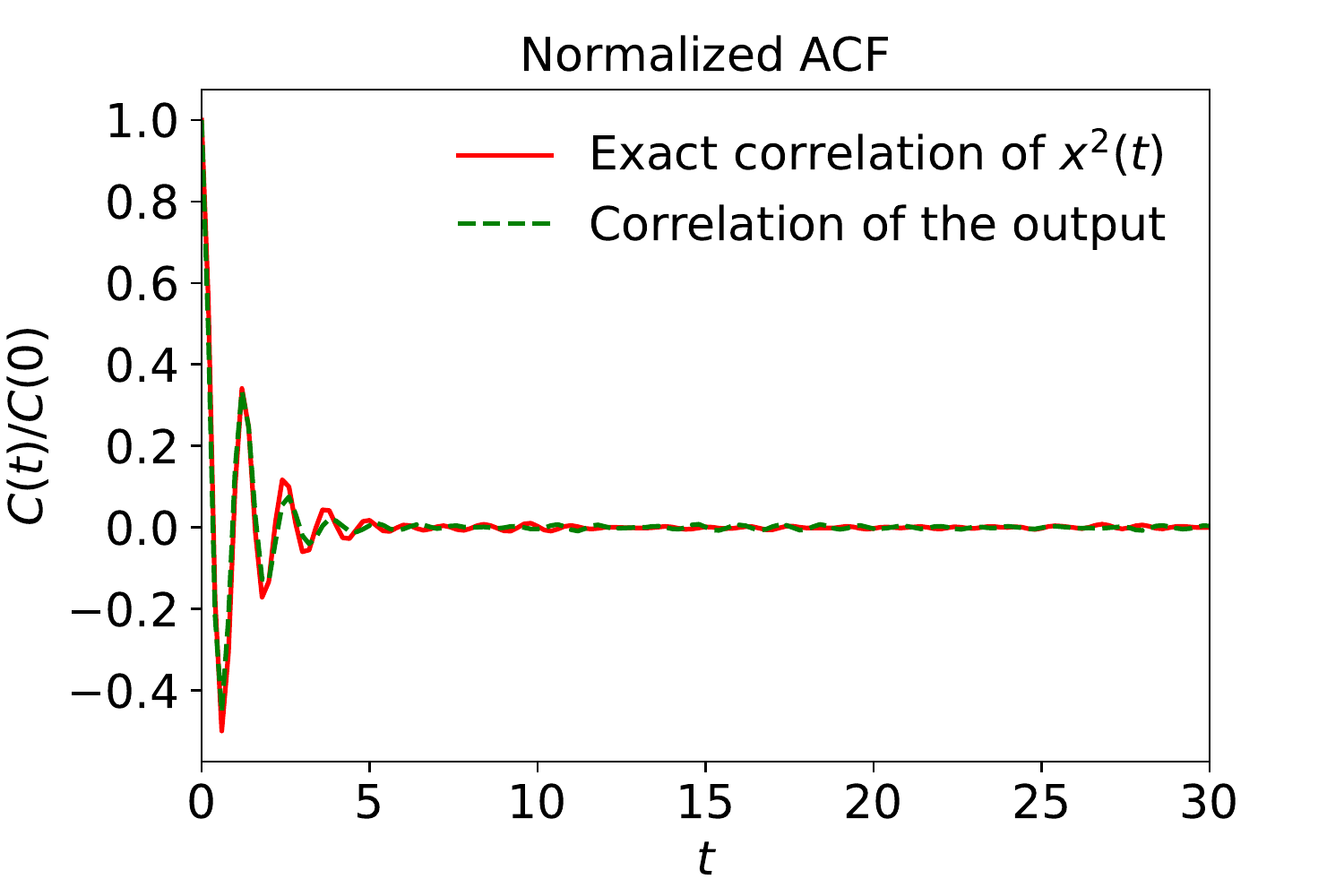}
}
\caption{Comparison of the dynamics of the reaction coordinate $x(t)$ generated by MD simulation and the \SINN model. The exact statistics, including the PDFs and the ACFs for $x(t)$ and $x^2(t)$, are obtained through MD simulations of \eqref{langevin_doublewell_dynamics} and averaged over $5\times 10^4$ trajectories. The statistics for the SINN outputs are similarly calculated.}
\label{fig:p_Sample_path_compare} 
\end{figure}

\subsection{Further Assessment of SINN}
\label{sec:asssement}

In this section, we thoroughly study the SINN model for the double-well system and use this example to further assess the modeling capability of SINN in various aspects. In particular, we focus on the temporally coarse-grained nature of SINN and explain why this property is important for rare event simulations.  More simulations are carried out to show that with short-time data, SINN can extrapolate and predict long-time dynamics. We also discuss favorable numerical features of SINN, including numerical convergence and consistency of the randomized optimization procedure.

\paragraph{\SINN as a Coarse-grained Time Integrator.}
The \SINN models we have used so far are trained using the coarse-grained sample trajectories of $x(t)$ with a time step size $dt=0.2$, which is much larger than the MD integration time step size $\Delta t = 10^{-3}$. The output of \SINN, \ie the approximated trajectories of $x(t)$, has the same coarse-grained step size as the training data. This makes \SINN a natural coarse-grained time integrator for the reduced-order dynamics of $x(t)$. This coarse-grained nature of our SINN model provides an efficient means to generate long-time approximated trajectories of $x(t)$ because the sampling gets 200 times sparser. For the calculation of physical quantities, such as the reaction rate, where the local fast-time dynamics becomes irrelevant, this leads to huge computational advantages. In \Cref{fig:Sample_path_coarse}, we compare sample trajectories of $x(t)$ generated by well-trained SINNs with different time step sizes $dt=0.1$, 0.2, 0.5. For these three time-scales, SINN all reproduces the correct hopping dynamics. It is also clearly observed that while fast-time dynamics is filtered out as $dt$ increases, the statistics, \ie PDFs and ACFs, of the predicted trajectories remain essentially unchanged. This means that the coarse-grained trajectory is {\em sufficient} to capture the important physics, \ie the statistics of the hopping dynamics. The capability to generate temporally coarse-grained stochastic trajectories is one of the most prominent features of SINN, which gives huge computational advantages. We emphasize that this feature is not easily achievable using established stochastic modeling methods such as the ones based on the Mori-Zwanzig equation, GLE, or the NeuralSDE.

\begin{figure}[t]
\centerline{
\includegraphics[height=4cm]{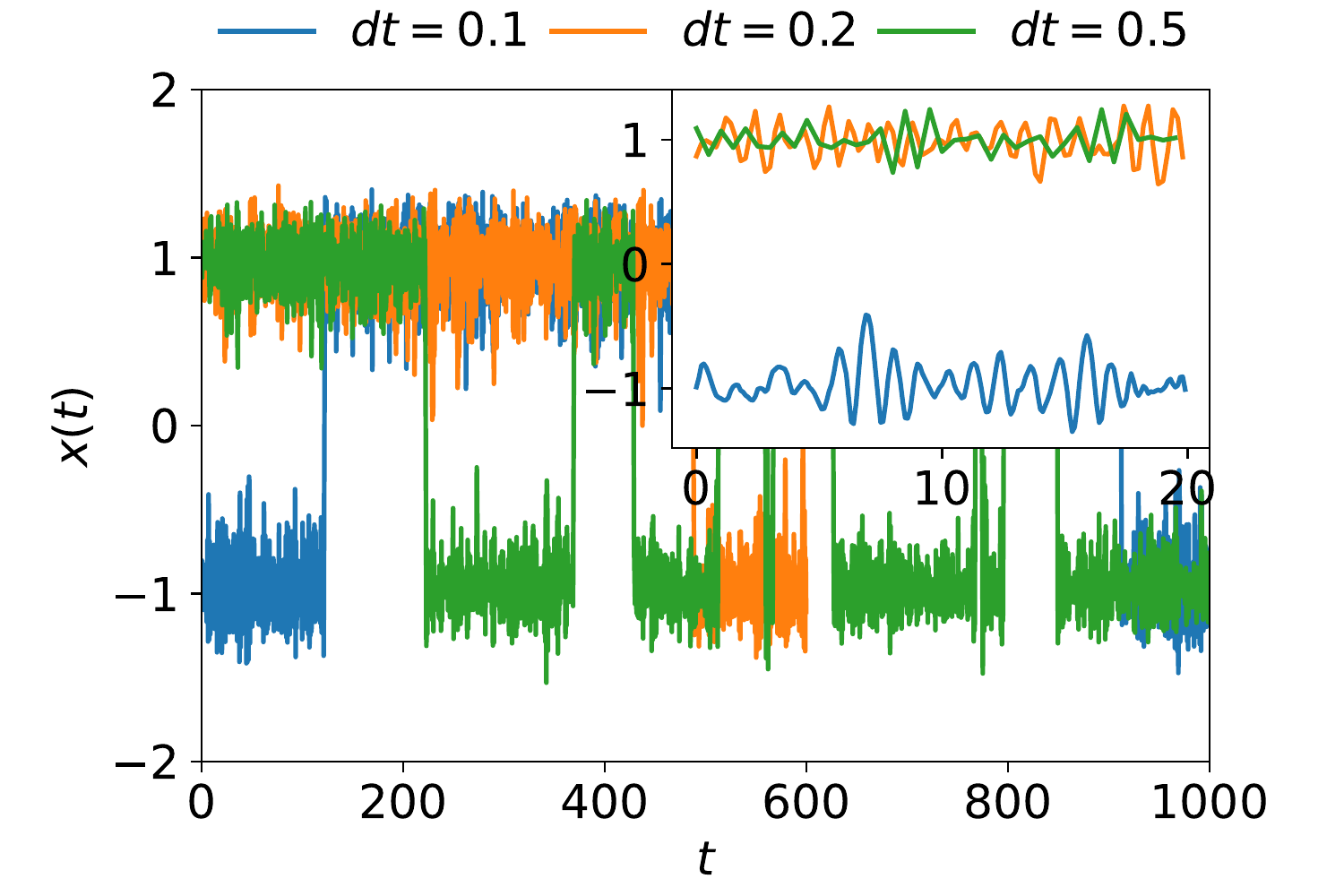}
\includegraphics[height=4cm]{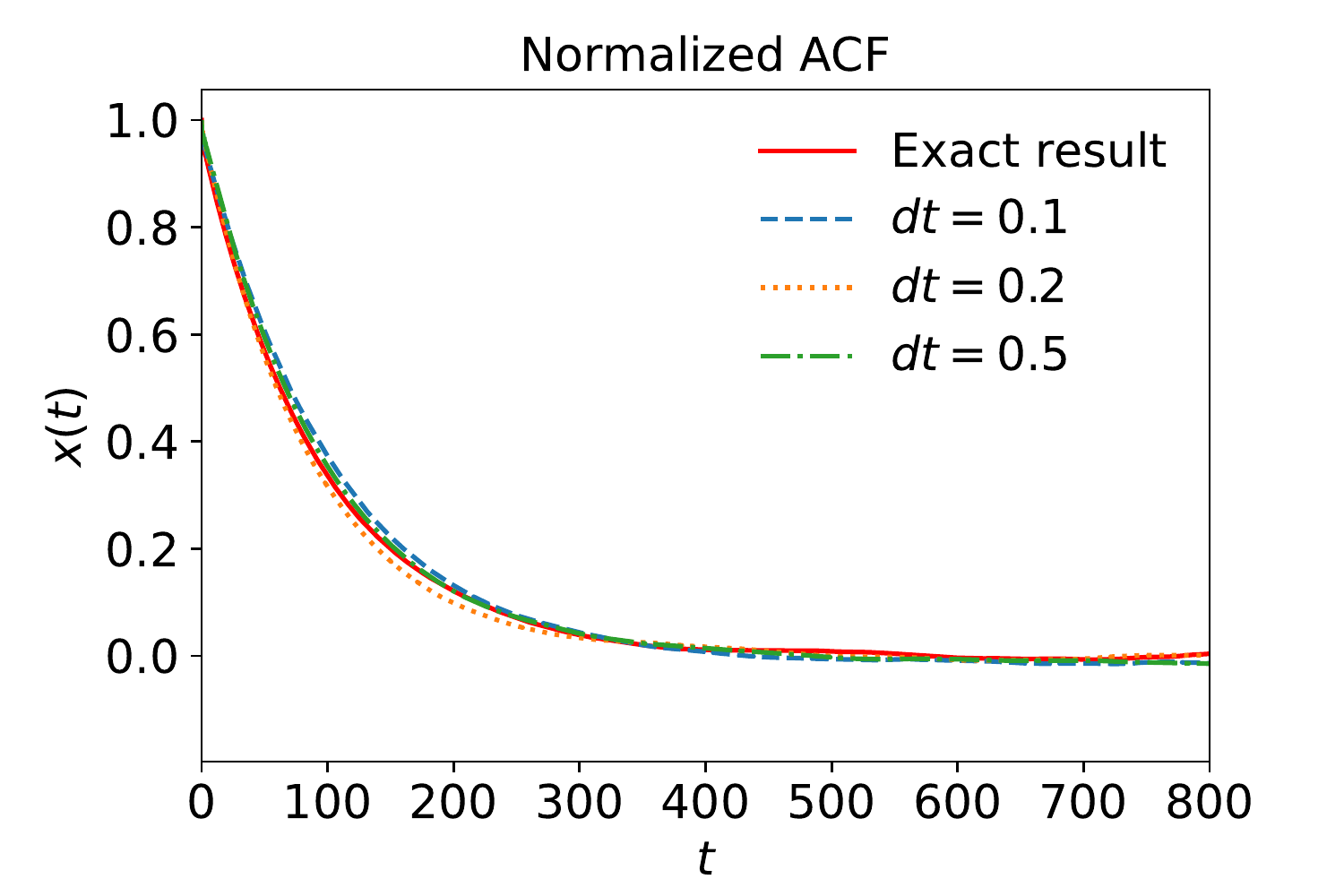}
\includegraphics[height=4cm]{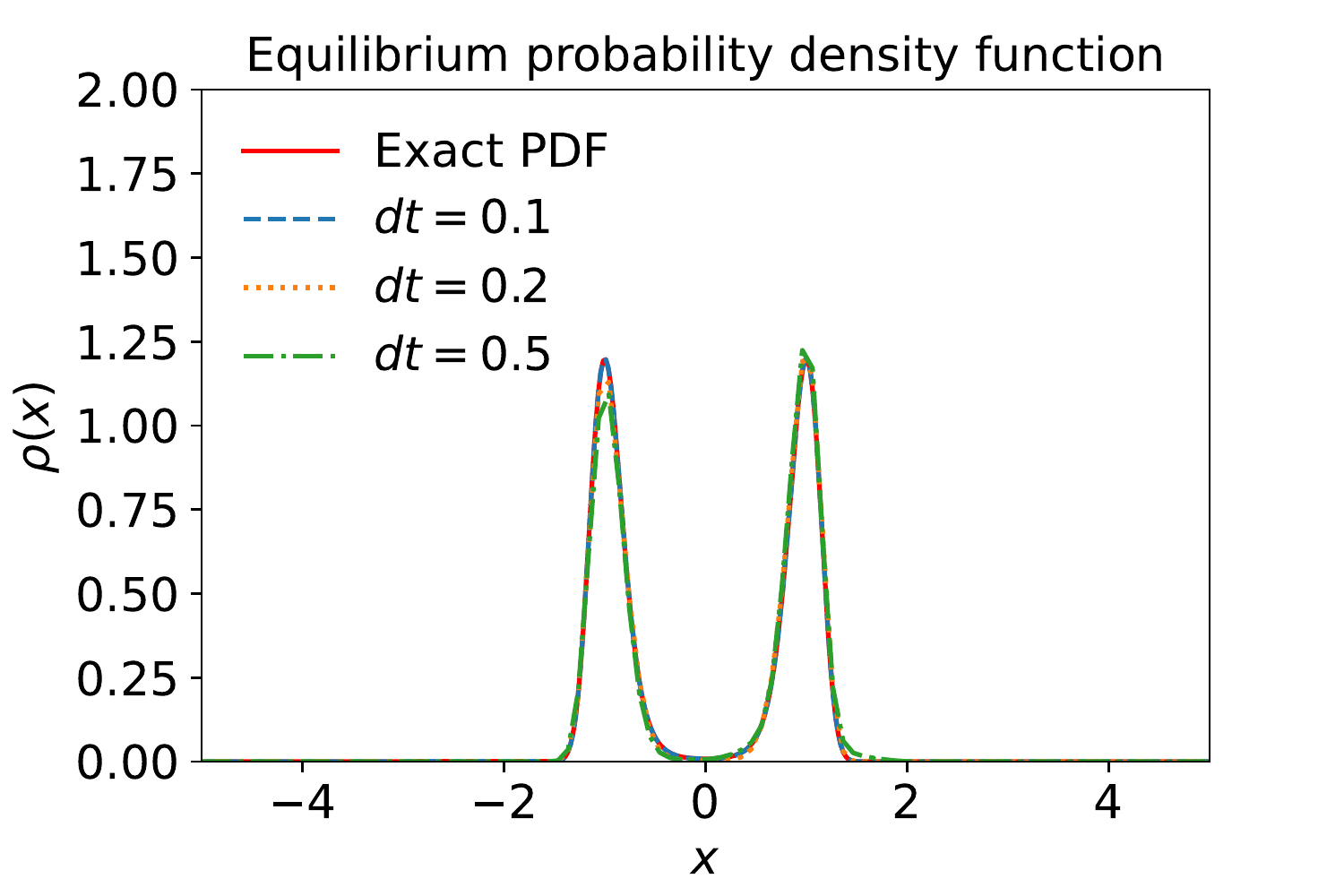}
}
\caption{Comparison of SINN models with different coarse-grained time scales $dt$. As $dt$ increases, local information gets gradually filtered out as shown by the short-time trajectories in the inset figure of the left panel. However, the ACF and PDF of the simulated trajectories remain essentially the same.}
\label{fig:Sample_path_coarse} 
\end{figure}

We also compare the computational time of the Euler--Maruyama scheme versus SINN in \Cref{Tab:runtime}. The savings in simulation time of using SINN to generate coarse-grained trajectories can be clearly seen. The advantage is particularly obvious in the CG example since the underlying MD system has a high dimensionality of 200. We emphasize that the dimensionality of many realistic MD systems, \eg proteins, is much higher and hence the potential saving in computational time will be enormous. Meanwhile, the time spent on training \SINN is well amortized.

\paragraph{Remark.} A possible intuitive explanation on why neural networks can actually learn a coarse-grained time integrator for dynamical systems might be due to the similarity between the multi-fold function composition structure of neural network with that of the Runge--Kutta (RK) method for solving an ODE $y'=f(t,y)$. The well-known fourth-order RK method has in fact a function composition structure and can be rewritten as:
\begin{align}\label{RK4}
   y_{n+1}=y_n+\frac{\Delta t}{6}\left\{ f(t_n,y_n)+2f\left(t_n+\frac{\Delta t}{2},y_n+\frac{\Delta t}{2}f(t_n,y_n))\right)+\cdots\right\},
\end{align}
where $y_n$ is the approximated solution at time $t_n$. By further expressing $y_n$ in \eqref{RK4} using $y_{n-1}$ and $f(t_{n-1},y_{n-1})$ and repeating this procedure, we obtain a natural coarse-grained time integrator with a multi-fold function composition structure. One can imagine that a neural network may have learned a similar structure during training. In fact, this connection was already noticed in the literature~\cite{rico1992discrete}.\footnote{We thank Prof.\ Yannis Kevrekidis for pointing this out to us through private communications.}

\paragraph{Calculation of Transition Rate.}
We use our SINN model as a simulator for the transition dynamics of the double-well system~\eqref{langevin_doublewell_dynamics} and assess how well it predicts the transition rate for rare events. To calculate the transition rate between the two states, we first divide the phase space for $x(t)$ into two regions: $A=(-\infty,0]$  and $B=(0,+\infty)$. Obviously, $-x_0\in A$ and $x_0\in B$. Consider the equilibrium time correlation function $C_{A,B}(t)$ defined by:
\begin{align}\label{transition_correlation}
    \frac{C_{A,B}(t)}{C_A}=\frac{\langle h_A(x(0))h_B(x(t))\rangle}{\langle h_A(x(0))\rangle},
\end{align}
where $h_A(x(t))$ is an indicator function of system configuration satisfying $h_A(x(t))=1$ if $x(t)\in A$ and $h_A(x(t))=0$ if $x(t)\notin A$, while $h_B(x(t))$ is analogously defined. Thus, the ratio \eqref{transition_correlation} is the probability of finding the system in state $B$ after time $t$ when the system is initially at state $A$. As a result, the transition rate from $A$ to $B$ can be calculated as \cite{bolhuis2002transition,dellago1998transition}:
\begin{align}\label{transition_rate_formula}
        k_{AB}=\frac{\d}{\d t}\frac{C_{A,B}(t)}{C_A}, \qquad \tau_{mol}<t\ll \tau_{rxn},
\end{align}
which is the slope of the $\frac{C_{A,B}(t)}{C_A}$ curve in the time range between a short transient time scale $\tau_{mol}$ and the exponential relaxation time $\tau_{rxn}=1/(k_{AB}+k_{BA})$. 

\begin{figure}[t]
\centerline{
\includegraphics[height=5cm]{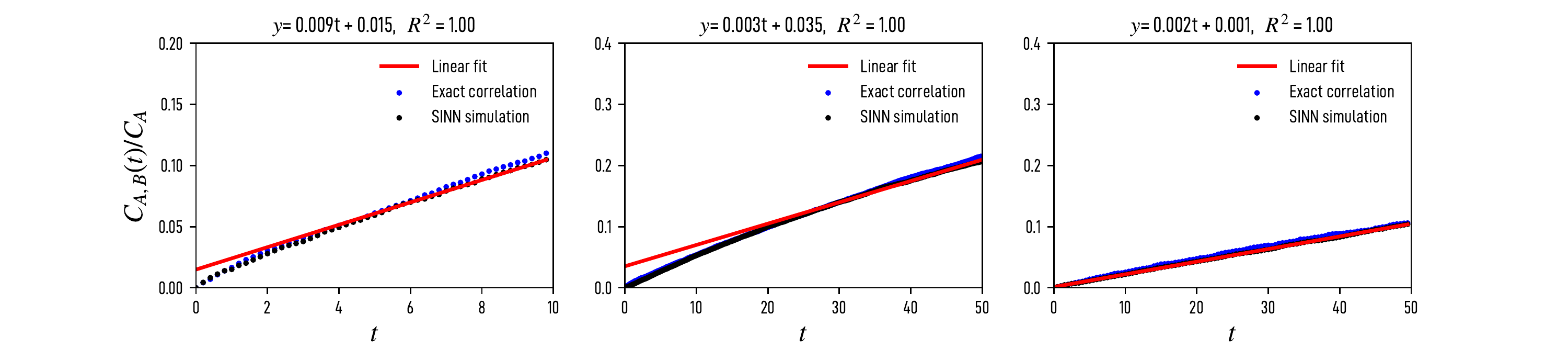}
}
\caption{Prediction of the transition rate using \SINN as the simulator for rare events. The time profiles of the equilibrium correlation function $C_{A,B}(t)/C_A$ for double-well dynamics \eqref{langevin_doublewell_dynamics} are plotted for different values of the barrier depth $V_0$ and the coarse-grained time scale $dt$: (left) $V_0=4$, $dt=0.2$; (middle) $V_0=5$, $dt=0.2$; (right) $V_0=6$, $dt=0.5$. The results obtained from SINN are compared with the numerical simulation results obtained from long-time MD trajectories. The linear regression is used in fitting $C_{A,B}(t)/C_A$ in between the transient time scale $\tau_{mol}$ and the exponential relaxation time scale $\tau_{rxn}$ in order to evaluate $k_{AB}$. The specific time domains for the linear regression are chosen to be (from left to right) $[5,10]$, $[25,50]$ and $[25,50]$, respectively. $R^2$ is the coefficient of determination.}
\label{fig:reaction_rate} 
\end{figure}

The numerical results are summarized in \Cref{fig:reaction_rate}. The time profiles of the equilibrium time correlation function match well with those obtained by MD simulations. The resulting values of $k_{AB}$ are approximately estimated to be $0.009$, $0.003$, and $0.002$ for transition dynamics with energy barrier height values $V_0=4$, 5, and 6, respectively. These agree well with the values obtained by MD simulations. By calculating the transition rate using SINNs trained with coarse-grained trajectories with different temporal resolutions $dt$, we also observe that the temporal coarse-graining of the trajectories does not significantly influence the calculated reaction rate. This is consistent with our previous analysis. The successful prediction of the transition rate $k_{AB}$ indicates that, with the equilibrium PDFs and ACFs for $x(t)$ and $x^2(t)$, it is {\em practically} sufficient to create a reliable numerical approximation for the reduced-order dynamics of $x(t)$ using SINN, although this information is not enough to {\em theoretically} guarantee the uniqueness of the non-Gaussian process.

\paragraph{Long-time Predictability, Numerical Convergence, and Consistency.}
Lastly, we discuss long-time predictability, numerical convergence, and training consistency of \SINN. As a neural network based on the LSTM architecture, \SINN makes prediction of the long-time dynamics of the reduced-order observable $x(t)$ by quantifying the memory effect of the non-Markovian system. This is similar to reduced-order modeling using the Mori--Zwanzig formalism or GLEs. Since these approaches are proven to have predictability of the long-time stochastic dynamics \cite{lei2016data,zhu2021generalized}, it is reasonable to expect \SINN to have a similar behavior. This is indeed the case as we have already shown in \Cref{fig:p_Sample_path_compare}, where the SINN model faithfully predicts the long-time dynamics of $x(t)$ using short-time training data for $t\in [0,80]$.

Due to the usage of the randomized training protocol, each learned \SINN model may have a {\em differently} parameterized memory model. This raises a reasonable doubt that whether our demonstrated long-time predictability of SINN is merely a coincidence. Due to the well-known difficulties on the theoretical convergence analysis for deep neural networks, here we only provide numerical verifications. To this end, we verify the convergence and consistency of the obtained SINN model by comparing the long-time tail of the ACF for $x(t)$ with the MD simulation result. We trained an ensemble of independently initialized \SINN models using the same data set. Specifically, we obtained three ensembles of SINN models. The first ensemble was trained using 400 trajectories of $x(t)$ for $t\in[0,40]$ with a coarse-grained step size $\Delta t=0.2$, while the time domain for the second and third ensemble trajectories are $[0,70]$ and $[0,100]$, respectively. In each ensemble, we repeated the training process to create 20 {\em candidate} SINN models. Each candidate model was obtained by independent training of SINN until the training and validation error satisfied $\epsilon_T,\epsilon_V=l_1+l_2\leq 10^{-3}$. From these candidate SINN models, we performed time extrapolation to generate long-time dynamics of $x(t)$ and re-evaluated the validation error $\epsilon_V$ to select the top 5 {\em qualified} SINN models with the smallest $\epsilon_V$. The evaluation time domains for $\epsilon_V$ were $[0,40]$, $[0,70]$, or $[0,100]$, respectively, since these were the only time frames for which the ground truth was known. This procedure ensured the qualified SINN models produce stationary time sequences in order to be consistent with the equilibrium dynamics of $x(t)$. The simulation results and the analysis are presented in \Cref{fig:sinn_prediction}. We see that with training data as short as $t \in [0,40]$, {\em qualified} SINN models, \ie the top ones with the smallest validation errors, yield overall good predictions of the long-time dynamics of $x(t)$. This validates the long-time predictability of SINN. As we gradually increase the length of the target trajectories from $t\in[0,40]$ to $t\in [0,100]$, the $95\%$ confidence interval of the predicted dynamics gets smaller. This indicates that the collective output of the ensemble of SINN models converges to the correct dynamics of $x(t)$. Hence, a numerical validation of the convergence of SINN is established here in terms of the statistics of the input-output. All these repeated training leads to accurate stochastic models for the transition dynamics. This confirms that the randomized optimizer used in the training does not compromise, at least numerically, the consistency of the trained SINN models.

\begin{figure}[t]
\centerline{
\includegraphics[height=4cm]{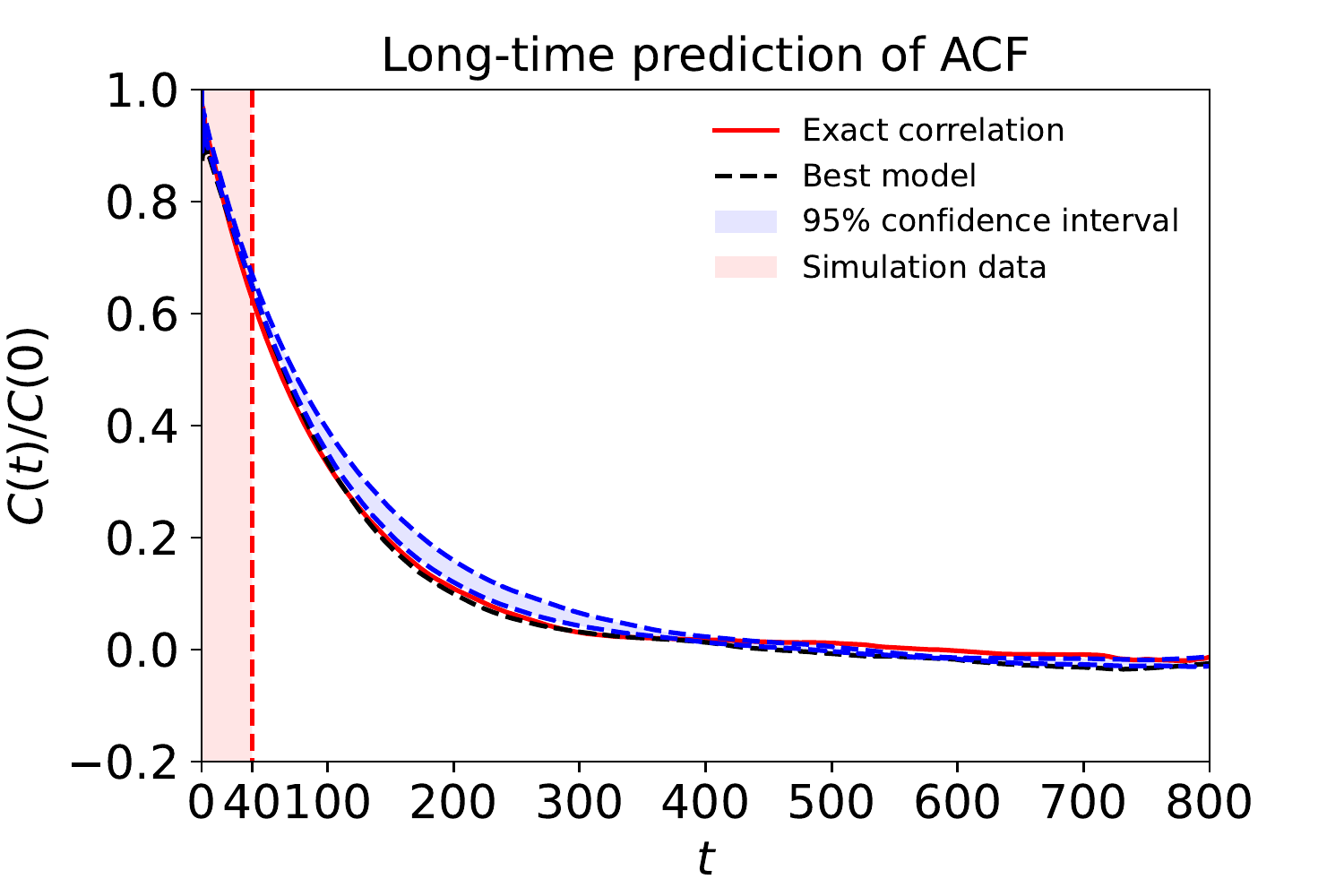}
\includegraphics[height=4cm]{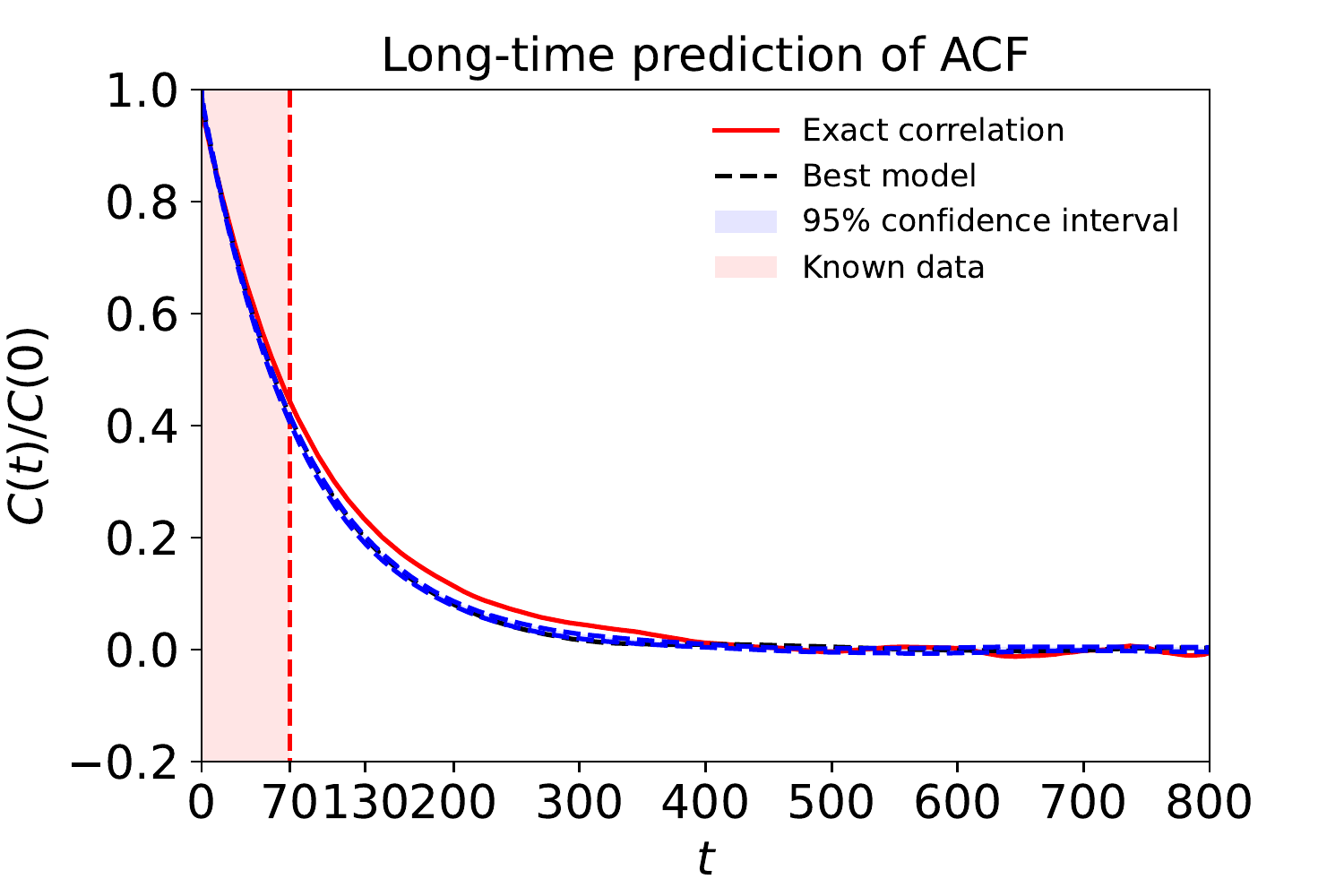}
\includegraphics[height=4cm]{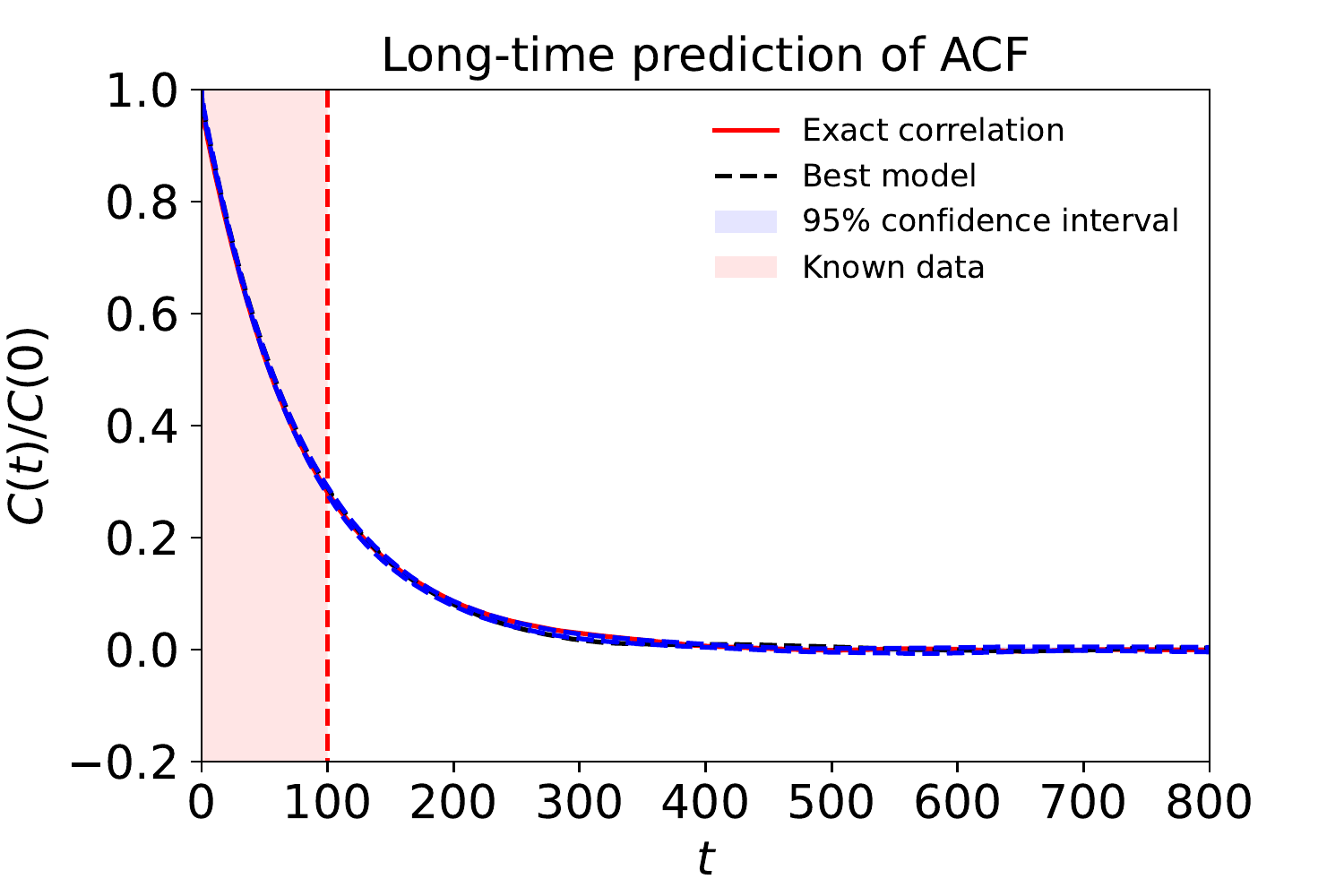}
}
\caption{Long-time predictions of the ACF of $x(t)$ using SINN. In each figure, the ACFs of the top 5 qualified SINN model are used to calculate the $95\%$ confidence interval of the predicted dynamics. For qualified SINN models, the one with the smallest validation error $\epsilon_V$ is selected to be the best model.}
\label{fig:sinn_prediction} 
\end{figure}

\section{Conclusion}
\label{sec:conclusion}
In this paper, we introduced a statistics-informed neural network (SINN) for learning stochastic dynamics. The design and construction of SINN is theoretically inspired by the universal approximation theorem for one-layer RNNs with stochastic inputs. This new model uses i.i.d.\ white-noise sequences as the input and layers of long short-term memory (LSTM) cells as the functional units to generate output sequences. The statistics of the target stochastic process, such as equilibrium probability density and time autocorrelation functions of different orders, are used in the loss function to train the parameters. SINN has a relatively simple architecture where deterministic transformations are applied to the random input and is easy to implement and train. Numerical simulation results have shown that SINN can effectively approximate Gaussian and non-Gaussian dynamics for both Markovian and non-Markovian stochastic systems. The successful application of SINN in modeling the transition dynamics clearly indicates that it can serve as a useful surrogate model to simulate rare events. Moreover, the coarse-grained nature and the long-time predictability of SINN makes it an efficient and reliable framework for reduced-order modeling. Further applications and extensions of this framework in the general area of stochastic modeling, uncertainty quantification, and time series analysis can be expected. The code we used to train and evaluate our models is available at \REPO.

\section{Acknowledgment}

This work was supported in part by the U.S.\ Department of Energy, Office of Science, Office of Advanced Scientific Computing Research, Scientific Discovery through Advanced Computing (SciDAC) Program through the FASTMath Institute under Contract No.\ DE-AC02-05CH11231 at Lawrence Berkeley National Laboratory, and the National Science Foundation under Grant No.\ 2213368. We would like to thank Prof.\ George Karniadakis, Prof.\ Yannis Kevrekidis, and Prof.\ Huan Lei for the discussion on the transition dynamics and neural networks. We also thank anonymous reviewers for providing valuable suggestions that improved the presentation of this work.

\bibliography{bibliography,20-LSTM}
\bibliographystyle{unsrt}

\clearpage

\appendix
\numberwithin{equation}{section}

\section{Universal Approximation Theorem for Recurrent Neural Networks with Gaussian Stochastic Inputs}
\label{app1:proof}

Following \cite{schafer2006recurrent}, we first introduce some useful definitions and established universal approximation results for the deterministic recurrent neural network (RNN).

\begin{defn}
For any (Borel-)measurable function $f(\cdot):\R^J\rightarrow\R^J$ and $I,N\in\mathbb{N}$, $\sum^{I,N}(f)$ is called a function class for three-layer feedforward neural networks if any $g\in \sum^{I,N}(f)$ is of the form:
\begin{align*}
    g(x)=Vf(Wx-\theta), \qquad 
    \text{where}\qquad
    x\in \R^I,\quad V\in \R^{N\times J},\quad W\in \R^{J\times I}, \quad \theta\in\R^J, \quad J\in \mathbb{N}.
\end{align*}
This three-layer feedforward neural network has $I$ input neurons, $J$ hidden neurons, and $N$ output neurons. Note that the function $f:\R^J\rightarrow\R^J$ is defined to be component-wise with
\begin{align}\label{component-wise_f}
    f(Wx-\theta):=
    \begin{bmatrix}
    f(W_1x-\theta_1)\\
    f(W_2x-\theta_2)\\
    \vdots\\
    f(W_Jx-\theta_J)
    \end{bmatrix}.
\end{align}
\end{defn}

\begin{defn}
A subset $S$ of a metric space $(X,\rho)$ is called $\rho$-dense in a subset $U$ if there exists $s\in S$ such that $\rho(s,u)<\epsilon$ for any $\epsilon>0$ and any $u\in U$.
\end{defn}

\begin{defn}
 Let $\C^{I,N}:\R^{I}\rightarrow \R^N$ be the set of all continuous functions. A subset $S\subset C^{I,N}$ is uniformly dense on a compact domain in $\C^{I,N}$ if for any compact subset $K\subset \R^I$, $S$ is $\rho_K$-dense in $\C^{I,N}$, where $\rho_K(f,g):=\sup_{x\in K}\|f(x)-g(x)\|_{\infty}$.
\end{defn}

\begin{defn}
A function $\sigma$ is called a sigmoid function if $\sigma$ is monotonically increasing and bounded. 
\end{defn}
\noindent Common choice of sigmoid function of neural networks are $1/(1+e^{-x})$ and $\tanh(x)$.
\\

The following result is the well-known universal approximation theorem (UAT) for feedforward neural networks.
\begin{theorem}\label{UAT_FNN}
(UAT for feedforward neural networks) For any sigmoid activation function $\sigma$ and any dimensions $I$ and $N$, $\sum^{I,N}(\sigma)$ is uniformly dense on a compact domain in
$\C^{I,N}$.
\end{theorem}
The above theorem simply implies that for any sigmoid function $\sigma$, as long as $J\in \mathbb{N}$ is large enough, \ie the number of hidden state (neuron) is large enough, a three-layer feedforward neural network can approximate any continuous function in any compact domain with arbitrary accuracy. This theorem was used in \cite{schafer2006recurrent} to prove the universal approximation theorem for RNN of type \eqref{differential_sde} when the input $x_t$ is deterministic. 
\\

We now introduce the following definition of the RNN class:
\begin{defn}
\label{def:RNN_class}
Let $\sigma(\cdot):\R^J\rightarrow\R^J$ be an arbitrary sigmod function and $I,N,T\in\mathbb{N}$. The class $RNN^{I,N}(\sigma)$ refers to discrete RNN system of the form \eqref{one_layer_RNN}, \ie
\begin{align*}
    s_{t+1}&=\sigma(As_t+Bx_t-\theta),\\
    y_t&=Cs_t,
\end{align*}
where $x_t\in\R^I$, $s_t\in\R^J$, and $y_t\in\R^N$ for all $t=1,\cdots,T$. Note that here $\sigma(As_t+Bx_t-\theta)$ is calculated component-wise as in \eqref{component-wise_f}. We also define $o(RNN^{I,N}(\sigma))$ to be the set of all possible output $y_t$ for the RNN of the class $RNN^{I,N}(\sigma)$. 
\end{defn}
\noindent It is proved in  \cite{schafer2006recurrent} that $RNN^{I,N}(\sigma)$ is ``dense'' in the ``space of discrete open dynamical systems'', in the sense that for any sigmoid $\sigma$ and $\delta>0$, there exists $\tilde y_t\in o(RNN^{I,N}(\sigma))$ such that $\|\tilde y_t-y_t\|_{\infty}<\delta$, where $y_t$ is the output of a $M$-dimensional open system:
\begin{equation}
\begin{aligned}
    s_{t+1}&=g(s_t,x_t),\\
    y_t&=h(s_t),
\end{aligned}
\end{equation}
where $g(\cdot):\R^M\rightarrow\R^M$ and $h(\cdot):\R^M\rightarrow\R^N$. 

For RNN with stochastic input, the proof is similar. But the corresponding universal approximation theorem, \ie Theorem \ref{UAT_stochastic_input}, holds only in the sense of probability essentially because whether one can find an RNN model such that $\|\hat y_t-y_t\|_{\infty}<\delta$ for all $\delta>0$ becomes a random event. Using the above definitions and Theorem \ref{UAT_FNN}, we can prove the following theorem:
\begin{theorem}\label{UAT_rigorous}
(UAT for RNN with Gaussian inputs). Let $g(\cdot):\R^M\times\R^I\rightarrow \R^M$ be locally Lipschitz and $h(\cdot):\R^M\rightarrow\R^N$ be continuous. In addition, the external input $x_t\in\R^{I}$ are i.i.d.\ Gaussian random variables, the inner state $s_t\in \R^M$, and the outputs $y_t\in \R^N$ $(t=1,\cdots, T)$. Then, for a finite number of iteration steps, the probability of finding an RNN model \eqref{one_layer_RNN} such that the outputs of the RNN pathwisely approximate the solution of the discrete, stochastic dynamical system of the form \eqref{differential_sde} arbitrarily accurate, is asymptotically 1.
\end{theorem}

\paragraph{Proof of Theorem \ref{UAT_rigorous}.}
We first show that the dynamics of an $M$-dimensional open dynamical system with $s_{t+1}=g(s_t,x_t)$ can be represented by an RNN with an update function of the form $\bar s_{t+1}=\sigma(A\bar s_t+B x_t-\theta)$ for all $t=1,\cdots,T$ asymptotically almost surely. For any realization of $x_t$, Theorem \ref{UAT_FNN} implies that for any compact set $K\subset \R^M\times \R^I$ which contains $(s_t,x_t)$, one can find suitable $\bar g(s_t,x_t)\in \sum^{I+M,M}(\sigma)$ with weight matrices $V\in \R^{M\times J},W\in \R^{J\times M}$, $B\in\R^{J\times I}$, and a bias $\theta\in \R^{J}$ such that for all $t=1,\cdots, T$, 
\begin{align}\label{est_1}
    \sup_{x_t,s_t\in K}\|g(s_t,x_t)-\bar g(s_t,x_t)\|_{\infty}\leq \delta, \qquad \text{where}\qquad \bar g(s_t,x_t)=V\sigma(Ws_t+Bx_t- \theta).
\end{align}
Here $\delta>0$ is an arbitrary constant and $\sigma$ is an arbitrary component-wise applied sigmoid activation function. 

We denote approximated dynamics generated by the feedback neural network by 
\begin{align*}
    \bar s_{t+1}=\bar g(\bar s_t,x_t)=V\sigma(W\bar s_t+Bx_t-\theta).
\end{align*}
Further assuming that $\bar s_t\in K$, for any $\delta>0$, we can find suitable $W,B,V,\theta$ such that
\begin{align*}
\|s_t-\bar s_t\|_{\infty}
&=\|g(s_{t-1},x_{t-1})-g(\bar s_{t-1},x_{t-1})+g(\bar s_{t-1},x_{t-1})-\bar s_t\|_{\infty}\\
&\leq \|g(s_{t-1},x_{t-1})-g(\bar s_{t-1},x_{t-1})\|_{\infty}
+\|g(\bar s_{t-1},x_{t-1})-\bar s_t\|_{\infty}\\
&\leq \|g(s_{t-1},x_{t-1})-g(\bar s_{t-1},x_{t-1})\|_{\infty}+\delta.
\end{align*}
Since $g$ is continuously differentiable, in the compact set $K$, it is also Lipschitz continuous. This implies for any $\epsilon>0$, there is $\delta>0$ and thus there are suitable $W,B,V,\theta$, such that 
\begin{align}\label{est_2}
\|s_t-\bar s_t\|_{\infty}
\leq C\|s_{t-1}-\bar s_{t-1}\|_{\infty}+\delta
\leq \delta (1+C+\cdots C^{T-1})=\delta\frac{1-C^T}{1-C}\leq \epsilon,
\end{align}
where we have used $s_0=\bar s_0$.
Estimate \eqref{est_2} indicates that for deterministic inputs $x_t$, the open dynamical system update function $g(s_t,x_t)$ can be universally approximated by the feedward neural network update function $\hat g(\hat s_t,x_t)$ since the output of each step, \ie $s_t$, can be approximated arbitrarily accurate.

For an RNN with Gaussian random input $x_t$, however, since $x_t$ is not compactly supported, whether $x_t,s_t\in K$ becomes a random event. For any {\em fixed} $K\subset\R^{M}\times \R^I$ and sigmoid function $\sigma$, one can associate it with the function class $\sum^{I+M,M}(\sigma)$. The probability that one can find a suitable approximation to function $g(s_t,x_t)$ within $\sum^{I+M,M}(\sigma)$ for any initial $s_0\in \R^{M}$ and $x_t\in \R^I$ can be written as 
\begin{align}\label{prob_find_g}
\mathrm{Pr}\left[\inf_{\hat g\in\sum^{I+M,M}(\sigma)}
\sup_{\substack{x_t\in \R^I,s_0\in \R^M\\ t=1,\cdots, T}}
\|g(s_t,x_t)-\bar g(\bar s_t,x_t) \|_{\infty}\leq \epsilon, \forall\epsilon>0\right],
\end{align}
where for stochastic $x_t$, the norm $\|g(s_t,x_t)-\bar g(\bar s_t,x_t) \|_{\infty}$ is interpreted in the pathwise sense \cite{kloeden2007pathwise}, \ie it is valid for each realization of $x_t$. Also note that that the above probability depends on $K$, in particular, the size
\footnote{$|K|$ can be defined as, e.g. $|K|:=\sup_{x\in K}{\|x\|_2}$, where $\|x\|_2$ is the $l_2$-norm of the vector.}
of it, which is denoted as $|K|$. Taking the limit $|K|\rightarrow+\infty$, we obtain
\begin{equation}\label{prob_rela}
\begin{aligned}
&\lim_{|K|\rightarrow+\infty}\mathrm{Pr}\left[\inf_{\bar g\in\sum^{I+M,M}(\sigma)}
\sup_{\substack{x_t\in \R^I,s_0\in \R^M\\ t=1,\cdots, T}}
\|g(s_t,x_t)-\bar g(\bar s_t,x_t) \|_{\infty}\leq \epsilon, \forall\epsilon>0\right]\\
=&
\lim_{|K|\rightarrow+\infty}\mathrm{Pr}\left[\inf_{\bar g\in\sum^{I+M,M}(\sigma)}
\sup_{\substack{x_t,s_t,\bar s_t\in K\\ t=1,\cdots, T}}
\|g(s_t,x_t)-\bar g(\bar s_t,x_t)\|_{\infty}\leq \epsilon,\forall\epsilon>0\bigg| x_t,s_t,\bar s_t\in K\right]\mathrm{Pr}[x_t,s_t,\bar s_t\in K]\\
=&\lim_{|K|\rightarrow+\infty}\mathrm{Pr}[x_t,s_t,\bar s_t\in K].
\end{aligned}
\end{equation}
Here we used \eqref{est_1} and \eqref{est_2} to show that under the condition $x_t,s_t,\bar s_t\in K$, the event $\inf_{\bar g\in\sum^{I+M,M}}\sup_{x_t,s_t,\bar x_t\in K}\|g(s_t,x_t)-\bar g(\bar s_t,x_t)\|_{\infty}<\delta,\forall \delta>0,t=1,\cdots, T$ happens with probability 1, \ie it holds for almost all $x_t$. 

Now the problem of calculating the probability \eqref{prob_find_g} in the limit $|K|\rightarrow+\infty$ boils down to the calculation of $\lim_{|K|\rightarrow+\infty}\mathrm{Pr}[x_t,s_t,\bar s_t\in K]$. To this end, we note that if we can choose a compact set $K_0\subset \R^M\times \R^I $ such that $x_t,s_0\in K_0$ for all $t=1,\cdots,T$, then $\|s_1\|=\|g(s_0,x_0)\|\leq C_0$ for some $C_0>0$ because $g$ is a continuous function therefore bounded in $K_0$. We choose another compact set $K_0\cup B(0,C_0)\subset K_1\subset \R^J\times \R^J$, where $B(0,C_0)$ is a ball centered at 0 with radius $C_0$. Then we must have $\|s_2\|=\|g(s_1,x_1)\|\leq C_1$ since $x_1,s_1\in K_1$. Continuing this procedure for $T$ times, we can find a compact set $K_T\subset \R^M\times \R^I$ such that $s_t,x_t\in K_T$. Since the same logic applies to $\bar s_t$, we can find a compact $K\subset \R^M\times \R^I$ such that $s_t,x_t,\bar s_t\in K_T$. On the other hand, Chebyshev’s inequality for a standard normal random variable $X$ implies 
\begin{align}
    \mathrm{Pr}[|X|\geq b]\leq \frac{1}{b^2} \qquad \Rightarrow\qquad \mathrm{Pr}[X<b]\geq 1-\frac{1}{b^2}.
\end{align}
Then for i.i.d.\ random variables $X_1,\cdots X_T$, we have 
\begin{align*}
    \mathrm{Pr}[|X_1|<b,\cdots ,|X_T|<b]=\prod_{i=1}^T \mathrm{Pr}[|X_i|<b]\geq \left[1-\frac{1}{b^2}\right]^T.
\end{align*}
Therefore, for fixed $T$, we have 
\begin{align}\label{prob_containing_region}
   \lim_{b\rightarrow+\infty} \mathrm{Pr}[|X_1|<b,\cdots, |X_T|<b]=1.
\end{align}
This means that asymptotically almost surely one can find a compact subset $B$ such that $x_t\in B$ for all $t=1,\cdots T$. As a direct result of this, we have 
\begin{align}\label{limit_K}
    \lim_{|K|\rightarrow+\infty}\mathrm{Pr}[x_t,s_t,\bar s_t\in K]=1.
\end{align}

Combining \eqref{prob_rela} and \eqref{limit_K}, we show that for an RNN with i.i.d.\ Gaussian input $x_t$, asymptotically almost surely, one can find suitable $W,B,V,\theta$ such that $\|s_t-\bar s_t\|_{\infty}<\epsilon$ for any sigmoid $\sigma$ and any $\epsilon>0$. Furthermore, let 
\begin{align*}
    s_{t+1}'=\sigma(W\bar s_t+Bx_t-\theta),
\end{align*}
which yields $\bar s_t=Vs'_t$. By defining $A:=WV \in \R^{J\times J}$, we obtain
\begin{align}\label{s_t+1}
    s_{t+1}'=\sigma(A s_t'+Bx_t-\theta).
\end{align}
The dynamics of the RNN update function \eqref{s_t+1} {\em encodes} (not equals) the dynamics of the open dynamical systems. Hence we claim that the dynamics of an $M$-dimensional open dynamical system with $s_{t+1}=g(s_t,x_t)$ can be represented by an RNN with an update function of the form $\bar s_{t+1}=\sigma(A\bar s_t+B x_t-\theta)$ for all $t=1,\cdots,T$ asymptotically almost surely. We note  that the transformation $\bar s_t=Vs'_t$ often involves an enlargement of the hidden state dimensionality since $A\in \R^{J\times J}$, where $J$ is set to be large enough to guarantee the validity of the universal approximation. 

The second part of the proof is to show that the output of the dynamical system, \ie $y_t=h(s_t)$ can be approximated by the output of an RNN $\tilde y_t=\tilde C \tilde s_t$ asymptotically almost surely where $\tilde s_t$ is an {\em extended} vector satisfying the RNN update rule: $\tilde s_{t+1}=\sigma(\tilde A\tilde s_t+\tilde Bx_t-\tilde \theta)$. For an RNN with deterministic input, the proof is done in \cite{schafer2006recurrent}. Hence we simply state the result obtained therein.

\paragraph{Claim.} For $x_t,s_t,\bar s_t\in K\subset \R^{M}\times \R^I$, there exist enlarged matrices $\tilde{A},\tilde{B},\tilde C$ and $\tilde \theta$ for an RNN model:
\begin{align}\label{tilde_s_t}
    \tilde s_{t+1}&=\sigma(\tilde A\tilde s_t+\tilde Bx_t-\tilde \theta)\nonumber,\\
    \tilde y_t&=\tilde C\tilde s_t,
\end{align}
such that the output vector $\|\tilde y_t-y_t\|\leq \epsilon$ for all $\epsilon>0$. In \eqref{tilde_s_t}, we have 
\begin{align*}
&
\tilde J=J+\bar J
,\qquad
r_t=\sigma(Es_t'+Fx_t-\bar\theta)\in \R^{\bar J}
,\quad E\in \R^{\bar J\times J},
F\in \R^{\bar J\times I},
\bar \theta\in \R^{\bar J}
,\qquad
\tilde s_t=
\begin{bmatrix}
s_t'\\
r_t
\end{bmatrix}
\in \R^{\tilde J},
\\
&\tilde A=
\begin{bmatrix}
A& 0\\
E&0
\end{bmatrix}
\in \R^{\tilde J\times\tilde J}
,\qquad
\tilde B=
\begin{bmatrix}
B\\
F
\end{bmatrix}
\in \R^{\tilde J\times J}
,\qquad
\tilde C=[0\quad D]\in\R^{N\times \tilde J}
,\quad 
\text{and}\quad
\tilde \theta=
\begin{bmatrix}
\theta\\
\bar\theta
\end{bmatrix}
\in \R^{\tilde J}.
\end{align*}

\begin{proof}
\begin{align*}
\|y_t-\tilde y_t\|
&=\|y_t- D\sigma(Es_{t-1}'+Fx_{t-1}-\bar \theta)\|\\
&\leq 
\|y_t- Q\sigma(GV\sigma(As_{t-1}'+Bx_{t-1}-\theta)-\hat \theta)\|
+\|Q\sigma(GV\sigma(As_{t-1}'+Bx_{t-1}-\theta)-\hat \theta)
-D\sigma(Es_{t-1}'+Fx_{t-1}-\bar \theta)\|.
\end{align*}
Here $Q\sigma(GV\sigma(As_{t-1}'+Bx_{t-1}-\theta)-\hat \theta)$ is a bounded function defined in the compact domain $K$. Hence, the universal approximation theorem implies that for any $\epsilon_1>0$, we can find suitable
$D,E,F,\bar\theta$ such that 
$\|Q\sigma(GV\sigma(As_{t-1}'+Bx_{t-1}-\theta)-\hat \theta)
-D\sigma(Es_{t-1}'+Fx_{t-1}-\bar \theta)\|\leq\epsilon_1$. Then we obtain 
\begin{align*}
\|y_t-\tilde y_t\|
&\leq 
\|y_t- Q\sigma(GV\sigma(As_{t-1}'+Bx_{t-1}-\bar\theta)-\hat \theta)\|+\epsilon_1\\
&=\|y_t-Q\sigma(GVs_t'-\hat\theta)\|+\epsilon_1\\
&=\|y_t-Q\sigma(G\bar s_t-\hat\theta)\|+\epsilon_1\\
&\leq 
\|y_t-h(\bar s_t)\|+\|h(\bar s_t)-Q\sigma(G\bar s_t-\hat\theta)\|+\epsilon_1.
\end{align*}
Again applying the universal approximation theorem, we have $\|h(\bar s_t)-Q\sigma(G\bar s_t-\hat\theta)\|\leq \epsilon_2,\forall \epsilon_2>0$. On the other hand, $h(x)$ is continuously differentiable and thus Lipschitz in the compact set $K$ and we have $\|y_t-h(\bar s_t)\|=\|h(s_t)-h(\bar s_t)\|\leq C\|s_t-\bar s_t\|\leq C\epsilon_0$. This leads to  
\begin{align}\label{estimate_yt}
\|y_t-\tilde y_t\|\leq C\epsilon+\epsilon_1+\epsilon_2\leq\delta, \forall \delta>0.
\end{align}
\end{proof}

This claim indicates that under the condition $x_t,s_t,\bar s_t\in K\subset \R^{M}\times \R^I$, the probability of finding suitable $\tilde A,\tilde B,\tilde \theta, \tilde C$ such that $\|\tilde y_t-y_t\|\leq \delta,\forall\delta>0$ is 1. Hence using again relations \eqref{prob_find_g}, \eqref{prob_rela}, and \eqref{limit_K}, we have 
\begin{align}\label{prob_rela_3}
&\lim_{|K|\rightarrow+\infty}\mathrm{Pr}\left[\inf_{\tilde y_t\in o\left(RNN^{I,N}(\sigma)\right)}\sup_{x_t\in \R^I,s_0\in \R^M}\|y_t-\tilde y_t\|_{\infty}\leq \delta, \forall\delta>0\right],\qquad t=1,\cdots,T\nonumber\\
=&
\lim_{|K|\rightarrow+\infty}
\mathrm{Pr}\left[\inf_{\tilde y_t\in o\left(RNN^{I,N}(\sigma)\right)}\sup_{x_t,s_t,\bar s_t\in K}\|y_t-\bar y_t\|_{\infty}\leq \delta, \forall\delta>0\bigg|x_t,s_t,\bar s_t\in K\right]\mathrm{Pr}[x_t,s_t,\bar s_t\in K]\nonumber\\
=&\lim_{|K|\rightarrow+\infty}\mathrm{Pr}[x_t,s_t,\bar s_t\in K]=1.
\end{align}
This implies that any open dynamical system \eqref{differential_sde} with Gaussian inputs $x_t$, the existence of finding a {\em deterministic} RNN model of the form \eqref{one_layer_RNN} with the {\em same stochastic} input $x_t$ that pathwisely approximates the solution of \eqref{differential_sde} arbitrarily accurate tends to 1 as $|K|\rightarrow +\infty$. This concludes the proof of Theorem \ref{UAT_rigorous}.

\begin{theorem}\label{UAT_rigorous_SDE}
Suppose $b(x):\R^d\rightarrow \R^d$ and $\sigma(x):\R^m\rightarrow \R^d$ in SDE \eqref{Ito_SDE} are locally Lipschitz and $b(x)$ satisfies conditions (i)--(iii) listed in \cite{gyongy1998note}. Consider a uniform time grid $0=t_0<t_1\cdots<t_{N_T}=T$ with step size $\Delta t=t_{i+1}-t_i$. Then in this time grid, the probability of finding an RNN model \eqref{one_layer_RNN} such that the outputs of the RNN pathwisely approximate the exact solution of the SDE arbitrarily accurate, is asymptotically 1.
\end{theorem}

\begin{proof}
The proof to this theorem is obtained by combining Theorem \ref{UAT_rigorous} and the established pathwise convergence result for the Euler--Maruyama (EM) scheme to the exact solution to SDE \eqref{Ito_SDE}. It was proved by Gy\"ongy \cite{kloeden2007pathwise,gyongy1998note} that under the local Lipschitz condition on $b(x),\sigma(x)$ and (i)--(iii) listed in \cite{gyongy1998note}, the EM scheme with uniform step size $\Delta t$ satisfies pathwise error estimate:
\begin{align}\label{pathwise_c}
    \sup_{i=0,\cdots,N_T}\|X(t_i,\omega)-\hat X_{EM}^{\Delta}(t_i,\omega)\|_2\leq C_{T}(\omega)\Delta^{\frac{1}{2}-\epsilon},\quad \forall \epsilon>0, \quad \text{for almost all $\omega\in\Omega$},
\end{align}
where $\|\cdot\|_2$ is the Euclidean norm, $\omega\in \Omega$ defines a sample realization of the Wiener process $W(t)$, and $\hat X_{EM}^{\Delta}(t_i,\omega)$ is the pathwise approximated solution generated using EM scheme with the same Wiener process realization. This result also can be restated as: for any $\omega\in \Omega$, the above inequality holds in probability 1. 

Now, suppose the EM scheme is the open dynamical system \eqref{differential_sde} as we defined in \eqref{EM_scheme} and we specify the output of the open dynamical system as $y_t=h(s_t)=h(X(t))=X(t)$. Since $b(x)$ and $\sigma(x)$ in SDE \eqref{Ito_SDE} are locally Lipschitz, naturally $g(\cdot)$ in \eqref{EM_scheme} is also locally Lipschitz
\footnote{Here we mean the function $g(x,y,z)=x-zb(x)+\sigma\sqrt{z}y$ is locally Lipschitz for fixed $z$.}
for any fixed $\Delta t>0$. Applying the result in Theorem \ref{UAT_rigorous}, in particular, estimate \eqref{estimate_yt}, we know there exists a stochastic RNN such that
\begin{align}\label{pathwise_c1}
   \sup_{i=0,\cdots,N_T} \|\hat X_{EM}^{\Delta}(t_i,\omega)-\hat X_{RNN}^{\Delta}(t_i,\omega)\|_{\infty}<\delta, \forall \delta>0, \quad\text{for any $\omega\in \Omega$, $\{\xi(t_i)\}_{i=0}^{N_T},\{\hat X_{EM}^{\Delta}(t_i,\omega)\}_{i=0}^{N_T},\{\hat X_{RNN}^{\Delta}(t_i,\omega)\}_{i=0}^{N_T}\in K$,}
\end{align}
where $\hat X_{RNN}^{\Delta}(t_i,\omega)$ is the path generated by stochastic RNN with any fixed $\omega\in \Omega$. Combining \eqref{pathwise_c} and \eqref{pathwise_c1} and then using the triangle inequality and  $\|f\|_{2}\leq \sqrt{d}\|f\|_{\infty}$, we obtain
\begin{equation}\label{pathwise_c2}
\begin{aligned}
    \sup_{i=0,\cdots,N_T}\|X(t_i,\omega)-\hat X_{RNN}^{\Delta}(t_i,\omega)\|_2&\leq
   \sup_{i=0,\cdots,N_T}\|X(t_i,\omega)-\hat X_{EM}^{\Delta}(t_i,\omega)\|_{2}+
    \sqrt{d}\sup_{i=0,\cdots,N_T} \|\hat X_{EM}^{\Delta}(t_i,\omega)-\hat X_{RNN}^{\Delta}(t_i,\omega)\|_{\infty}\\
    &
   \leq C_{T}(\omega)\Delta^{\frac{1}{2}-\epsilon}+\sqrt{d}\delta,
\end{aligned}
\end{equation}
which is valid for all $\epsilon,\delta>0$, almost all $\omega\in\Omega$ and $\{\xi(t_i)\}_{i=0}^{N_T},\{\hat X_{EM}^{\Delta}(t_i,\omega)\}_{i=0}^{N_T},\{\hat X_{RNN}^{\Delta}(t_i,\omega)\}_{i=0}^{N_T}\in K\subset \R^{d}\times\R^{m}$. 
As a result, by using estimate \eqref{prob_rela_3}, we obtain
\begin{equation}\label{prob_rela_4}
\begin{aligned}
&\lim_{|K|\rightarrow+\infty}
\lim_{\Delta\rightarrow 0}
\mathrm{Pr}\left[
\inf_{\hat X^{\Delta}(t,\omega)\in o\left(RNN^{m,d}(\sigma)\right)}
\sup_{\substack{\{\xi(t_i)\}_{i=0}^{N_T}\in\R^{m}, X(0,\omega)\in \R^{d}\\ t=1,\cdots, T}}
\|X(t_i,\omega)-\hat X_{RNN}^{\Delta}(t_i,\omega)\|_{\infty}\leq \delta, \forall\delta>0
\right]\\
=&
\lim_{|K|\rightarrow+\infty}
\lim_{\Delta\rightarrow 0}
\mathrm{Pr}\Bigg[
\inf_{\hat X^{\Delta}(t,\omega)\in o\left(RNN^{m,d}(\sigma)\right)}
\sup_{\substack{\{\xi(t_i)\}_{i=0}^{N_T}\in\R^{m},\{\hat X_{EM}^{\Delta}(t_i,\omega)\}_{i=0}^{N_T},\{\hat X_{RNN}^{\Delta}(t_i,\omega)\}_{i=0}^{N_T}\in K\\ t=1,\cdots, T}}
\|X(t_i,\omega)-\hat X_{RNN}^{\Delta}(t_i,\omega)\|_{\infty}
\leq \delta, \forall\delta>0\\
&\bigg|\{\xi(t_i)\}_{i=0}^{N_T},\{\hat X_{EM}^{\Delta}(t_i,\omega)\}_{i=0}^{N_T},\{\hat X_{RNN}^{\Delta}(t_i,\omega)\}_{i=0}^{N_T}\in K]\Bigg]
\times
\mathrm{Pr}[\{\xi(t_i)\}_{i=0}^{N_T},\{\hat X_{EM}^{\Delta}(t_i,\omega)\}_{i=0}^{N_T},\{\hat X_{RNN}^{\Delta}(t_i,\omega)\}_{i=0}^{N_T}\in K]\\
=&\lim_{|K|\rightarrow+\infty}
\lim_{\Delta\rightarrow 0}
\mathrm{Pr}[\{\xi(t_i)\}_{i=0}^{N_T},\{\hat X_{EM}^{\Delta}(t_i,\omega)\}_{i=0}^{N_T},\{\hat X_{RNN}^{\Delta}(t_i,\omega)\}_{i=0}^{N_T}\in K]=1.
\end{aligned}
\end{equation}
This probability can be interpreted as follows. For any $\omega\in \Omega$, taking the limit $\Delta\rightarrow 0$ and then $|K|\rightarrow+\infty$ (note that the order is not exchangeable), the probability of finding a suitable stochastic RNN from $RNN^{m,d}(\sigma)$, which takes sample discrete white noise $\{\xi(t_i)\}_{i=0}^{N_T}$ and generates outputs $\{\hat X_{RNN}(t_i)\}_{i=0}^{N_T}$ that accurately approximate (error$<\delta$ holds for any $\delta>0$) the exact solution $\{X(t_i)\}_{i=0}^{N_T}$ in all time grid $0<t_1\cdots < t_{N_T}$, tends to 1. This concludes the proof of Theorem \ref{UAT_rigorous_SDE}.
\end{proof}

As a final note, we point out that our UAT results are established based on pathwise convergence, while what SINN seeks is weak convergence, \ie convergence in terms of statistical moments. While the UAT convergence results do not imply weak convergence, they do provide a plausible hint that weak convergence can be numerically achieved. To obtain a UAT for weak or strong convergence, a prerequisite that needs further examination is the ergodicity condition of the stochastic RNN model\cite{mattingly2002ergodicity}. This is beyond the scope of the present paper and awaits future investigation.
\end{document}